\documentclass{article}
\usepackage[letterpaper,margin=1in]{geometry}
\usepackage{amsmath,amsfonts,amssymb}
\usepackage{mathtools}
\usepackage{amsthm}
\usepackage{algorithm}
\usepackage{algorithmic}
\usepackage{array}
\usepackage{booktabs}
\usepackage{multirow}
\usepackage[caption=false,font=normalsize,labelfont=sf,textfont=sf]{subfig}
\usepackage{textcomp}
\usepackage{stfloats}
\usepackage{url}
\usepackage{graphicx}
\usepackage{xcolor}
\usepackage{cite}
\usepackage[hidelinks]{hyperref}

\newtheorem{definition}{Definition}[section]
\newtheorem{proposition}{Proposition}[section]
\newtheorem*{proposition*}{Proposition}
\newtheorem{remark}{Remark}[section]

\title{Heterogeneous Information-Bottleneck Coordination Graphs\\for Multi-Agent Reinforcement Learning}
\author{Wei~Duan, Junyu~Xuan, En~Yu, Xiaoyu~Yang, and Jie~Lu\\
Australian Artificial Intelligence Institute (AAII), University of Technology Sydney\\
\texttt{\{wei.duan, junyu.xuan, en.yu-1, jie.lu\}@uts.edu.au}, \texttt{xiaoyu.yang-3@student.uts.edu.au}}
\date{}

\begin{document}
\maketitle

\begin{abstract}
  Coordination graphs specify which agents exchange information in cooperative multi-agent reinforcement learning (MARL). Existing sparse-graph methods, however, rely on heuristic, edge-uniform criteria for topology and structurally blind bottlenecks for message bandwidth, and lack a principled way to jointly learn heterogeneous connectivity and allocate differentiated communication capacity.

  We propose \textbf{Heterogeneous Information-Bottleneck Coordination Graphs (HIBCG)}, which couples a group-aligned block-diagonal prior---assigning heterogeneous edge density per group block---with per-agent message compression on the learned graph.
  Theoretically, we show that the group-aligned prior is never worse than a flat isotropic prior, that the structural penalty decomposes additively across group blocks, and that minimising the standard TD loss lower-bounds IB relevance without a separate mutual-information estimator.
  Across nine scenarios in SMACv1, SMACv2, and MAgent Battle (up to 100 agents), HIBCG attains the strongest results on heterogeneous multi-role maps, scales where several baselines fail to converge, and is validated by ablations and theory-aligned diagnostics of the group-aligned prior and dual-path design.
\end{abstract}

\paragraph{Keywords.} Multi-agent reinforcement learning, coordination graph, graph learning, information bottleneck, graph neural network.

\section{Introduction}
\label{sec:intro}

Cooperative multi-agent reinforcement learning (MARL) relies on agents exchanging task-relevant information so that local decisions can be coordinated toward a shared team objective~\cite{sukhbaatar2016commnet,jiang2018atoc,TarMAC,guo2024survey}. 
A \emph{coordination graph} makes this communication selective by specifying which agent pairs exchange information at each step~\cite{foerster2016learning,liu2023partially,Duan2025bayesian}. 
By keeping only informative links, sparse coordination graphs reduce noise aggregation, lower computation, and generalise better to unseen states~\cite{CASEC,SOPCG,GACG,DBLP:conf/ifaamas/VarelaSM25}. The core open problem is therefore not \emph{whether} to use a graph, but \emph{how to learn} a graph whose topology is faithful to the underlying coordination structure of the task. Existing graph learners answer this question with \emph{homogeneous} criteria---fixed thresholds~\cite{GACG}, top-$k$ selection~\cite{SOPCG}, variance payoffs~\cite{CASEC}, attention scores~\cite{DICG,GA2NET}, or dynamic factor-graph generation policies~\cite{DDFG}---that treat every link with the same rule.

Two limitations remain.
\textbf{(1) Topology learning is edge-uniform and heuristic:} these methods lack a principled mechanism to learn heterogeneous connectivity---e.g., how dense intra-group edges should be relative to inter-group links when agents form functional sub-teams.
\textbf{(2) Message bandwidth is structurally blind:} even after a topology is learned, surviving edges share the same representational capacity or pass through a single global bottleneck~\cite{NDQ,MASIA,MAGI,BVME}. 
In heterogeneous deployments, agents collaborating within a functional sub-team (e.g., co-located mobile robots synchronising a joint manoeuvre) require dense, high-bandwidth messaging to maintain coherence, whereas cross-team links primarily carry sparse status or handoff signals. Yet, no existing method allocates differentiated communication capacity to these structurally distinct agent relationships.

\begin{figure}[t]
\centering
\includegraphics[width=\columnwidth]{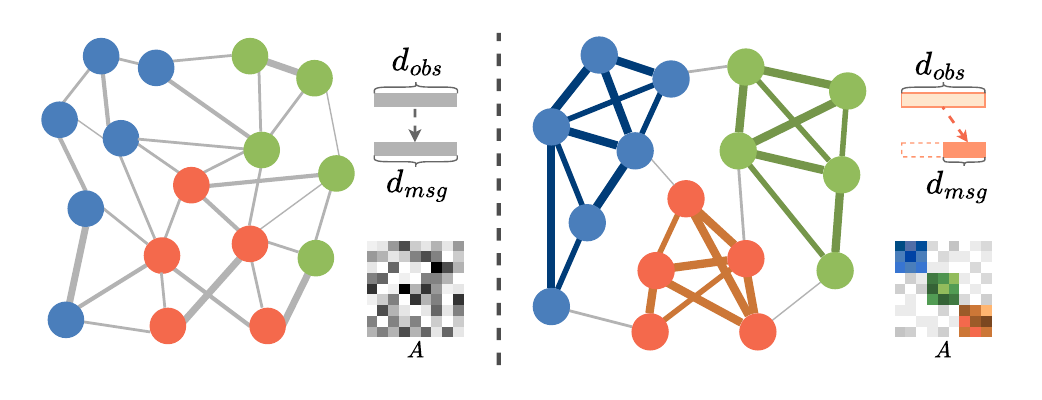}
\caption{%
  \textbf{Edge-uniform vs.\ heterogeneous coordination graphs.}
  \emph{Left:} existing methods apply edge-uniform sparsification to all links (uniform grey edges), leave message bandwidth uncompressed or squeezed by a single global bottleneck ($d_{\mathrm{obs}}\!\to\!d_{\mathrm{msg}}$ without reduction), and yield a structurally indifferent adjacency~$A$.
  \emph{Right:} HIBCG learns heterogeneous connectivity---dense, thick intra-group edges and sparse, thin cross-group links---via a group-aligned block-diagonal prior, and compresses per-agent messages on the resulting graph (orange dashed bottleneck). The learned adjacency~$A$ exhibits block-diagonal structure aligned with the agent partition.
}
\label{fig:conceptual}
\end{figure}

\medskip
To address both limitations, we propose \textbf{Heterogeneous Information-Bottleneck Coordination Graphs (HIBCG)}, which jointly learns heterogeneous connectivity and allocates differentiated message capacity for structurally different agent relationships (Figure~\ref{fig:conceptual}). Using the graph information bottleneck (GIB) as a \emph{tool}, HIBCG couples two controls. \textbf{Topology learning} uses a group-aligned block-diagonal prior to provide a closed-form criterion for edge retention, making intra-group edges cheap to keep while requiring cross-group edges to justify their information cost. \textbf{Message control} then regulates the per-agent feature bandwidth on the learned topology, compressing messages so that the surviving communication channels carry task-relevant content. Together, these two controls learn not only which edges should exist, but also how much information should flow through the resulting graph.

\noindent\textbf{Contributions.}
\begin{enumerate}
\item \textbf{A new method for heterogeneous coordination-graph learning.}\;
We propose HIBCG, to our knowledge the first method that jointly learns a sparse topology---with group-aware density control per edge block---and a compressed per-agent message representation matched to that topology---with per-agent bandwidth allocation---using closed-form Gaussian KL penalties built on a group-aligned block-diagonal prior.

\item \textbf{Theoretical guarantees via the GIB framework.}\;
We establish that the joint GIB objective decomposes into structural and message paths via the mutual-information chain rule (Prop.~\ref{prop:dual-path}); that the block-diagonal prior family always contains the flat prior as a special case, providing a no-regret guarantee for group-aligned priors (Remark~\ref{rem:no-regret}); that the structural penalty decomposes additively across group blocks (Prop.~\ref{prop:group-decomp}); and that under standard value-decomposition assumptions, reducing the TD loss increases a lower bound on the IB relevance term, removing the need for a separate mutual-information estimator (Prop.~\ref{prop:relevance}).

\item \textbf{Broad empirical validation.}\;
HIBCG adds only $3.8\%$ wall-clock overhead. Experiments on SMACv1, SMACv2, and MAgent Battle (up to 100 agents) show state-of-the-art performance and scalability, with pronounced gains on heterogeneous multi-role tasks. 
Comprehensive ablations and theory-aligned diagnostics validate the group-aligned prior and dual-path design, confirming the group-aligned prior as the primary source of gains on heterogeneous tasks.
\end{enumerate}

The remainder of this paper is organised as follows.
Section~\ref{sec:related} reviews related work.
Section~\ref{sec:prelim} presents preliminaries.
Section~\ref{sec:method} develops the HIBCG framework---method and theoretical analysis---following the three-stage architecture in Figure~\ref{fig:architecture}.
Section~\ref{sec:experiments} reports experiments.
Section~\ref{sec:method-discussion} discusses the relation to the closest prior methods.
Section~\ref{sec:conclusion} concludes.
All proofs and additional analyses are provided in the Appendix.

\section{Related Work}
\label{sec:related}

\subsection{Graph Learning for MARL}
\label{sec:related-graph-ib}

Graph neural networks (GNNs)~\cite{DBLP:conf/iclr/KipfW17,DBLP:conf/aaai/DuanXQ022,duan2024layerdiverse,10255371,DBLP:journals/eswa/YaoHLDQYS25} provide a natural substrate for MARL coordination, where agents are nodes and edges define information flow. Early graph-based MARL methods such as DGN~\cite{DGN}, NerveNet~\cite{NerveNet}, DCG~\cite{DCG}, and DICG~\cite{DICG} showed that graph message passing can improve credit assignment and coordination. Later work focuses on learning more scalable or adaptive topologies: attention-based routing~\cite{TarMAC,MAGIC,GA2NET}, sparse coordination graphs~\cite{SOPCG,CASEC,LTSCG,CommFormer,DBLP:conf/icml/LinL24}, recurrent graph message passing for generalisation~\cite{DBLP:conf/atal/WeilBAM24}, hand-designed scalable topologies~\cite{HAMA,ExpoComm}, and dynamic factor-graph generation~\cite{DDFG}.

These methods improve \emph{how} a graph is parameterised or updated, but edge existence is usually determined by homogeneous scores such as attention weights, thresholds, variance payoffs, or learned generation policies. 
They do not provide a theoretical criterion for deciding whether a particular edge should exist under different structural relationships. 
Role- and group-aware MARL methods such as ROMA~\cite{ROMA}, VAST~\cite{VAST}, REFIL~\cite{REFIL}, SOG~\cite{SOG}, and GoMARL~\cite{GoMARL} recognise that agents may form functional sub-teams, but they do not learn a coordination-graph topology.
Our earlier GACG~\cite{GACG} is, to our knowledge, the first method that jointly learns a sparse coordination graph and an agent group partition; yet it still decides edge existence by hard thresholds, without a principled compression cost tied to each group block.

\subsection{Communication-Efficient MARL}

Communication is a second major axis of cooperative MARL. Classic differentiable communication methods~\cite{foerster2016learning,CommFormer,TarMAC} and scheduling or pruning methods~\cite{kim2019schednet,DBLP:conf/aaai/MaoZXGN20,DBLP:conf/ifaamas/DolanNAB25} reduce when or with whom agents communicate, while information-theoretic methods such as NDQ~\cite{NDQ}, MASIA~\cite{MASIA}, and MAGI~\cite{MAGI} encourage compact or robust messages. The information bottleneck principle~\cite{IB_Tishby,GIB_Wu,alemi2017fixing,igl2019generalization} provides a natural language for this goal, and our earlier BVME~\cite{BVME} applies variational message compression on top of GACG's sparse graph.
These approaches, however, apply a global or topology-agnostic regulariser and do not couple message capacity with edge existence across structurally different group blocks---the gap that HIBCG addresses via per-block topology learning and per-agent message control.

\section{Preliminaries}
\label{sec:prelim}

\subsection{Cooperative Dec-POMDPs}
We consider cooperative multi-agent tasks modeled as a Decentralized Partially Observable Markov
Decision Process (Dec-POMDP)~\cite{DBLP:series/sbis/OliehoekA16} with tuple
$\langle \mathcal{A}, \mathcal{S}, \{\mathcal{U}_i\}_{i=1}^n, P, \{\mathcal{O}_i\}_{i=1}^n, \{\pi_i\}_{i=1}^n, R, \gamma \rangle$.
At time $t$, each agent $i\!\in\!\mathcal{A}$ receives $o_i^t\!\in\!\mathcal{O}_i$ and selects
$u_i^t\!\in\!\mathcal{U}_i$ via a local policy $\pi_i(u_i^t\mid\tau_i^t)$ where
$\tau_i^t=(o_i^0,u_i^0,\dots,o_i^t)$.
The joint action $\mathbf{u}^t=(u_1^t,\dots,u_n^t)$ yields the next state via
$P(s^{t+1}\!\mid s^t,\mathbf{u}^t)$ and a shared reward $R(s^t,\mathbf{u}^t)$.
Over an episode of length $T$ with discount $\gamma\in[0,1)$, the goal is to maximize the discounted return; equivalently, learn a joint action-value $Q_{\mathrm{tot}}(s,\mathbf{u})=\mathbb{E}[\sum_{t=0}^{T}\gamma^t R(s^t,\mathbf{u}^t)\mid s^0\!=\!s,\mathbf{u}^0\!=\!\mathbf{u}]$, typically factored from per-agent utilities via monotonic mixing~\cite{DBLP:conf/icml/RashidSWFFW18}.

\subsection{Group-Aware Coordination Graphs (GACG)}
\label{sec:GACG}

GACG learns a \emph{Gaussian} distribution over all $n^2$ candidate edges,
\begin{equation}
\label{eq:gacg_gauss_final}
\mathrm{vec}(\mathcal{E}^{t}) \;\sim\; \mathcal{N}\!\big(\boldsymbol{\mu}^{t},\,\widehat{\mathbf{M}}^{t}\big),
\end{equation}
where the mean $\boldsymbol{\mu}^{t}\!\in\!\mathbb{R}^{n^2}$ encodes agent-pair importance: embeddings $\hat{o}_i^{\,t}=f_{oe}(o_i^t)$ are scored as $\mu_{ij}^{t}=f_{ap}(\hat{o}_i^{\,t},\hat{o}_j^{\,t})$ and vectorized into $\boldsymbol{\mu}^{t}=\mathrm{vec}([\mu_{ij}^{t}]_{i,j})$; and the covariance $\widehat{\mathbf{M}}^{t}\!\in\!\mathbb{R}^{n^2\times n^2}$ encodes group-level dependence: a time-varying partition $\mathcal{G}^{t}=\{g_1,\dots,g_m\}$ (we drop the time superscript when clear from context), produced as $\mathcal{G}^{t}=f_g(\mathcal{O}^{t-k:t})$, induces a block mask $M_{ij}^{t}=\mathbb{I}[\text{agents }i,j\text{ in same group}]$ and $\widehat{\mathbf{M}}^{t}=\mathrm{vec}(M^{t})\mathrm{vec}(M^{t})^{\top}$, so co-group edges share higher correlated strength.
Further details are given in GACG~\cite{GACG}.
In HIBCG, GACG supplies the initial topology and partition (Stage~1); Stages~2--3 then refine layer-wise adjacency and regulate per-agent message bandwidth (\S\ref{sec:method})---adding the differentiated capacity GACG lacks.

\begin{figure*}[t]
       \centering
       \includegraphics[width=\textwidth]{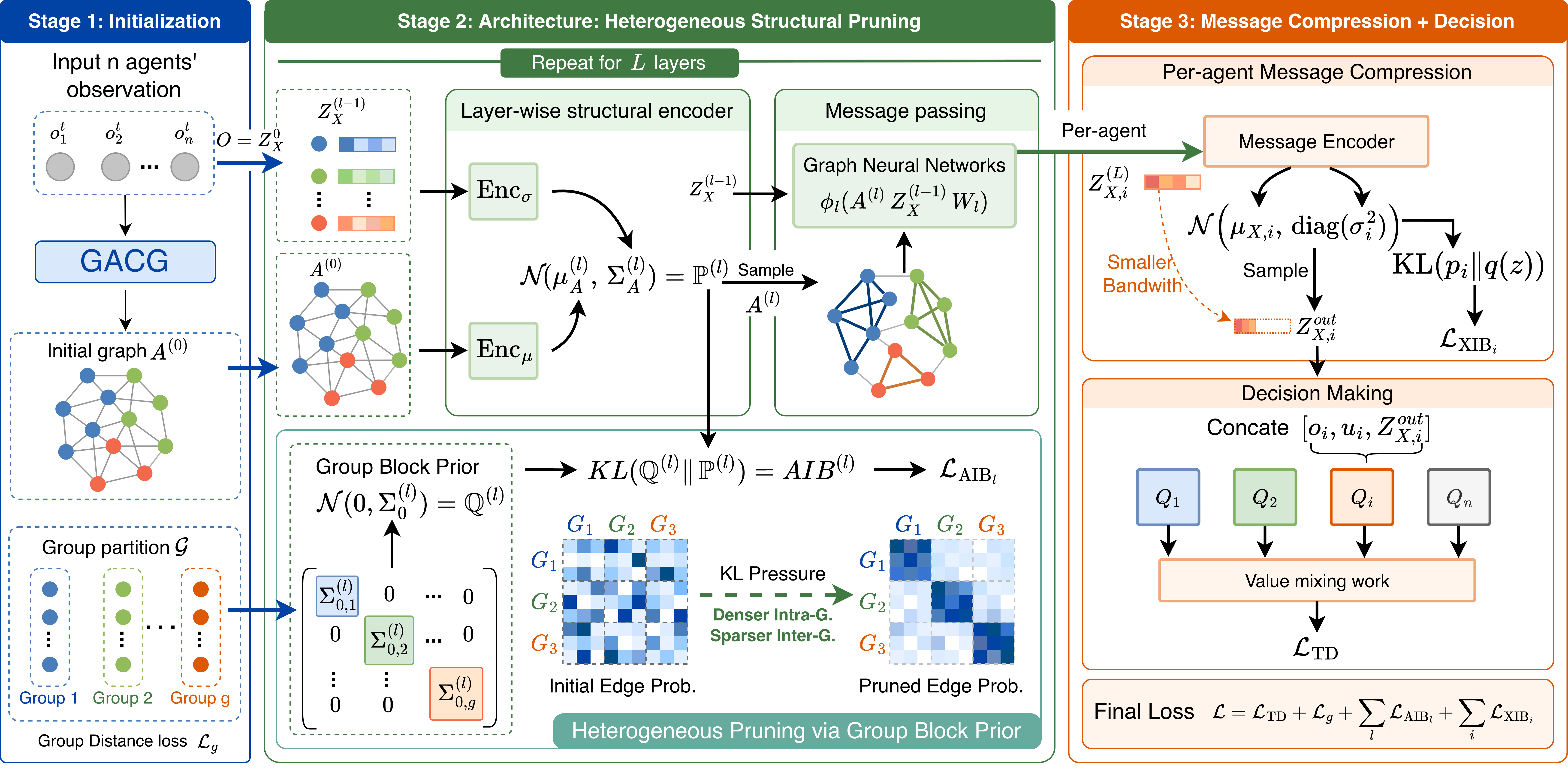}
       \caption{%
         \textbf{HIBCG architecture overview.}
         \textbf{Stage~1 (Initialisation):} GACG produces the initial graph
         $A^{(0)}$ and group partition~$\mathcal{G}$ from agent observations.
         \textbf{Stage~2 (Group-Differentiated Structural Pruning):} At each GNN
         layer~$l$, a layer-wise structural encoder maps $(A^{(0)},\,Z_X^{(l-1)})$
         to Gaussian parameters $\mathbb{P}^{(l)}$; the sampled-and-gated adjacency
         $\widetilde{A}^{(l)}$ gates message passing.
         A group-aligned block-diagonal prior $\mathbb{Q}^{(l)}$ applies
         asymmetric KL pressure: a larger $\sigma_{\mathrm{intra}}^2$ preserves
         dense intra-group edges while a smaller $\sigma_{\mathrm{cross}}^2$
         aggressively prunes inter-group edges (bottom: initial vs.\ pruned edge
         probability).
         \textbf{Stage~3 (Message Compression $+$ Decision):} After all $L$ layers,
         a per-agent message encoder compresses the aggregated representation
         $Z_{X,i}^{(L)}$ into a lower-bandwidth code $Z_{X,i}^{\mathrm{out}}$,
         which is concatenated with local inputs for $Q$-value computation and
         value mixing.
         The final loss combines the TD error, the GACG group-distance loss, and
         the closed-form structural and message KL penalties.
       }
       \label{fig:architecture}
\end{figure*}

\subsection{The Graph Information Bottleneck}
\label{sec:ib-principle}

The \emph{graph information bottleneck} (GIB)~\cite{GIB_Wu} extends the information bottleneck (IB) principle~\cite{IB_Tishby} to graph-structured data. For a forward pass with data $\mathcal{D}=(A,X)$---where $X=[x_1;\dots;x_n]\!\in\!\mathbb{R}^{n\times d_{\mathrm{in}}}$ stacks per-agent input features (each $x_i$ concatenates the local observation $o_i$ with optional auxiliary inputs such as the last action and an agent identifier) and $A$ is any explicit adjacency available at input time---processed by an $L$-layer GNN with task target $Y$ and learned representation $Z$, the GIB objective trades off \emph{relevance}~vs.\ \emph{compression}:
\begin{equation}
\label{eq:gib-objective}
\min_Z\;
-\,I(Y;\,Z)
\;+\;\beta\,I(\mathcal{D};\,Z),
\qquad \beta>0,
\end{equation}
where $I(\cdot\,;\cdot)$ denotes mutual information, the first term encourages $Z$ to be predictive of $Y$, and the second penalises the total information that $Z$ retains about the raw input.
The GNN builds $Z=Z_X^{(L)}$ through successive message-passing layers that produce intermediate latents for both graph structure $\{Z_A^{(l)}\}$ and node features $\{Z_X^{(l)}\}$, and the compression term $I(\mathcal{D};Z_X^{(L)})$ can be upper-bounded via layer-wise variational bounds, yielding practical, closed-form regularisers when the encoder distributions are Gaussian~\cite{GIB_Wu}.

Although we adopt GIB as our theoretical tool, effectively instantiating it for heterogeneous coordination-graph learning is not straightforward. Three design questions remain:
\emph{(Q1)~How should connectivity compression (which edges survive, and with what block-wise density) be separated from message-bandwidth compression (per-agent capacity on the learned graph)?
(Q2)~How should functional group structure inform the structural prior, rather than applying the same isotropic penalty to every edge block?
(Q3)~How should topology learning and message control be coupled during training, rather than optimised independently?}
Section~\ref{sec:method} answers all three:
Q1 via the dual-path architecture (\S\ref{sec:framework}),
Q2 via the group-aligned block-diagonal prior (\S\ref{sec:aib}),
and Q3 via the coupled training objective (\S\ref{sec:xib}).

\section{Method: HIBCG}
\label{sec:method}

This section introduces HIBCG in two parts.
\textbf{Framework and architecture (\S\ref{sec:framework}):} we formulate the heterogeneous graph-learning objective and describe the three-stage architecture previewed in Figure~\ref{fig:architecture}---(Stage~1) GACG produces an initial group-aware graph $A^{(0)}$ and partition $\mathcal{G}$; (Stage~2) Gaussian structural encoders refine the topology with a group-aligned prior; (Stage~3) per-agent encoders compress messages for decentralised decision making. The full training procedure is summarised in Algorithm~\ref{alg:hibcg}.
\textbf{Theoretical Analysis (\S\ref{sec:theory}):} we establish the guarantees that underpin the framework, including the dual-path decomposition, the group-aligned prior, and the relevance bound.

\subsection{Framework and Three-Stage Architecture}
\label{sec:framework}
\label{sec:overview}

Existing coordination-graph learners operate in the \emph{edge-uniform} regime characterised below.

\begin{definition}[Edge-uniform IB coordination graph]
\label{def:flat-ib}
An \emph{edge-uniform IB coordination graph} is defined by a graph-learning loss of the form
\begin{equation}
\label{eq:flat-ib}
\mathcal{L}
= \mathcal{L}_{\mathrm{task}}
+ \beta\,\mathcal{R}(\theta),
\qquad \beta\ge 0,
\end{equation}
where $\mathcal{L}_{\mathrm{task}}$ is the primary training objectives (e.g., TD error and any non-IB graph-learning losses), $\mathcal{R}(\theta)$ is a single, global compression penalty (e.g., a KL on the entire GNN output), and $\beta\!\ge\!0$ is a \emph{uniform} trade-off weight applied identically to all agents and all edges.
\end{definition}

\noindent
Definition~\ref{def:flat-ib} covers both explicit ($\beta>0$, e.g.\ MAGI~\cite{MAGI}) and implicit ($\beta=0$,
 e.g.\ GACG~\cite{GACG}, DCG~\cite{DCG}) edge-uniform graph learners.
We propose instead a \emph{heterogeneous} graph learner that jointly learns heterogeneous connectivity for edges and allocates differentiated message capacity per agent:

\begin{definition}[Heterogeneous IB coordination graph]
\label{def:ibcg}
A \emph{heterogeneous IB coordination graph} is defined by a graph-learning loss of the form
\begin{equation}
\label{eq:hibcg-objective}
\mathcal{L}
= \mathcal{L}_{\mathrm{task}}
+ \underbrace{\sum_{g}\lambda_A^{(g)}\,\mathcal{R}_A^{(g)}(\theta)}_{\text{heterogeneous connectivity}}
+ \underbrace{\sum_{i}\lambda_X^{(i)}\,\mathcal{R}_X^{(i)}(\theta)}_{\text{differentiated message capacity}},
\end{equation}
where the agent partition $\mathcal{G}=\{g_1,\dots,g_m\}$ induces $m^2$ edge blocks $\mathcal{G}'\!\triangleq\!\{g_a\!\times\!g_b: g_a,g_b\!\in\!\mathcal{G}\}$; $\mathcal{R}_A^{(g)}$ is a \emph{per-group} structural compression penalty on edge block $g\!\in\!\mathcal{G}'$; $\mathcal{R}_X^{(i)}$ is a \emph{per-agent} feature compression penalty; and $\lambda_A^{(g)},\lambda_X^{(i)}\!\ge\!0$ are \emph{independent} capacity weights.
An edge-uniform IB graph (Def.~\ref{def:flat-ib}) is recovered by setting $|\mathcal{G}|=1$ and $\lambda_A^{(g)}=\lambda_X^{(i)}=\beta$.
\end{definition}

\noindent
The $\sum_g$ term makes connectivity heterogeneous: each edge block $g\!\in\!\mathcal{G}'$ carries its own structural penalty $\mathcal{R}_A^{(g)}$ and capacity weight $\lambda_A^{(g)}$, so different blocks can be regularised with different compression patterns rather than a single global constraint.
The $\sum_i$ term makes message capacity heterogeneous: each agent $i$ carries its own feature penalty $\mathcal{R}_X^{(i)}$ and capacity weight $\lambda_X^{(i)}$, so agents can be compressed to different bandwidths independently.

The three stages below describe how Eq.~\eqref{eq:hibcg-objective} is realised as concrete network components; the theoretical guarantees underpinning this objective are given in \S\ref{sec:theory}.

\subsubsection{Stage 1: Group-Aware Initial Graph}
\label{sec:init}

GACG (\S\ref{sec:GACG}) takes $O^t=[o_1^t;\dots;o_n^t]\!\in\!\mathbb{R}^{n\times d_{\mathrm{obs}}}$ and produces the initial graph $A^{(0)}$ together with the group partition $\mathcal{G}^t$ by sampling from its learned Gaussian (Eq.~\eqref{eq:gacg_gauss_final}):
\begin{equation}
\label{eq:base-struct}
z_A^{(0)} \sim \mathcal{N}\!\big(\boldsymbol{\mu}^{t},\,\widehat{\mathbf{M}}^{t}\big),
\qquad
A^{(0)}=\mathrm{reshape}(z_A^{(0)}).
\end{equation}
In parallel, the same observations are used to form the initial node features $Z_X^{(0)}=f_{\mathrm{in}}(O^t)\!\in\!\mathbb{R}^{n\times d_0}$, where $f_{\mathrm{in}}$ is a lightweight MLP that maps each agent's observation $o_i^t$ (optionally concatenated with the last action $u_i^{t-1}$ and an agent identifier) into the GNN input space.
Stage~1 therefore yields a group-aware topology $A^{(0)}$ and per-agent features $Z_X^{(0)}$, but a \emph{single} static graph with no per-layer control over structural information flow. Stages~2 and 3 below close that gap by instantiating the two compression terms of Eq.~\eqref{eq:hibcg-objective}.

\subsubsection{Stage 2: Group-Differentiated Structural Pruning}
\label{sec:aib}

Stage~2 refines the static initial graph $A^{(0)}$ and observation-derived features $Z_X^{(0)}$ into a sequence of learned post-gating adjacencies $\{\widetilde{A}^{(l)}\}_{l=1}^{L}$ that gate message passing at each depth.
It comprises three components: a per-layer Gaussian structural encoder, a group-aligned block-diagonal prior, and closed-form per-block KL penalties---collectively the \emph{adjacency information bottleneck} (AIB)---with group-differentiated weights.

\noindent\textbf{Layer-wise structural encoder.}\;
Because each message-passing layer transforms agent representations $Z_X^{(l-1)}\!\mapsto\!Z_X^{(l)}$, the importance of an edge for message passing changes with depth: an edge needed for early exchange may become redundant once its content has been absorbed by neighbours, while new inter-agent dependencies can emerge after multi-hop aggregation. A single shared adjacency cannot capture this depth-varying pattern, so the AIB re-evaluates the structural latent at every layer from $A^{(0)}$ and the current agent features $Z_X^{(l-1)}$ ($Z_X^{(0)}=f_{\mathrm{in}}(O^t)$ from Stage~1):
\begin{equation}
\label{eq:a-encoder}
\mathbb{P}(Z_A^{(l)}\!\mid\!A^{(0)},Z_X^{(l-1)})=\mathcal{N}(\mu_A^{(l)},\Sigma_A^{(l)}),\qquad
\widetilde{A}^{(l)}=g_l(Z_A^{(l)}),
\end{equation}
where $g_l$ is a gating function (e.g.\ sigmoid). Sampling uses the reparameterisation trick $\tilde{z}_A^{(l)}\!=\!\mu_A^{(l)}+\delta\,\sigma_A^{(l)} \odot \varepsilon$ with $\varepsilon\!\sim\!\mathcal{N}(0,I)$ and a noise scale $\delta\!\in\!(0,1]$ that controls the amount of structural exploration. After gating, layer-$l$ node features are computed by graph message passing, $Z_X^{(l)}=\phi_l(\widetilde{A}^{(l)}\,Z_X^{(l-1)}\,W_l)$, where $W_l\!\in\!\mathbb{R}^{d_{l-1}\times d_l}$ is the layer-$l$ trainable linear transform and $\phi_l$ is an element-wise nonlinearity.

\noindent\textbf{Group-aware blocking and prior.}\;
At each layer~$l$, we partition the $n^2$ edge latents into blocks $B_{l,g}\subseteq\{1,\dots,n\}^2$ ($k_{l,g}=|B_{l,g}|$) over $g\!\in\!\mathcal{G}'$ (Def.~\ref{def:ibcg}) and place a zero-mean block-diagonal prior over the structural latent, assigning each block its own covariance $\Sigma_{0,g}^{(l)}$:
\begin{equation}
\label{eq:prior}
\mathbb{Q}(Z_A^{(l)})=\mathcal{N}(0,\,\mathrm{blkdiag}(\Sigma_{0,g}^{(l)})_{g\in\mathcal{G}'}).
\end{equation}
In the default configuration we set $\Sigma_{0,g}^{(l)}=\sigma_{0,g}^2 I$ with $\sigma_{0,g}=\sigma_{\mathrm{intra}}$ for intra-group blocks ($g_a\!=\!g_b$) and $\sigma_{0,g}=\sigma_{\mathrm{cross}}$ for inter-group blocks ($g_a\!\neq\!g_b$), encoding the prior expectation that within-team edges are denser than cross-team links.
Remark~\ref{rem:no-regret} shows that aligning the prior with $\mathcal{G}'$ never loosens the variational bound relative to a flat isotropic prior.

\noindent\textbf{Closed-form group-differentiated pruning.}\;
The block-diagonal structure of $\mathbb{Q}$ (Eq.~\eqref{eq:prior}), together with a diagonal (hence factorising) block encoder, makes each block's KL independent (formalised by Prop.~\ref{prop:group-decomp}). With $\Sigma_{A,g}^{(l)}=\mathrm{diag}(\sigma_1^2,\dots,\sigma_{k_{l,g}}^2)$ and isotropic block prior $\Sigma_{0,g}^{(l)}=\sigma_{0,g}^2 I$ (the setting used in practice):
\begin{equation}
\label{eq:gauss-kl}
\begin{aligned}
\mathrm{KL}\big(
\mathcal{N}(\boldsymbol{\mu}_{A,g}^{(l)},\Sigma_{A,g}^{(l)})
&\,\Vert\,
\mathcal{N}(\mathbf{0},\sigma_{0,g}^2 I)
\big) \\
&\quad= \frac{1}{2}\sum_{d=1}^{k_{l,g}}\!\left(
\frac{\sigma_d^2+\mu_d^2}{\sigma_{0,g}^2}-1+\log\frac{\sigma_{0,g}^2}{\sigma_d^2}
\right),
\end{aligned}
\end{equation}
where $\mu_d^2/\sigma_{0,g}^2$ penalises high-magnitude edges, $\sigma_d^2/\sigma_{0,g}^2$ penalises diffuse encoding, and the log term encourages matching the prior scale. The KL is in closed form and requires no Monte-Carlo estimation; the structural latent itself is still sampled via the reparameterisation trick of Eq.~\eqref{eq:a-encoder}.
The per-block AIB penalty and its per-layer aggregation are:
\begin{align}
\label{eq:aib-hat}
\widehat{\mathrm{AIB}}^{(l,g)}&=\mathbb{E}_{\mathcal{D}}\!\left[\mathrm{KL}\!\left(\mathcal{N}(\boldsymbol{\mu}_{A,g}^{(l)},\Sigma_{A,g}^{(l)})\,\Vert\,\mathcal{N}(0,\Sigma_{0,g}^{(l)})\right)\right], \notag\\
\widehat{\mathrm{AIB}}^{(l)}&=\textstyle\sum_{g\in\mathcal{G}'}\widehat{\mathrm{AIB}}^{(l,g)}.
\end{align}
This is the concrete instantiation of the $\sum_g\mathcal{R}_A^{(g)}$ term (heterogeneous connectivity) in Eq.~\eqref{eq:hibcg-objective}, with $\widehat{\mathrm{AIB}}^{(l,g)}$ estimating the layer-$l$ block penalty $\mathcal{R}_A^{(g)}$.
In the default HIBCG configuration we set $\lambda_A^{(l,g)}=\lambda_A$ for every block, so that group-differentiated compression is supplied entirely by the per-block prior scales $(\sigma_{0,g})_{g\in\mathcal{G}'}$ (i.e.\ $\sigma_{\mathrm{intra}}$ vs.\ $\sigma_{\mathrm{cross}}$). Per-block $\lambda_A^{(l,g)}$ remain available as a higher-flexibility variant; a constrained dual-ascent formulation is given in Appendix~\ref{app:dual-ascent}.

\subsubsection{Stage 3: Per-Agent Message Compression and Decision}
\label{sec:xib}

After the $L$-layer GCN stack produces $Z_X^{(L)}\!\in\!\mathbb{R}^{n\times d_L}$, a per-agent variational encoder---the \emph{feature information bottleneck} (XIB)---compresses it into a low-bandwidth code for decentralised decision making:
\begin{align}
\label{eq:x-encoder}
Z_X^{\mathrm{out}} &= \mu_X + \sigma_X\!\odot\!\varepsilon,\quad\varepsilon\!\sim\!\mathcal{N}(0,I), \notag\\
\mu_X &= h_\mu(Z_X^{(L)}),\quad \log\sigma_X^2 = h_\sigma(Z_X^{(L)}).
\end{align}
Each agent~$i$ concatenates its observation $o_i^t$ and last action $u_i^{t-1}$ with the compressed message $\tanh(Z_{X,i}^{\mathrm{out}})$ and feeds $[o_i^t,\,u_i^{t-1},\,\tanh(Z_{X,i}^{\mathrm{out}})]$ through its $Q$-network; utilities are mixed via QMIX into $Q_{\mathrm{tot}}$.

\noindent\textbf{Per-agent XIB penalty.}\;
The per-agent XIB term uses the same diagonal-Gaussian KL as the AIB (Eq.~\eqref{eq:gauss-kl}), with $\sigma_{0,g}$ replaced by a feature-level prior scale $\sigma_{X,0}$:
\begin{align}
\label{eq:xib-hat}
\widehat{\mathrm{XIB}}^{(i)}&=\mathbb{E}_{\mathcal{D}}\!\left[\mathrm{KL}\!\left(\mathcal{N}(\mu_{X,i},\mathrm{diag}(\sigma^2_{X,i}))\,\Vert\,\mathcal{N}(0,\sigma_{X,0}^2 I)\right)\right], \notag\\
\widehat{\mathrm{XIB}}&=\textstyle\sum_{i=1}^n\widehat{\mathrm{XIB}}^{(i)}.
\end{align}
This is the concrete instantiation of the $\sum_i\mathcal{R}_X^{(i)}$ term (differentiated message capacity) in Eq.~\eqref{eq:hibcg-objective}.
The key difference from AIB is granularity: AIB operates over edge blocks at every layer, while XIB is applied per agent after the GCN stack. This is consistent with Prop.~\ref{prop:dual-path}: message passing $Z_X^{(l)}=\phi_l(\widetilde{A}^{(l)}\,Z_X^{(l-1)}\,W_l)$ is deterministic given $(\widetilde{A}^{(l)},Z_X^{(l-1)})$, so $\mathrm{XIB}^{(l)}=0$ for all $l\!\in\!\{1,\dots,L\}$, and only the stochastic decision-facing code $Z_X^{\mathrm{out}}$ contributes a nonzero message penalty.

\noindent\textbf{Full training loss.}\;
The complete training loss instantiates Eq.~\eqref{eq:hibcg-objective}:
\begin{equation}
\label{eq:td-loss}
\begin{aligned}
\mathcal{L}= \underbrace{\mathbb{E}[(Q_{\mathrm{tot}}-y)^2]+\lambda_g\mathcal{L}_g}_{\mathcal{L}_{\mathrm{task}}}
& +\lambda_A\!\sum_l\widehat{\mathrm{AIB}}^{(l)}+\lambda_X\widehat{\mathrm{XIB}},
\end{aligned}
\end{equation}
where $\mathbb{E}[(Q_{\mathrm{tot}}-y)^2]$ is the TD loss with target $y=r+\gamma\max_{\mathbf{u}'}Q_{\mathrm{tot}}'(s',\mathbf{u}';\theta^-)$, $\mathcal{L}_g$ is the group distance loss of GACG (Eq.~(8))~\cite{GACG}, and $\lambda_A,\lambda_X\!\ge\!0$ are scalar capacity weights on the structural and message penalties (default: $\lambda_A^{(l,g)}=\lambda_A$ and $\lambda_X^{(i)}=\lambda_X$ for all blocks and agents).

\subsection{Theoretical Analysis}
\label{sec:theory}
\label{sec:ib_gacg_obj}

We now establish theoretical guarantees for the HIBCG objective in Eq.~\eqref{eq:hibcg-objective}.
The analysis proceeds in four parts:
(i)~a dual-path decomposition that justifies the AIB/XIB split;
(ii)~a no-regret guarantee for group-aligned structural priors;
(iii)~an additive per-block KL decomposition that enables group-differentiated control;
and (iv)~a relevance bound showing that the TD loss implicitly maximises $I(Y;Z_X^{\mathrm{out}})$.
All proofs are provided in the Appendix.

\subsubsection{Dual-Path Decomposition}
The chain rule of mutual information, applied to $(\mathcal{D};Z_A,Z_X)$ and combined with a layer-wise variational bound, gives an \emph{exact} split of the total information flow.

\begin{proposition}[Dual-Path Decomposition~\cite{GIB_Wu}]
\label{prop:dual-path}
Under two conditional independencies induced by the message-passing recursion---(CI-A): $Z_A^{(l)}\!\perp\!(\mathcal{D},Z_A^{(1:l-1)},Z_X^{(1:l-2)})\!\mid\!(A^{(0)},Z_X^{(l-1)})$; and (CI-X): $Z_X^{(l)}\!\perp\!(\mathcal{D},Z_X^{(1:l-2)},Z_A^{(1:l-1)})\!\mid\!(Z_X^{(l-1)},Z_A^{(l)})$---the joint information flow factorises as
\begin{equation}
\label{eq:chain-rule-decomp}
I(\mathcal{D};Z_A,Z_X)
\;=\;\underbrace{I(\mathcal{D};Z_A)}_{\text{AIB}}
\;+\;\underbrace{I(\mathcal{D};Z_X\mid Z_A)}_{\text{XIB}},
\end{equation}
and each term admits a per-layer variational upper bound:
\begin{equation}
\label{eq:aibxib-vub}
I(\mathcal{D};Z_A)\le\!\sum_{l=1}^{L}\!\mathrm{AIB}^{(l)},\qquad I(\mathcal{D};Z_X\!\mid\!Z_A)\le\!\sum_{l=1}^{L}\!\mathrm{XIB}^{(l)},
\end{equation}
with $\mathrm{AIB}^{(l)}\!\triangleq\!\mathbb{E}[\mathrm{KL}(\mathbb{P}(Z_A^{(l)}\!\mid\!\mathcal{D}_A^{(l)})\Vert\mathbb{Q}(Z_A^{(l)}))]$, where $\mathcal{D}_A^{(l)}\!\triangleq\!(A^{(0)},Z_X^{(l-1)})$ denotes the structural-encoder inputs, and $\mathrm{XIB}^{(l)}$ is defined analogously for a stochastic message encoder at layer~$l$ (in HIBCG, GCN message passing is deterministic, so $\mathrm{XIB}^{(l)}=0$ and only the post-stack code $Z_X^{\mathrm{out}}$ is regularised; see \S\ref{sec:xib}).
\end{proposition}

\noindent
Eq.~\eqref{eq:chain-rule-decomp} is the structural reason HIBCG splits into an AIB path and an XIB path: they are the two terms of a single decomposition, not two parallel regularisers. The conditional structure $I(\mathcal{D};Z_X\!\mid\!Z_A)$ means that message compression is defined \emph{relative to} edge pruning, so tighter AIB pruning forces XIB to do more work on fewer surviving channels.

\subsubsection{No-Regret Property of Group-Aligned Priors}
The AIB penalty is a variational upper bound on the structural information retained at each layer: the better the prior $\mathbb{Q}$ matches the encoder's edge distribution, the tighter this bound, and the more faithfully the penalty reflects true compression.
A single isotropic prior $\mathcal{N}(0,\sigma_0^2 I)$ cannot accommodate different edge scales across group blocks (such as denser within-team links than cross-team links), so any such mismatch leaves the bound unnecessarily loose.
HIBCG's block-diagonal prior (Eq.~\eqref{eq:prior}) removes this limitation without risk:

\begin{remark}[No-Regret Guarantee of Group-Aligned Priors]
\label{rem:no-regret}
Let $\mathcal{Q}_{\mathrm{blk}}
=\big\{\mathcal{N}\!\big(0,\mathrm{blkdiag}(\sigma_{0,g}^{2}\,I_{k_g})_{g\in\mathcal{G}'}\big):
\sigma_{0,g}>0\big\}$
be the family of zero-mean block-diagonal Gaussian priors aligned with $\mathcal{G}'$ (Def.~\ref{def:ibcg}). Since the uniform prior $\mathbb{Q}_{\mathrm{flat}}=\mathcal{N}(0,\sigma_0^2 I)$ is a special case ($\sigma_{0,g}=\sigma_0$ for all $g$), we have $\mathbb{Q}_{\mathrm{flat}}\in\mathcal{Q}_{\mathrm{blk}}$. Therefore any optimally matched group-aligned prior $\mathbb{Q}_{\mathrm{group}}^{\star}\triangleq\arg\min_{\mathbb{Q}\in\mathcal{Q}_{\mathrm{blk}}}\mathrm{KL}(\bar{\mathbb{P}}\Vert\mathbb{Q})$ satisfies
\begin{equation}
\label{eq:tighter-bound}
\begin{aligned}
\mathbb{E}\!\Big[
\mathrm{KL}\!\big(
\mathbb{P}(Z_A^{(l)}\!\mid\!\mathcal{D}_A^{(l)})
&\,\Vert\, \mathbb{Q}_{\mathrm{group}}^{\star}
\big)\Big] \\
&\;\le\;
\mathbb{E}\!\Big[
\mathrm{KL}\!\big(
\mathbb{P}(Z_A^{(l)}\!\mid\!\mathcal{D}_A^{(l)})
\,\Vert\, \mathbb{Q}_{\mathrm{flat}}
\big)\Big],
\end{aligned}
\end{equation}
with equality iff $\sigma_0^2=(\mathrm{tr}(\Sigma_g)+\lVert\bar{\mu}_g\rVert^2)/k_g$ for every block $g$, where $\Sigma_g$ and $\bar{\mu}_g$ are the covariance and mean of the $g$-th block of the aggregate posterior and $k_g$ its dimension. The improvement is strict whenever the per-block second moments differ across groups (proof in Appendix~\ref{app:group-prior-proof}).
\end{remark}

\noindent
This no-regret property holds \emph{unconditionally} for any group partition: the block-diagonal family always contains the uniform prior, so injecting group information adds upside without theoretical risk. Well-separated groups yield a large gap; a homogeneous setting recovers the uniform bound automatically. Our practical fixed-scale instantiation ($\sigma_{\mathrm{intra}}\!\gg\!\sigma_{\mathrm{cross}}$) is motivated by this guarantee; in experiments the bound gap shrinks by $2.2$--$6.5\times$ across heterogeneous benchmarks (Table~\ref{tab:bound-tightening}). A worked $n\!=\!10$ example showing a $>75\%$ reduction is given in Appendix~\ref{app:example-10agents}.

\subsubsection{Per-Block Additive Control}
The block-diagonal prior enables per-group capacity weights $\lambda_A^{(g)}$ because the KL decomposes additively:

\begin{proposition}[Group-Aligned Block Decomposition]
\label{prop:group-decomp}
If $\mathbb{Q}(Z_A^{(l)})=\mathcal{N}\!\big(0,\mathrm{blkdiag}(\Sigma_{0,g}^{(l)})_{g\in\mathcal{G}'}\big)$ and the encoder factorises over $\mathcal{G}'$ as $\mathbb{P}(Z_A^{(l)}\!\mid\!\mathcal{D}_A^{(l)})=\prod_{g\in\mathcal{G}'}\mathbb{P}(Z_{A,g}^{(l)}\!\mid\!\mathcal{D}_A^{(l)})$, then
\begin{align}
\label{eq:group-decomp}
\mathrm{AIB}^{(l)}&=\sum_{g\in\mathcal{G}'}\mathrm{AIB}^{(l,g)}, \notag\\
\mathrm{AIB}^{(l,g)}&\triangleq\mathbb{E}_{\mathcal{D}_A^{(l)}}\!\left[\mathrm{KL}\!\left(\mathbb{P}(Z_{A,g}^{(l)}\!\mid\!\mathcal{D}_A^{(l)})\,\Vert\,\mathbb{Q}(Z_{A,g}^{(l)})\right)\right],
\end{align}
each summand being a closed-form Gaussian KL (Eq.~\eqref{eq:gauss-kl}). Proof: KL additivity, Appendix~\ref{app:block-decomp-proof}.
\end{proposition}

\noindent
Setting $\sigma_{\mathrm{intra}}\!\gg\!\sigma_{\mathrm{cross}}$ therefore enforces sparse inter-team links while preserving dense intra-team connectivity---the group-differentiated edge pruning that defines HIBCG.

\subsubsection{Relevance via TD Loss}
The GIB objective Eq.~\eqref{eq:gib-objective} also has a relevance term $-I(Y;Z_X^{\mathrm{out}})$, which is in general intractable. The next proposition---the main theoretical result---shows this term is implicitly handled by the standard TD loss, so HIBCG requires no separate mutual-information estimator. In our architecture the decision-facing representation is the compressed code $Z_X^{\mathrm{out}}$ (Stage~3).

\begin{proposition}[Relevance via Q-Value Optimisation]
  \label{prop:relevance}
  Let $Y=(a_1^*,\dots,a_n^*)$ denote the optimal joint action,
  $K=\prod_i|\mathcal{U}_i|$ the joint action-space size, and assume:
  \begin{enumerate}
  \item[(A1)] QMIX-style monotonic value decomposition with greedy action
  selection $a_i=\arg\max_{u_i}Q_i(\tau_i,u_i,Z_{X,i}^{\mathrm{out}})$;
  \item[(A2)] the target network has converged, so that
  $\mathcal{L}_{\mathrm{TD}}=\mathbb{E}[e^2]$ with
  $e(s,\mathbf{u})=Q_{\mathrm{tot}}(s,\mathbf{u})-Q^*(s,\mathbf{u})$;
  \item[(A3)] the minimum action-value gap
  $\Delta_{\min}\triangleq\min_{s,\,\mathbf{u}\neq Y}\!
  (Q^*(s,Y)-Q^*(s,\mathbf{u}))>0$.
  \end{enumerate}
  Then, with $c\triangleq\Delta_{\min}^2/(4K)>0$,
  \begin{equation}
  \label{eq:relevance-bound}
  I(Y;\,Z_X^{\mathrm{out}})
  \;\ge\;
  H(Y) \;-\; h\!\Big(\tfrac{\mathcal{L}_{\mathrm{TD}}}{c}\Big)
       \;-\; \tfrac{\mathcal{L}_{\mathrm{TD}}}{c}\,\log(K\!-\!1),
  \end{equation}
  whenever $\mathcal{L}_{\mathrm{TD}}/c\!\le\!(K\!-\!1)/K$ (vacuous otherwise),
  where $h(p)=-p\log p-(1\!-\!p)\log(1\!-\!p)$ is the binary entropy.
\end{proposition}
\begin{proof}
Define $e(s,\mathbf{u})\triangleq Q_{\mathrm{tot}}(s,\mathbf{u})-Q^*(s,\mathbf{u})$
so that $\mathcal{L}_{\mathrm{TD}}=\mathbb{E}[e^2]$.
By Markov's inequality applied to the average state-level error, the
action prediction error satisfies
$P_e\le\mathcal{L}_{\mathrm{TD}}/c$.
Fano's inequality then gives
$H(Y\mid Z_X^{\mathrm{out}})\le h(P_e)+P_e\log(K\!-\!1)$.
Substituting the $P_e$ bound and using
$I(Y;Z_X^{\mathrm{out}})=H(Y)-H(Y\mid Z_X^{\mathrm{out}})$
yields Eq.~\eqref{eq:relevance-bound}. Full proof in Appendix~\ref{app:relevance-proof}.
\end{proof}

\noindent
Both penalty terms on the right-hand side of Eq.~\eqref{eq:relevance-bound} vanish as $\mathcal{L}_{\mathrm{TD}}\!\to\!0$, so minimising the TD loss drives the lower bound towards $H(Y)$ and guarantees that $Z_X^{\mathrm{out}}$ retains task-relevant information---no separate mutual-information estimator is needed.
The constant $c=\Delta_{\min}^2/(4K)$ makes the dependence on the environment explicit: larger joint action spaces (larger~$K$) shrink~$c$ and amplify both penalty terms, while a larger optimal-action gap $\Delta_{\min}$ enlarges~$c$ and tightens the bound.
For small TD error the gap to $H(Y)$ therefore scales roughly as $\mathcal{L}_{\mathrm{TD}}\!\cdot\!K\log K/\Delta_{\min}^2$ (Appendix~\ref{app:relevance-Kscaling}), so a fixed TD loss purchases progressively less relevance guarantee as the team grows---consistent with the widening HIBCG--QMIX gap on the MAgent $n$-sweep (\S\ref{sec:exp-scale}).
The role of AIB and XIB is therefore to distribute a limited compression budget across structural blocks and agent messages, while the standard TD loss handles relevance.

\begin{algorithm}[t]
\caption{HIBCG Training (one gradient step)}
\label{alg:hibcg}
\begin{algorithmic}[1]
\REQUIRE Observations $\{o_i^t\}_{i=1}^n$; replay batch $\mathcal{B}$;
         layer count $L$; scalar capacity weights $\lambda_A,\lambda_X$
         (warmed up over $T_{\mathrm{warm}}$ steps);
         group-conditional prior scales $\sigma_{\mathrm{intra}},\sigma_{\mathrm{cross}}$;
         truncation horizon $T_{\mathrm{trunc}}$; noise scale $\delta\!\in\!(0,1]$
\STATE \textbf{Encode observations:}
       $O^t \leftarrow [o_1^t;\dots;o_n^t]$
       \hfill\COMMENT{stack all agent observations}
\STATE \textbf{Base GACG sample} from $O^t$
       (Eq.~\eqref{eq:gacg_gauss_final}):
       compute $\boldsymbol{\mu}^{t}$ via attention on $O^t$;
       \textbf{if} $t\!\geq\!T_{\mathrm{trunc}}$:
       compute group mask $M^t$ from recent trajectory;
       set $\Sigma_A^{t}\!=\!\alpha\,\widehat{\mathbf{M}}^{t}+\varepsilon I$;
       sample $z_A^{(0)}\!\sim\!\mathcal{N}(\boldsymbol{\mu}^{t},\Sigma_A^{t})$;
       \textbf{else}: sample $z_A^{(0)}$ via
       $\mathrm{RelaxedBernoulli}(\boldsymbol{\mu}^{t})$;
       symmetrize and normalize to obtain $A^{(0)}$
\STATE Initialize $Z_X^{(0)} \leftarrow f_{\mathrm{in}}(O^t)$
       \hfill\COMMENT{MLP maps observations to GNN input features}
\FOR{$l = 1$ \TO $L$}
  \STATE \textbf{Structure encoder:}
         $(\mu_A^{(l)},\log\sigma_A^{2(l)}) \leftarrow f_A^{(l)}(A^{(0)}, Z_X^{(l-1)})$;
         sample $\tilde{z}_A^{(l)}\!\leftarrow\!\mu_A^{(l)}
         +\delta\,\sigma_A^{(l)}\!\odot\!\varepsilon$,
         $\;\varepsilon\!\sim\!\mathcal{N}(0,I)$
  \STATE \textbf{Gate edges:}
         symmetrize and normalize $\tilde{z}_A^{(l)}$;
         $\widetilde{A}^{(l)} \leftarrow g_l(\tilde{z}_A^{(l)})$
         \hfill\COMMENT{e.g.\ sigmoid or hard threshold}
  \STATE \textbf{Message passing:}
         $Z_X^{(l)} \leftarrow \phi_l\!\left(\widetilde{A}^{(l)}\,Z_X^{(l-1)}\,W_l\right)$
  \STATE Accumulate $\widehat{\mathrm{AIB}}^{(l)}$
         (Eq.~\eqref{eq:aib-hat}) using $(\mu_A^{(l)},\sigma_A^{2(l)})$
         and group-conditional prior
         $\sigma_0^2\!=\!\sigma_{\mathrm{intra}}^2$ (intra-group),
         $\sigma_{\mathrm{cross}}^2$ (inter-group)
\ENDFOR
\STATE \textbf{Feature encoder (XIB):}
       $(\mu_X,\sigma_X^{2}) \leftarrow
       (h_\mu(Z_X^{(L)}),\;\exp h_\sigma(Z_X^{(L)}))$;
       $\;Z_X^{\mathrm{out}} \leftarrow \mu_X + \sigma_X \odot \varepsilon$,
       $\;\varepsilon\!\sim\!\mathcal{N}(0,I)$
\STATE Compute $\widehat{\mathrm{XIB}}$
       (Eq.~\eqref{eq:xib-hat}) using $(\mu_X,\sigma_X^{2})$
\STATE \textbf{Agent input:} for each agent $i$, concatenate $[o_i^t,\,u_i^{t-1},\;\tanh(Z_{X,i}^{\mathrm{out}})]$;
       forward through agent network to obtain $Q$-values
\STATE \textbf{Value mixing:} compute $Q_{\mathrm{tot}}(s,\mathbf{u};\theta)$
       and target $y$ (Eq.~\eqref{eq:td-loss})
\STATE \textbf{Form loss:}
       $\mathcal{L} \leftarrow \mathcal{L}_{\mathrm{task}}
       + \lambda_A\!\cdot\!\textstyle\sum_{l}\widehat{\mathrm{AIB}}^{(l)}
       + \lambda_X\!\cdot\!\widehat{\mathrm{XIB}}$,
       where $\mathcal{L}_{\mathrm{task}}=\mathcal{L}_{\mathrm{TD}}+\lambda_g\mathcal{L}_g$
       (Eq.~\eqref{eq:td-loss})
\STATE Backpropagate $\nabla_\theta\mathcal{L}$; update $\theta$; update
       target $\theta^-$ (Polyak or periodic copy)
\STATE \textbf{(Optional)} Dual ascent: update $\lambda_A,\lambda_X$
       (Appendix~\ref{app:dual-ascent})
\end{algorithmic}
\end{algorithm}

\section{Experiments}
\label{sec:experiments}

We evaluate HIBCG on SMACv1, SMACv2, and MAgent against six external and three internal baselines, asking: (Q1)~does HIBCG learn better coordination graphs than existing sparse-graph and IB-based methods (\S\ref{sec:exp-overall})? (Q2)~which components contribute (\S\ref{sec:exp-ablation})? (Q3)~do the propositions hold empirically (\S\ref{sec:exp-theory})? (Q4)~does the method scale to 100 agents (\S\ref{sec:exp-scale})? (Q5)~what does the learned graph look like (\S\ref{sec:exp-info}--\ref{sec:exp-visualization})?

\subsection{Experimental Setup}
\label{sec:exp-setup}

\noindent\textbf{Environments.}\;
We test on three benchmarks that span fixed and stochastic team compositions and scale up to 100 agents (full map list in Appendix Table~\ref{tab:envs}). \emph{SMACv1}~\cite{SMAC} uses fixed team compositions; we select two heterogeneous maps (3s5z, 1c3s5z), one role-rich three-type map (MMM2: marines/marauders/medivacs), and one homogeneous map (25m) as a negative control. \emph{SMACv2}~\cite{DBLP:conf/nips/EllisCMSSMFW23} extends the benchmark with stochastic compositions sampled via \texttt{weighted\_teams}; we evaluate on protoss\_8v8 (2-type) and terran\_10v10 (3-type) to probe HIBCG under distribution shift. \emph{MAgent Battle}~\cite{MAgent} provides spatial battles at $n\in\{36,64,100\}$ (up to $n^2{=}10{,}000$ candidate edges).

\noindent\textbf{Baselines.}\;
We compare against six external methods spanning value decomposition (\textbf{QMIX}~\cite{DBLP:conf/icml/RashidSWFFW18}), learned communication graphs (\textbf{CommFormer}~\cite{CommFormer}, ICLR'24), fixed-topology communication (\textbf{ExpoComm}~\cite{ExpoComm}, ICLR'25), information-bottleneck communication (\textbf{MAGI}~\cite{MAGI}, TPAMI'24---our primary edge-uniform IB reference, \S\ref{sec:method-discussion}), and our two earlier works (\textbf{GACG}~\cite{GACG}, group-aware topology with hard thresholds; \textbf{BVME}~\cite{BVME}, message-only edge-uniform IB). All graph-based methods use the same QMIX mixing network. ExpoComm uses
its official hyperparameter recipe (the bimodal MMM2 outcome reported in
\S\ref{sec:exp-overall} is reproducible under that recipe; we did not
re-tune). We focus on the strongest published graph-learning and IB-based
comparators on each axis (sparse graph: CommFormer; fixed topology:
ExpoComm; IB messaging: MAGI; group-aware: GACG; message IB: BVME);
generic policy-gradient (HAPPO) and dense-attention / variance-based
graph methods (DICG, CASEC) are not included in these tables, with
results in their original papers~\cite{DICG,CASEC,HAPPO} for reference.

\noindent\textbf{HIBCG variants.}\;
To isolate each component, we evaluate:
\begin{itemize}
\item \textbf{AIB-only}: structural IB with learned Gumbel-softmax
      sparsification; XIB disabled ($\lambda_X{=}0$).
\item \textbf{XIB-only}: message IB on GACG's hard-threshold
      graph; AIB disabled ($\lambda_A{=}0$). This coincides with our
      published BVME baseline~\cite{BVME}; we report it as BVME in
      Overall Performance and as XIB-only in the ablation.
\item \textbf{HIB-flat}: both AIB and XIB enabled with a flat
      (isotropic) prior $\sigma_0{=}0.01$ on all edges.
\item \textbf{HIBCG} (full): AIB + XIB + group-conditional
      block-diagonal prior ($\sigma_{\mathrm{intra}}{>}\sigma_{\mathrm{cross}}$).
\end{itemize}

\noindent\textbf{Implementation details.}\;
\label{sec:impl-details}
\label{sec:interaction}
HIBCG is built on EPyMARL~\cite{papoudakis2021benchmarking}.
Capacity weights $\lambda_A$ and $\lambda_X$ are warmed up linearly over $T_{\mathrm{warm}}$ steps; group-differentiated allocation is realised via the prior ($\sigma_{\mathrm{intra}}^2\!>\!\sigma_{\mathrm{cross}}^2$).
IB regularisers add ${\approx}3.8\%$ wall-clock overhead (MAgent-64).
Empirically, AIB and XIB trade off under a roughly conserved total flow $\sum_l\mathrm{AIB}^{(l)}+\mathrm{XIB}$ at fixed task quality.
Hyperparameters and the centralised-graph / decentralised-action execution model are deferred to Appendix~\ref{app:impl-details}--\ref{app:execution-model}.
Each configuration uses at least $5$ random seeds ($6$--$7$ on high-variance MAgent); all runs are included, with per-cell counts in Appendix Table~\ref{tab:magent}.
Claimed gaps in \S\S\ref{sec:exp-overall}--\ref{sec:exp-ablation} are significant at $p{<}0.05$ under a paired bootstrap over seeds ($10^4$ resamples), unless flagged otherwise.

\begin{figure*}[t]
\centering
\includegraphics[width=\textwidth]{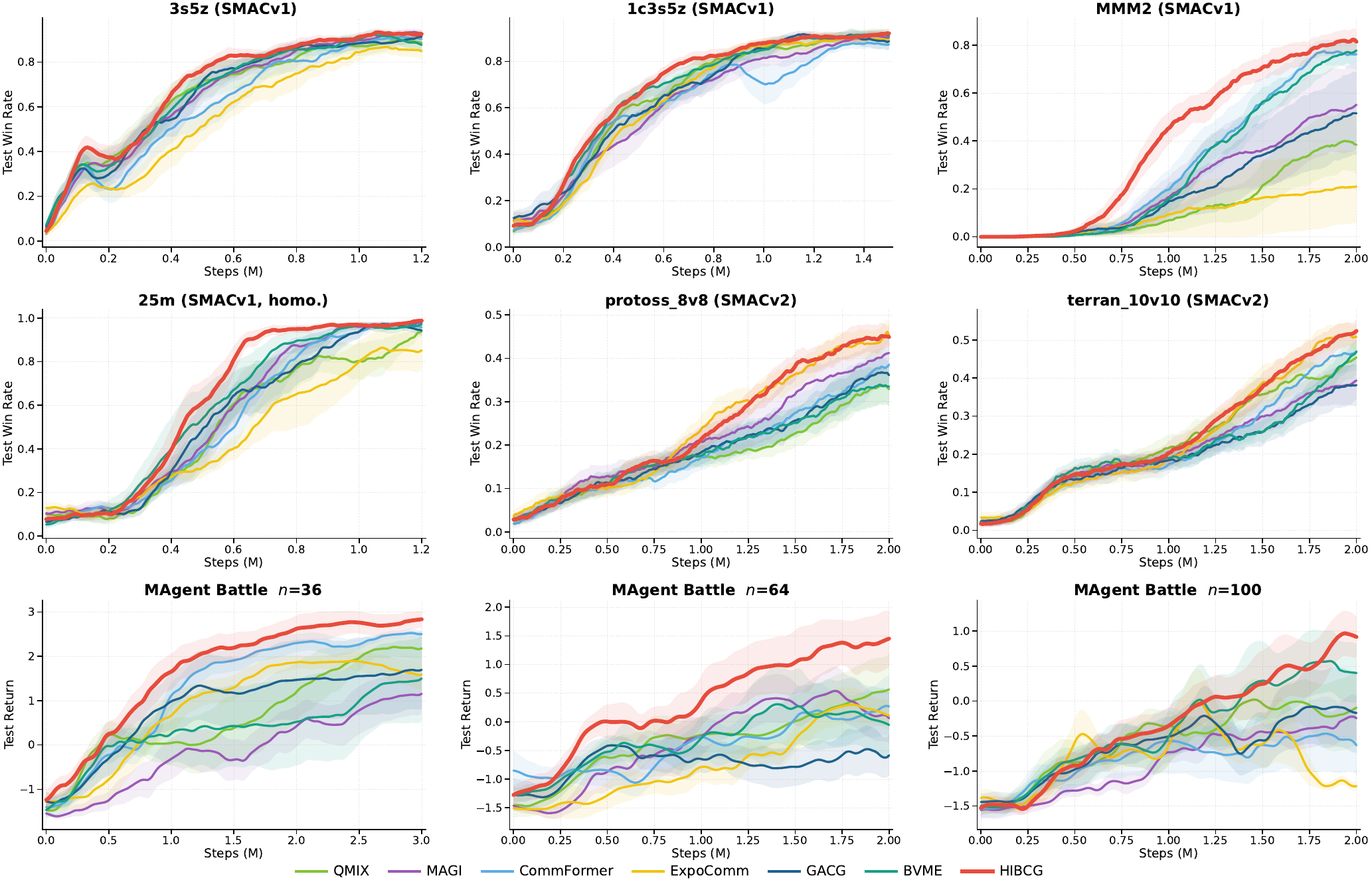}
\caption{\textbf{Overall performance comparison across nine scenarios}
spanning SMACv1 (3s5z, 1c3s5z, MMM2, 25m), SMACv2 (protoss\_8v8,
terran\_10v10) and MAgent Battle at three scales ($n\in\{36,64,100\}$).
Curves show mean over seeds; shaded bands are $\pm 0.5\sigma$.
We compare HIBCG against six external baselines (QMIX, MAGI,
CommFormer, ExpoComm, GACG, and BVME---our published message-IB prior
work), and include HIB-flat as an internal reference curve.
HIBCG (red) is either the clear winner or matches the ceiling on every
map, with the largest margins on heterogeneous (MMM2, protoss\_8v8,
terran\_10v10) and scale-stressed (MAgent $n{\ge}64$) scenarios.
25m is our homogeneous negative-control: HIBCG does not win, as
predicted by Remark~\ref{rem:no-regret}.
}
\label{fig:learning-curves}
\end{figure*}

\subsection{Overall Performance}
\label{sec:exp-overall}

All values are tail-10\% means (average over the final $10\%$ of training-time test evaluations) over $\ge\!5$ seeds; per-configuration seed counts are listed in the corresponding Appendix tables. Full per-map numbers are in Appendix Tables~\ref{tab:smacv1}--\ref{tab:magent}; learning curves in Figure~\ref{fig:learning-curves}.

\noindent\textbf{SMACv1.}\;
On the role-rich \textbf{MMM2} (3 unit types), HIBCG achieves \textbf{81.4\%} WR, $+4.0$\,pp over CommFormer (77.4\%) and $+27$\,pp over HIB-flat (54.2\%). Near-saturated maps (3s5z, 1c3s5z) show all methods within 2\,pp---the ceiling regime. On the \textbf{homogeneous 25m}, HIBCG ties HIB-flat ($97.2\%$ vs.\ $96.5\%$), confirming that the group prior adds no cost when Gdistance is trivial (Remark~\ref{rem:no-regret}). Notably, ExpoComm performs poorly on MMM2 ($0.20\pm0.31$; only $1/5$ seeds converge) under its official hyperparameter setting; this is consistent with the difficulty of using a fixed exponential topology in role-sensitive scenarios where specific cross-role interactions (e.g.\ medivac--marauder healing) are critical. HIBCG's learned graph converges on $5/5$ seeds.

\noindent\textbf{SMACv2.}\;
Under stochastic team compositions, HIBCG achieves the highest final WR on \textbf{protoss\_8v8} ($0.448\pm0.022$) and the best AUC on \textbf{terran\_10v10} ($0.246\pm0.017$). The realised cross/intra AIB-loss ratio collapses to ${\approx}1{\times}$--$9.5{\times}$ (vs.\ $40{\times}$--$906{\times}$ on fixed-composition benchmarks), showing that the learned posterior allocates information more evenly across group blocks when group assignments are stochastic (\S\ref{sec:exp-theory}).

\noindent\textbf{MAgent Battle ($n\!\in\!\{36,64,100\}$).}\;
\textbf{HIBCG leads on all metrics at every scale.} At $n{=}36$: $\mathbf{2.74}\pm0.43$ vs.\ CommFormer $2.50$, QMIX $2.19$. At $n{=}64$: peak return $2.13$, $49\%$ above CommFormer; GACG collapses. At $n{=}100$: HIBCG is the only method with positive mean ($0.83\pm0.54$) and peak above $1.5$. The cross/intra AIB ratio is ${\approx}906{\times}$ at all scales, enforcing heterogeneous compression as an architectural invariant (Table~\ref{tab:bound-tightening}).

\subsection{Ablation Study: Component Decomposition}
\label{sec:exp-ablation}

We systematically evaluate the contribution of each HIBCG component via
a five-level decomposition on three representative maps that jointly
span the coordination spectrum: 3s5z (SMACv1, heterogeneous 2-type,
near-saturated), MMM2 (SMACv1, heterogeneous 3-type, role-rich), and
MAgent-36 (spatial, scale-stressed).
All ablation variants share the same backbone, replay buffer, and
training schedule; only the IB components are toggled.

\noindent\textbf{Heterogeneous maps: full stack wins.}\;
Table~\ref{tab:ablation-component} reports per-component numbers; Figure~\ref{fig:ablation-component} shows learning curves.
On heterogeneous MMM2 and scale-stressed MAgent-36, the ordering is
$
\text{HIBCG} > \text{HIB-flat} \ge \text{AIB-only}> \text{XIB-only} > \text{GACG} > \text{QMIX},
$
as predicted by the dual-path chain rule (Prop.~\ref{prop:dual-path}).
On \textbf{MMM2} (3 unit types), HIBCG achieves a dominant
\textbf{81.4\%} win rate versus HIB-flat 54.2\%, because the
$g{=}3$ prior aligns with the marine/marauder/medivac role split,
reducing the AIB loss by \textbf{3.8$\times$} (33.3 $\to$ 8.7, see
Table~\ref{tab:bound-tightening}) and concentrating compression
pressure on \emph{cross-role} edges (cross/intra ratio $607\times$).
On \textbf{MAgent-36}, HIBCG achieves $2.74\pm 0.43$ return versus
HIB-flat $1.72\pm 1.36$---a $+$1.02 absolute gain with $3\times$
lower variance.

\noindent\textbf{Homogeneous negative control (25m).}\;
On the homogeneous 25m map (25 identical marines, Gdistance $\approx 1.6$--$2.1$), HIBCG scores $97.2\pm 0.8\%$---comparable to HIB-flat ($96.5\pm 2.4\%$) and above AIB-only ($94.3\pm 3.6\%$). This matches the behaviour Remark~\ref{rem:no-regret} predicts at the prior-family level: when the learned partition is trivial, HIBCG provides no meaningful advantage over HIB-flat in this homogeneous setting. Practitioners with no clear role structure may therefore prefer HIB-flat as a simpler variant. The reversed-prior stress test in \S\ref{sec:exp-theory} (Table~\ref{tab:sigma-ablation}) extends this to a two-sided falsifier: prior direction matters \emph{only} where the partition is non-trivial.

\begin{table}[t]
\centering\small
\caption{Component ablation on three representative maps.
Mean $\pm$ std over 5 seeds. Best in \textbf{bold}.}
\label{tab:ablation-component}
\renewcommand{\arraystretch}{1.15}
\setlength{\tabcolsep}{3pt}
\begin{tabular}{@{}lccc@{}}
\toprule
\textbf{Config}
  & \textbf{3s5z} & \textbf{MMM2} & \textbf{MAgent-36} \\
\midrule
QMIX (no graph)          & $0.897\pm0.062$ & $0.386\pm0.25$  & $2.19\pm0.37$ \\
GACG (hard thr.)         & $0.897\pm0.020$ & $0.488\pm0.33$  & $1.63\pm1.79$ \\
XIB-only (msg.\ IB)      & $0.896\pm0.033$ & $0.755\pm0.08$  & $1.43\pm1.85$ \\
AIB-only (struct.\ IB)   & $0.926\pm0.021$ & $0.502\pm0.18$  & $1.87\pm0.84$ \\
HIB-flat (AIB+XIB)       & $0.908\pm0.034$ & $0.542\pm0.38$  & $1.72\pm1.36$ \\
\textbf{HIBCG (full)}
  & $\mathbf{0.926\pm0.018}$
  & $\mathbf{0.814\pm0.06}$
  & $\mathbf{2.74\pm0.43}$ \\
\bottomrule
\end{tabular}
\end{table}

\begin{figure*}[t]
\centering
\includegraphics[width=\textwidth]{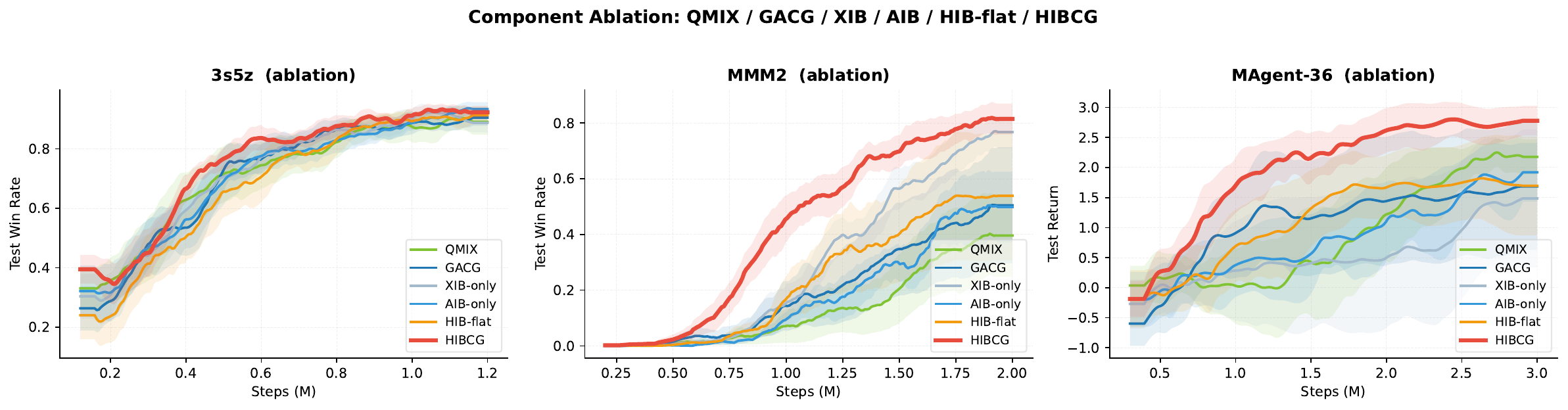}
\caption{\textbf{Component ablation learning curves} on 3s5z, MMM2,
and MAgent-36.  On heterogeneous maps the gap
between HIBCG and HIB-flat is large and sustained; on 2-type
near-saturated 3s5z all variants cluster at the ceiling.}
\label{fig:ablation-component}
\end{figure*}

\subsection{Theory-Aligned Empirical Diagnostics}
\label{sec:exp-theory}

We provide empirical evidence consistent with each theoretical proposition.

\noindent\textbf{Proposition~\ref{prop:dual-path} (Dual-path decomposition).}\;
Table~\ref{tab:ablation-component} supports the dual-path structure: on MMM2 $\mathrm{HIBCG}\,(0.814) > \text{XIB-only}\,(0.755) > \text{AIB-only}\,(0.502)$, and on MAgent-36 $\mathrm{HIBCG}\,(2.74) > \text{AIB-only}\,(1.87) > \text{XIB-only}\,(1.43)$.
Neither single-path variant matches HIBCG on any heterogeneous map,
confirming that structural and message compression provide
\emph{complementary}, not redundant, benefits.
The graph-density evolution in Appendix Figure~\ref{fig:mechanism}(a)
provides a mechanistic diagnostic: AIB-only exhibits a characteristic
``warm-then-prune'' trajectory, and HIBCG reaches an even lower asymptotic
density when XIB is active---direct evidence that message compression raises
per-edge information value and permits more aggressive structural pruning.

\noindent\textbf{Remark~\ref{rem:no-regret} (Group prior tightens bound).}\;
Table~\ref{tab:bound-tightening} compares the tail-10\% mean AIB loss
of HIB-flat vs.\ HIBCG across all scenarios.  The group-conditional
prior yields a \emph{strict} reduction of the variational bound on
every heterogeneous scenario, with tightening ratios ranging from
$2.9\times$ to $6.5\times$.

\begin{table}[t]
\centering\small
\caption{Bound tightening across scenarios: tail-10\% mean loss$_\text{AIB}$
for HIB-flat vs.\ HIBCG, and the learned cross/intra AIB-loss ratio
inside HIBCG, evidencing heterogeneous compression.
}
\label{tab:bound-tightening}
\renewcommand{\arraystretch}{1.15}
\setlength{\tabcolsep}{4pt}
\begin{tabular}{@{}lccccc@{}}
\toprule
\textbf{Scenario} & HIB-flat & HIBCG & Ratio
                  & \multicolumn{2}{c}{\textbf{Cross/Intra}} \\
\cmidrule(l){5-6}
                  & loss$_A$ & loss$_A$ &
                  & ratio & (group $g$) \\
\midrule
3s5z              & $26.09$ & $6.02$   & $4.3\times$  & $41\times$   & $g{=}2$ \\
MMM2              & $33.27$ & $8.74$   & $3.8\times$  & $607\times$  & $g{=}3$ \\
protoss\_8v8      & $23.07$ & $3.56$   & $6.5\times$  & $0.9\times$  & $g{=}2$ (var.) \\
terran\_10v10     & $26.03$ & $4.09$   & $6.4\times$  & $9.5\times$  & $g{=}3$ (var.) \\
MAgent-36         & $8.06$  & $2.74$   & $2.9\times$  & $908\times$  & $g{=}2$ \\
MAgent-64         & $25.46$ & $8.87$   & $2.9\times$  & $905\times$  & $g{=}2$ \\
MAgent-100        & $62.17$ & $28.13$ & $2.2\times$   & $906\times$  & $g{=}2$ \\
\bottomrule
\end{tabular}
\end{table}

\noindent
Three structural findings emerge.
\textbf{(i)~Tightening is uniform and structurally graded.} The group prior strictly tightens the AIB loss on every heterogeneous scenario ($2.9$--$6.5\times$), and the tightening ratio grows with structural diversity (SMACv2 $\approx 6.5\times$ $>$ MMM2 $3.8\times$ $>$ MAgent $2.9\times$), exactly as Remark~\ref{rem:no-regret} predicts.
\textbf{(ii)~Architectural invariance.} The cross/intra AIB-loss ratio on MAgent is $908\!/\!905\!/\!906\times$ at $n\in\{36,64,100\}$---constant across scales \emph{and} across converging vs.\ failing seeds, so differential compression is a structural property of the prior, not a reward-driven artefact.
\textbf{(iii)~Adaptive posterior allocation on SMACv2.} The prior scales $(\sigma_{\mathrm{intra}},\sigma_{\mathrm{cross}})$ are fixed, but the \emph{realised} cross/intra AIB-loss ratio (a function of the learned posterior block scales) \emph{collapses} to $0.9\times$ (protoss\_8v8) and $9.5\times$ (terran\_10v10) on stochastic-composition SMACv2, versus $41$--$908\times$ on fixed-composition maps. When group assignments vary per episode, the learned posterior allocates information more evenly across group blocks.

\begin{figure*}[t]
\centering
\includegraphics[width=\textwidth]{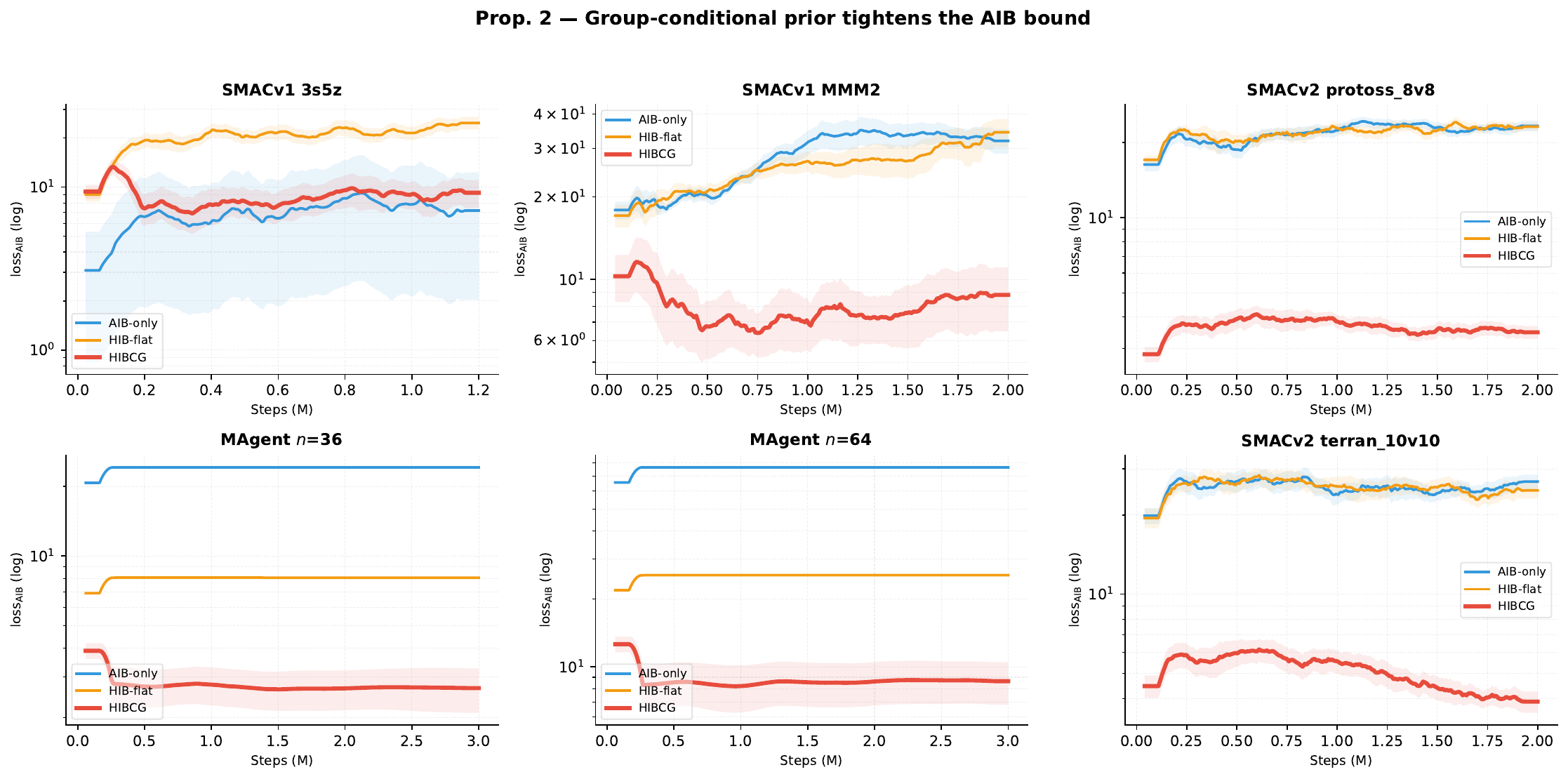}
\caption{\textbf{AIB bound tightening (Remark~\ref{rem:no-regret})
across six scenarios.}  Curves show the training trajectory of
loss$_\text{AIB}$ (log scale); lower is tighter.
HIBCG (red) attains a strictly lower AIB loss than HIB-flat (orange)
and AIB-only (blue) on every scenario.  Tightening ratios range from
$2.9\times$ (MAgent) to $6.5\times$ (SMACv2), matching Table~\ref{tab:bound-tightening}.}
\label{fig:aib-tightening}
\end{figure*}

\noindent\textbf{Prop.~\ref{prop:group-decomp} and Remark~\ref{rem:no-regret} (Per-group control and reversed-prior falsifier).}\;
A $\sigma$-sweep on 3s5z (Appendix) confirms that $\sigma_{\mathrm{intra}}/\sigma_{\mathrm{cross}}$ provides fine-grained, interpretable control over intra- vs.\ cross-group information allocation while WR remains robust ($0.917$--$0.929$). To test whether the gain comes from prior \emph{asymmetry} rather than any block-diagonal structure, we reverse the prior ($\sigma_{\mathrm{intra}}{<}\sigma_{\mathrm{cross}}$) on role-rich MMM2 and homogeneous 25m. Remark~\ref{rem:no-regret} predicts \emph{reversed\,$<$\,flat\,$<$\,default} on MMM2 and \emph{no separation} on 25m; both are confirmed (Table~\ref{tab:sigma-ablation}). On SMACv2 the learned cross/intra ratio settles near $1\times$ (vs.\ $906\times$ on MAgent; Table~\ref{tab:bound-tightening}), showing that stochastic group assignments push the realised posterior allocation toward symmetry.

\begin{table}[t]
\centering\small
\caption{Two-sided reversed-prior falsifier (Remark~\ref{rem:no-regret}). On role-rich MMM2: reversed\,$<$\,flat\,$<$\,default; on homogeneous 25m: no WR separation. The cross/intra KL ratio on 25m ($652\times$) reflects the architecturally imposed prior asymmetry $\sigma_{\mathrm{intra}}^2/\sigma_{\mathrm{cross}}^2\!=\!100$; because the partition is effectively trivial, this asymmetry has no downstream impact on WR. $\sigma$-sweep on 3s5z is in the Appendix.}
\label{tab:sigma-ablation}
\renewcommand{\arraystretch}{1.1}
\setlength{\tabcolsep}{3pt}
\begin{tabular}{@{}lccc@{}}
\toprule
\textbf{Config} & $\sigma_{\mathrm{in}}/\sigma_{\mathrm{cr}}$
  & \textbf{WR} & \textbf{KL ratio} \\
\midrule
\multicolumn{4}{l}{\emph{MMM2 ($g{=}3$, role-rich)}} \\
HIB-flat             & 0.01\,/\,0.01 & $0.542\pm0.38$ & --- \\
HIBCG ($\star$)      & 0.10\,/\,0.01 & $\mathbf{0.814\pm0.06}$ & $607\times$ \\
HIBCG-reversed       & 0.01\,/\,0.10 & $0.511\pm0.27$ & $0.001\times$ \\
\midrule
\multicolumn{4}{l}{\emph{25m (homogeneous; partition trivial in effect)}} \\
HIB-flat             & 0.01\,/\,0.01 & $0.965\pm0.02$  & --- \\
HIBCG ($\star$)      & 0.10\,/\,0.01 & $\mathbf{0.972\pm0.008}$ & $652\times$ \\
HIBCG-reversed       & 0.01\,/\,0.10 & $0.945\pm0.014$ & $0.001\times$ \\
\bottomrule
\end{tabular}
\end{table}

\noindent\textbf{Proposition~\ref{prop:relevance} (TD relevance guarantee).}\;
Across all environments and configurations, HIBCG's TD loss converges
to values comparable with QMIX and GACG (loss\_td within $\pm$10\%),
confirming that the KL penalties do not displace task-relevant
information from the learned representations.
The practical implication is that HIBCG can be added to existing
value-decomposition pipelines with no performance penalty from the
regularizers alone.
The bound is $K$-sensitive: substituting $c=\Delta_{\min}^2/(4K)$
into Eq.~\eqref{eq:relevance-bound} shows that the gap to $H(Y)$ scales as
$\mathcal{L}_{\mathrm{TD}}\!\cdot\!K\log K$, so the per-unit cost of TD
error is amplified by the joint action-space size. The MAgent sweep
$n\!\in\!\{36,64,100\}$ is the cleanest controlled test of this
dependence (per-agent action set and all hyperparameters held fixed;
only~$n$ changes), and the HIBCG--QMIX final-return gap grows
monotonically across the sweep while QMIX's convergence rate degrades
and HIBCG's stays steady (Appendix
Table~\ref{tab:magent}). The full $K$-amplifier analysis with
per-map $\log_2 K$ is in the Appendix
Section~\ref{app:relevance-Kscaling}.

\noindent\textbf{On the other hyperparameters.}\;
HIBCG inherits from GACG and EPyMARL a handful of additional knobs
(adjacency threshold, Gumbel-softmax temperature, group-prior warmup
$T_{\mathrm{warm}}$, group count $g$, and $\lambda_A$).
To keep the main text focused on the IB mechanism, we collect the full
sensitivity study---including a group-count off-by-one test
($g\!=\!m_{\mathrm{true}}\!\pm\!1$ on 1c3s5z / MMM2), a warmup ablation
on 3s5z and 8m\_vs\_9m, and $\lambda_A$ scans---in
the Appendix.  The two take-aways relevant here
are: (1)~HIBCG remains no worse than HIB-flat in our off-by-one
\emph{group count} $g$ tests, consistent with the prior-family robustness
suggested by Remark~\ref{rem:no-regret};
and (2)~the non-IB knobs inherited from GACG have at most second-order
effects on final WR compared with the IB prior shape.


\begin{figure*}[t]
\centering
\includegraphics[width=0.95\textwidth]{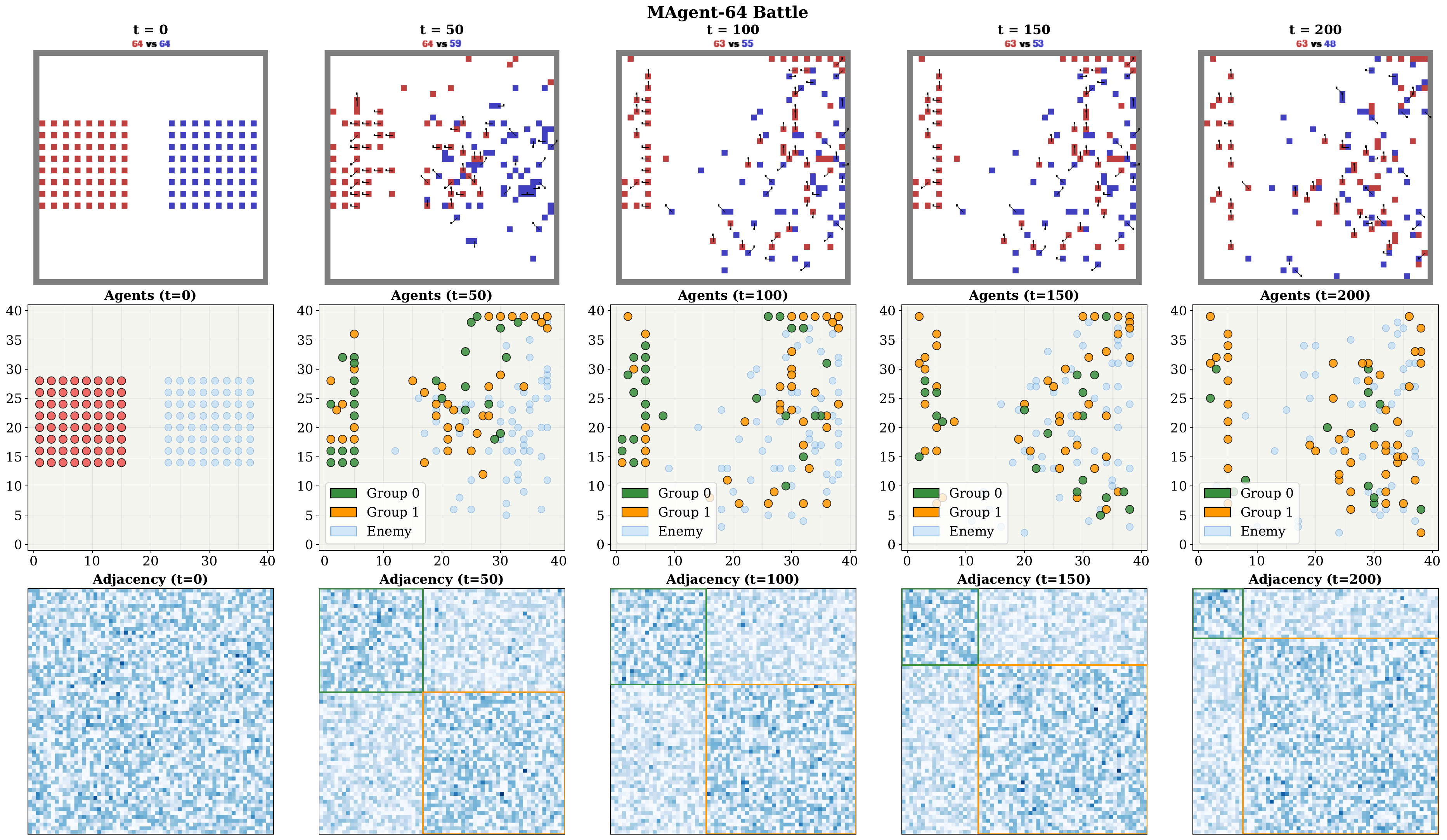}
\caption{%
  \textbf{HIBCG on MAgent-64 (Episode~3) across five timesteps.}
  \emph{Row~1}: battle render (red = our team, blue = enemy; 63 vs.\ 48 at $t{=}200$).
  \emph{Row~2}: agent positions colored by learned group (green = $G_0$, orange = $G_1$); groups track spatial flanks without supervision.
  \emph{Row~3}: $64{\times}64$ adjacency matrix; dense intra-group blocks and sparse inter-group blocks confirm the group-differentiated pruning driven by the group-aligned AIB prior.
}
\label{fig:temporal-snapshots}
\end{figure*}
\subsection{Scalability Analysis}
\label{sec:exp-scale}

The MAgent Battle platform provides a controlled scaling experiment:
the observation dimension is fixed ($d_{\mathrm{obs}}{=}338$) across
$n \in \{36, 64, 100\}$, so all HIBCG hyperparameters transfer
directly.  We report mean test return at the practical 2M-step
training budget together with convergence rate.

\noindent\textbf{Convergence rate.}\;
As $n$ increases, all methods degrade, but HIBCG degrades most
gracefully (Appendix Table~\ref{tab:magent}): $6/6$ seeds
converge at $n{=}36$, $6/6$ at $n{=}64$, and $5/6$ at $n{=}100$,
compared with $3/6$ for QMIX and $1/6$ for CommFormer at $n{=}100$.
GACG fails on the larger maps ($1/4$ seeds at $n{=}64$, $3/6$ at
$n{=}100$ but with negative mean return $-0.08\pm1.09$): the
$n^2{=}4096$--$10{,}000$ candidate edges make the fixed-threshold
graph too noisy for stable Q-learning.
HIBCG's AIB-learned sparsity retains only $13$--$14\%$ of edges at every scale, reducing per-agent reward noise and sustaining convergence.

\noindent\textbf{@3M advantage over QMIX.}\;
At the 3M-step comparison point, HIBCG leads QMIX by $+0.54$ at
$n{=}36$ (2.74 vs 2.19), $+0.18$ at $n{=}64$ (1.05 vs 0.87), and $+0.40$
at $n{=}100$ (0.45 vs 0.05).  HIBCG also dominates on peak return
($+$0.46, $+$0.48, $+$0.57 respectively) and reduces seed variance on
the converging maps ($\sigma{=}0.43$ vs QMIX $\sigma{=}0.37$ at $n{=}36$;
$\sigma{=}0.82$ vs QMIX $\sigma{=}0.56$ at $n{=}100$---comparable
dispersion despite converging $50\%$ more seeds).

\noindent\textbf{Architectural invariants across scale.}\;
The cross/intra AIB-loss ratio is $908{\times}$, $905{\times}$, and
$906{\times}$ at $n\in\{36,64,100\}$---the same value to within
$0.3\%$.
The order of magnitude is set by the fixed prior-scale ratio
$\sigma_{\mathrm{intra}}^2/\sigma_{\mathrm{cross}}^2\!=\!10^2$,
but the \emph{posterior} block scales (and hence the resulting
KL ratio) are themselves learned; varying the prior ratio in our
$\sigma$-sweep (Table~\ref{tab:sigma-ablation}) confirms that the
learned KL ratio tracks the prior, while WR is robust on
3s5z---so the cross/intra split is a controllable architectural knob,
not a hyperparameter artefact.
What is genuinely scale-invariant is that this same controllable knob
operates identically across $n\!\in\!\{36,64,100\}$ without
re-tuning---HIBCG's differential-compression pattern transfers
unchanged from $36$ to $100$ agents.
The AIB mean-KL scales as
$27376 \to 88696 \to 281326$ over $n\in\{36,64,100\}$,
or $3.24\times$ then $3.17\times$. The $36\!\to\!64$ ratio
closely matches the pairwise edge growth
$\tfrac{64\cdot 63}{36\cdot 35}\!\approx\!3.2$; the $64\!\to\!100$
ratio of $3.17\times$ exceeds the corresponding edge ratio
($\tfrac{100\cdot 99}{64\cdot 63}\!\approx\!2.45$),
indicating that per-edge AIB also grows modestly with $n$ as the
encoder allocates more capacity per surviving edge. The qualitative
trend is monotonic growth at roughly the order of the pairwise
edge count, consistent with AIB tracking the full $O(n^2)$
communication edge set.

\subsection{Diagnostics and Limitations}
\label{sec:exp-info}

\noindent\textbf{When does the group prior help?}\;
A practical diagnostic for whether HIBCG (group prior) or HIB-flat is the right variant is the \emph{Gdistance} metric inherited from GACG~\cite{GACG}: the average pairwise distance between the per-agent Q-value distributions of agents assigned to \emph{different} groups, normalised by within-group dispersion. Larger Gdistance therefore indicates that the learned partition separates agents whose contribution to $Q_{\mathrm{tot}}$ is qualitatively different. On maps where HIBCG delivers large gains (MMM2 Gdist $\gg 2$; MAgent-36 with spatial sub-teams), the block-diagonal prior aligns with genuine role structure. On the homogeneous \textbf{25m} (Gdist $\approx 1.6$--$2.1$), HIBCG ties HIB-flat---exactly as Remark~\ref{rem:no-regret} predicts when the partition is trivial. In our benchmarks, Gdistance $\lesssim\!2$ after $100$k training steps was a practical indicator that HIB-flat would perform comparably to HIBCG.

\noindent\textbf{XIB contribution depends on observation richness.}\;
Prop.~\ref{prop:dual-path} predicts that the AIB/XIB weighting is
benchmark-dependent. On SMACv1, XIB is the active path---XIB-only
lifts MMM2 WR by $+27$\,pp over GACG
(Table~\ref{tab:ablation-component}). On MAgent and SMACv2 the XIB
encoder collapses to its prior within $200$--$300$k steps (robust to
tighter priors and longer warmup), because the
$d_{\mathrm{obs}}\!=\!338$ spatial maps already carry enough local
information to make $I(\mathcal{D};Z_X\!\mid\!Z_A)$ in
Eq.~\eqref{eq:chain-rule-decomp} small; capacity is reallocated to
the AIB + group-prior path, consistent with the chain-rule split.
HIBCG degrades gracefully---the topology-side guarantees
(Remark~\ref{rem:no-regret}, Prop.~\ref{prop:group-decomp}) are
unaffected, and a high-$d_{\mathrm{obs}}$-resilient XIB prior is
natural future work (full diagnostics in the Appendix
Section~\ref{app:xib-collapse}).

\subsection{Case Study: Learned Graph on MAgent-64}
\label{sec:exp-visualization}

Figure~\ref{fig:temporal-snapshots} visualises three key properties on a representative MAgent-64 episode:
\textbf{(i)~Spatially meaningful groups:} the learned partition at $t{=}50$--$150$ clusters agents by battlefield flank without any spatial supervision, emerging purely from $\mathcal{L}_g$ and the block-diagonal prior.
\textbf{(ii)~Group-differentiated edge pruning:} from $t{\ge}50$ the adjacency shows clear block-diagonal structure---dense intra-group, sparse inter-group---consistent with the $906{\times}$ cross/intra KL ratio reported in \S\ref{sec:exp-scale}.
\textbf{(iii)~Dynamic adaptation:} group sizes re-balance as agents are eliminated, showing that HIBCG continuously adapts its coordination structure to the evolving tactical context.

\section{Discussion: Relation to MAGI, BVME, and GACG}
\label{sec:method-discussion}

In the taxonomy of Definitions~\ref{def:flat-ib}--\ref{def:ibcg}, the three closest prior works each occupy one corner of the design space (Table~\ref{tab:method-comparison}):
\emph{GACG}~\cite{GACG} answers \emph{who} should communicate via a group-aware sparse topology, but applies no explicit IB ($\beta\!=\!0$ edge-uniform regime);
\emph{BVME}~\cite{BVME}---our XIB-only ablation---adds final-layer message compression on GACG's graph, yet remains edge-uniform (no AIB, no group-aware blocking);
and \emph{MAGI}~\cite{MAGI} stacks a single categorical structural KL and a global Mixture-of-Gaussians message KL on a \emph{dense} attention graph, both with one global $\beta$ and a group-blind prior.

HIBCG occupies a distinct corner: a group-aware sparse graph with block-diagonal AIB, per-agent XIB, and closed-form per-block / per-agent control.
Each axis ties to a theoretical guarantee---no-regret group-aligned priors (Remark~\ref{rem:no-regret}), additive block decomposition (Prop.~\ref{prop:group-decomp}), and TD-driven relevance (Prop.~\ref{prop:relevance})---with empirical support in \S\ref{sec:exp-ablation} (dual-path), \S\ref{sec:exp-theory} (bound tightening), and the Appendix (hyperparameter sensitivity).

\noindent\textbf{Theory--experiment synthesis.}\;
Each guarantee leaves a distinct empirical signature.
Prop.~\ref{prop:dual-path} explains both the XIB-active SMACv1 regime and the XIB-dormant MAgent / SMACv2 regime (\S\ref{sec:exp-ablation},~\S\ref{sec:exp-info}).
Remark~\ref{rem:no-regret} matches the $2.2$--$6.5\times$ AIB-loss reduction on every heterogeneous map and the trivial-partition tie on 25m, with the reversed-prior ablation as falsifier (\S\ref{sec:exp-theory}).
Prop.~\ref{prop:relevance}'s $K$-amplifier is consistent with the monotonic HIBCG--QMIX gap on the MAgent $n$-sweep (\S\ref{sec:exp-scale}).
Group structure is the active ingredient: removing it (HIB-flat) or reversing it erodes the heterogeneous-map advantage.
A complete per-proposition synthesis is in Appendix~\ref{app:theory-exp-synthesis}.

\begin{table}[t]
\centering\small
\renewcommand{\arraystretch}{1.1}
\caption{Axis-by-axis comparison with the closest prior methods.}
\label{tab:method-comparison}
\setlength{\tabcolsep}{2pt}
\begin{tabular}{@{}lcccc@{}}
\toprule
& \textbf{GACG} & \textbf{BVME} & \textbf{MAGI} & \textbf{HIBCG} \\
\midrule
Graph type       & Sparse & Sparse & Dense attn. & Sparse \\
Structural IB    & ---    & ---    & Categ.\ KL  & Blk-diag.\ KL \\
Message IB       & ---    & Gauss. & MoG         & Gauss. \\
Group-aware      & \textbf{Yes} & No & No & \textbf{Yes} \\
Group-diff.\ pruning& ---    & No     & No          & \textbf{Yes} \\
Per-layer        & ---    & Last   & Selective   & \textbf{All} \\
Theory           & ---    & ---    & Var.\ bounds& \textbf{\shortstack{3 props\\+ 1 remark}} \\
\bottomrule
\end{tabular}
\end{table}

\section{Conclusion}
\label{sec:conclusion}

We proposed \textbf{HIBCG}, a heterogeneous coordination-graph learner that jointly learns heterogeneous connectivity and differentiated per-agent message capacity.
Using the graph information bottleneck, HIBCG decomposes graph learning into a topology path (AIB) and a message path (XIB), with a group-aligned block-diagonal prior that is never worse than a flat isotropic prior (Remark~\ref{rem:no-regret}), admits additive per-block penalties (Prop.~\ref{prop:group-decomp}), and relies on the standard TD loss as an IB relevance surrogate (Prop.~\ref{prop:relevance})---so no separate mutual-information estimator is required.
Across nine scenarios (SMACv1/v2, MAgent at $n\!\in\!\{36,64,100\}$), HIBCG wins on role-rich maps, scales to 100 agents where GACG and CommFormer collapse, and correctly yields no gain on the homogeneous 25m control, consistent with Remark~\ref{rem:no-regret}.
Natural follow-ups include high-$d_{\mathrm{obs}}$-resilient XIB priors (\S\ref{sec:exp-info}), data-driven group inference, the dual-ascent capacity variant (Appendix~\ref{app:dual-ascent}), and continuous-action benchmarks~\cite{DDFG}.

\clearpage
\appendix

\paragraph{Appendix guide.}
The appendix is organised into three groups: benchmarks
and full empirical results
(Sections~\ref{app:method-comparison}--\ref{app:info-allocation}),
proofs and derivations for the theoretical results plus a numerical worked example
(Sections~\ref{app:chain-rule-proof}--\ref{app:relevance-proof}),
and supporting material---theory--experiment synthesis, dual-ascent
variant, implementation details, and hyperparameter sensitivity
(Sections~\ref{app:theory-exp-synthesis}--\ref{app:hp-sensitivity}).
\begin{itemize}
\item \textbf{\S\ref{app:method-comparison}}---axis-by-axis
      method-comparison table referenced
      from~\S\ref{sec:method-discussion}.
\item \textbf{\S\ref{app:benchmarks}}---full list of benchmark maps
      and per-environment configurations used in our experiments.
\item \textbf{\S\ref{app:full-results}}---per-map SMACv1 and MAgent
      Battle result tables referenced
      from~\S\ref{sec:exp-overall}.
\item \textbf{\S\ref{app:info-allocation}}---mechanism /
      information-allocation figure referenced
      from~\S\S\ref{sec:exp-theory}--\ref{sec:exp-info},
      including XIB collapse diagnostics on MAgent and SMACv2
      (\S\ref{app:xib-collapse}).
\item \textbf{\S\ref{app:chain-rule-proof} (Proof)---dual-path
      decomposition} (Proposition~\ref{prop:dual-path}).
\item \textbf{\S\ref{app:group-prior-proof} (Proof)---group-prior
      no-regret property} (Remark~\ref{rem:no-regret}).
\item \textbf{\S\ref{app:block-decomp-proof} (Proof)---block
      decomposition} (Proposition~\ref{prop:group-decomp}).
\item \textbf{\S\ref{app:example-10agents}}---a complete numerical
      worked example with 10~agents illustrating the group-prior and
      block-decomposition mechanics.
\item \textbf{\S\ref{app:relevance-proof} (Proof)---relevance
      bound} (Proposition~\ref{prop:relevance}), including a
      $K$-scaling analysis across our benchmarks
      (\S\ref{app:relevance-Kscaling}).
\item \textbf{\S\ref{app:theory-exp-synthesis}}---per-proposition
      theory--experiment synthesis with prediction, falsifier, and
      headline number, referenced
      from~\S\ref{sec:method-discussion}.
\item \textbf{\S\ref{app:dual-ascent}}---dual-ascent variant for
      adaptive capacity control, an automatic alternative to fixing
      $\lambda_A^{(l,g)}$ by hand.
\item \textbf{\S\ref{app:impl-details}}---implementation details and
      training algorithm.
\item \textbf{\S\ref{app:execution-model}}---execution model: default
      regime (centralised graph construction, decentralised action
      selection) and the optional frozen-graph regime for fully
      decentralised deployment.
\item \textbf{\S\ref{app:hp-sensitivity}}---full hyperparameter
      sensitivity study (group count off-by-one, warmup, and
      $\lambda_A$ scans) referenced
      from~\S\ref{sec:exp-theory}.
\end{itemize}

\section{Axis-by-Axis Method Comparison}
\label{app:method-comparison}

For convenience, Table~\ref{tab:app-method-comparison} reproduces the axis-by-axis comparison from~\S\ref{sec:method-discussion} of the main body with additional detail. GACG and BVME are our earlier works; HIBCG is the only method that combines a sparse, group-aware coordination graph, a layer-wise block-diagonal structural IB, and a per-agent message IB with closed-form, per-block, per-agent capacity weights.

\begin{table}[h]
\centering\small
\renewcommand{\arraystretch}{1.15}
\caption{Axis-by-axis comparison of HIBCG with the closest existing coordination-graph and information-bottleneck methods for MARL communication. GACG and BVME are our earlier works.}
\label{tab:app-method-comparison}
\begin{tabular}{@{}lcccc@{}}
\toprule
& \textbf{GACG}~\cite{GACG}
& \textbf{BVME}~\cite{BVME}
& \textbf{MAGI}~\cite{MAGI}
& \textbf{HIBCG} (ours) \\
\midrule
Graph type
  & Sparse CG & Sparse CG & Dense attention & Sparse CG \\
Structural IB
  & --- & --- & Categorical KL & Block-diag.\ Gauss.\ KL \\
Message IB
  & --- & Diag.\ Gauss.\ KL & MoG KL & Diag.\ Gauss.\ KL \\
Group-aware
  & \textbf{Yes} & No & No & \textbf{Yes} \\
Group-diff.\ pruning
  & --- & No & No & \textbf{Yes} \\
Per-layer
  & --- & Last layer & Selective & \textbf{All layers} \\
Theory
  & --- & --- & Variational bounds & \textbf{3 props + 1 remark} \\
\bottomrule
\end{tabular}
\end{table}

\section{Benchmark Maps and Configurations}
\label{app:benchmarks}

Table~\ref{tab:envs} lists the full set of maps used in the experiments of~\S\ref{sec:experiments}. SMACv1 maps test fixed team compositions with one homogeneous map (25m) as a negative control; SMACv2 maps add stochastic team composition; MAgent Battle stresses scale up to 100 agents.

\begin{table}[h]
\centering\small
\caption{Environments and map configurations.
$|\mathcal{U}_i|$ is the per-agent action-space size
(SMAC: $6\!+\!n_{\mathrm{enemy}}$; MAgent Battle: $21$);
$\log_2 K\!=\!n\log_2|\mathcal{U}_i|$ is the log joint action-space size
referenced by Proposition~\ref{prop:relevance}.}
\label{tab:envs}
\renewcommand{\arraystretch}{1.10}
\setlength{\tabcolsep}{4pt}
\begin{tabular}{@{}llccccc@{}}
\toprule
\textbf{Platform} & \textbf{Map} & \textbf{$n$} & \textbf{Types}
  & $|\mathcal{U}_i|$ & $\log_2 K$ & \textbf{Role} \\
\midrule
\multirow{4}{*}{SMACv1~\cite{SMAC}}
  & 3s5z       & 8  & 2 (het) & 14 & $\approx 30.5$  & Primary ablation \\
  & 1c3s5z     & 9  & 3 (het) & 15 & $\approx 35.2$  & Het.\ validation \\
  & MMM2       & 10 & 3 (het) & 18 & $\approx 41.7$  & Role-matched GP \\
  & 25m        & 25 & 1 (homo)& 31 & $\approx 123.8$ & GP negative control \\
\midrule
\multirow{2}{*}{SMACv2~\cite{DBLP:conf/nips/EllisCMSSMFW23}}
  & protoss\_8v8    & 8  & 2 (var.) & 14 & $\approx 30.5$ & Variable comp.\ \\
  & terran\_10v10   & 10 & 3 (var.) & 16 & $\approx 40.0$ & 3-type var.\ comp.\ \\
\midrule
\multirow{3}{*}{MAgent~\cite{MAgent}}
  & Battle\_36  & 36  & spatial & 21 & $\approx 158.2$ & Scale benchmark \\
  & Battle\_64  & 64  & spatial & 21 & $\approx 281.2$ & Mid-scale \\
  & Battle\_100 & 100 & spatial & 21 & $\approx 439.3$ & Stress test \\
\bottomrule
\end{tabular}
\end{table}

\section{Full Per-Map Results and Component Ablation}
\label{app:full-results}

This section provides the full per-map result tables (SMACv1, SMACv2, and MAgent Battle) and the component ablation table referenced from the main body. The corresponding learning curves are shown in Figure~\ref{fig:learning-curves}.

\begin{table}[h]
\centering\small
\caption{Component ablation on three representative maps.
Values are final test WR on SMAC (3s5z at $1.2\mathrm{M}$ steps, MMM2 at $2\mathrm{M}$) and final test return on MAgent-36 ($3\mathrm{M}$). Mean $\pm$ std over 5 seeds. Best in \textbf{bold}.}
\label{tab:app-ablation-component}
\renewcommand{\arraystretch}{1.15}
\setlength{\tabcolsep}{4pt}
\begin{tabular}{@{}lccc@{}}
\toprule
\textbf{Component config}
  & \textbf{3s5z} (WR)  & \textbf{MMM2} (WR)  & \textbf{MAgent-36} (ret.) \\
\midrule
QMIX (no graph)            & $0.897\pm0.062$ & $0.386\pm 0.25$ & $2.19\pm 0.37$ \\
GACG (hard threshold)      & $0.897\pm0.020$ & $0.488\pm 0.33$ & $1.63\pm 1.79$ \\
BVME / XIB-only (msg.\ IB) & $0.896\pm0.033$ & $0.755\pm 0.08$ & $1.43\pm 1.85$ \\
AIB-only (structural IB)   & $0.926\pm0.021$ & $0.502\pm 0.18$ & $1.87\pm 0.84$ \\
HIB-flat (AIB+XIB, flat)   & $0.908\pm0.034$ & $0.542\pm 0.38$ & $1.72\pm 1.36$ \\
\textbf{HIBCG (AIB+XIB+GP)}
  & $\mathbf{0.926\pm0.018}$
  & $\mathbf{0.814\pm 0.06}$
  & $\mathbf{2.74\pm 0.43}$ \\
\bottomrule
\end{tabular}
\end{table}

\begin{table}[h]
\centering\small
\caption{SMACv2 test win rate and integrated WR-AUC (0--2M steps; tail-10\% mean $\pm$ std over 5 seeds). Best in \textbf{bold}, second \underline{underlined}.}
\label{tab:smacv2}
\renewcommand{\arraystretch}{1.12}
\setlength{\tabcolsep}{4pt}
\begin{tabular}{@{}lcccc@{}}
\toprule
& \multicolumn{2}{c}{\textbf{protoss\_8v8}}
& \multicolumn{2}{c}{\textbf{terran\_10v10}} \\
\cmidrule(lr){2-3}\cmidrule(lr){4-5}
\textbf{Method} & WR & AUC & WR & AUC \\
\midrule
QMIX          & $0.328\pm0.068$ & $0.174\pm0.020$ & $0.420\pm0.068$ & $0.231\pm0.034$ \\
MAGI          & $0.393\pm0.055$ & $0.214\pm0.036$ & $0.378\pm0.083$ & $0.208\pm0.033$ \\
CommFormer    & $0.363\pm0.078$ & $0.186\pm0.028$ & $0.459\pm0.083$ & $0.216\pm0.028$ \\
ExpoComm      & $\underline{0.433\pm0.020}$ & $\mathbf{0.238\pm0.013}$ & $\mathbf{0.511\pm0.041}$ & $\underline{0.237\pm0.029}$ \\
GACG          & $0.358\pm0.039$ & $0.188\pm0.019$ & $0.375\pm0.088$ & $0.196\pm0.027$ \\
\midrule
BVME          & $0.333\pm0.074$ & $0.186\pm0.029$ & $0.442\pm0.037$ & $0.217\pm0.026$ \\
AIB-only      & $0.374\pm0.062$ & $0.187\pm0.016$ & $0.380\pm0.074$ & $0.193\pm0.025$ \\
HIB-flat      & $0.355\pm0.063$ & $0.190\pm0.017$ & $0.394\pm0.063$ & $0.208\pm0.025$ \\
\textbf{HIBCG}
              & $\mathbf{0.448\pm0.022}$ & $\underline{0.237\pm0.014}$
              & $\underline{0.509\pm0.034}$ & $\mathbf{0.246\pm0.017}$ \\
\bottomrule
\end{tabular}
\end{table}

\begin{table}[h]
\centering\small
\caption{SMACv1 test win rate (tail-10\% mean $\pm$ std over 5 seeds).
Training budget is $1.2\mathrm{M}$ steps on 3s5z and $2\mathrm{M}$ steps on
1c3s5z, MMM2, and 25m.  Best in \textbf{bold}, second \underline{underlined}.
25m is a homogeneous negative control for the group prior.
\textsuperscript{$\ddagger$}\,ExpoComm on MMM2 exhibits bimodal
non-convergence ($1/5$ seeds converge; $3/5$ fully collapse to zero
reward); see discussion in the main body.
}
\label{tab:smacv1}
\renewcommand{\arraystretch}{1.12}
\begin{tabular}{@{}lcccc@{}}
\toprule
\textbf{Method}
  & \textbf{3s5z} (2-type, $1.2\mathrm{M}$) & \textbf{1c3s5z} (3-type) & \textbf{MMM2} (3-type)
  & \textbf{25m} (homo.\ control) \\
\midrule
QMIX        & $0.897\pm0.062$ & $\underline{0.914\pm0.034}$     & $0.386\pm 0.25$ & $\underline{0.981\pm 0.01}$ \\
MAGI        & $\underline{0.926\pm0.012}$  & $0.905\pm0.024$   & $0.528\pm 0.35$   & $0.975\pm0.012$ \\
CommFormer  & $0.912\pm0.026$   & $0.878\pm0.034$   & $\underline{0.774\pm 0.08}$ & $\mathbf{0.983\pm0.009}$ \\
ExpoComm    & $0.855\pm0.042$   & $0.895\pm0.015$   & $0.204\pm 0.31$\textsuperscript{$\ddagger$} & $0.915\pm0.092$ \\
GACG        & $0.897\pm0.020$   & $0.897\pm0.051$ & $0.488\pm 0.33$ & $0.960\pm 0.03$ \\
\midrule
BVME        & $0.896\pm0.033$     & $0.913\pm0.023$   & $0.755\pm 0.08$ & $0.959\pm 0.02$ \\
AIB-only    & $0.926\pm0.021$ & $\mathbf{0.914\pm0.022}$  & $0.502\pm 0.18$ & $0.943\pm 0.036$ \\
HIB-flat    & $0.908\pm0.034$   & $0.906\pm0.044$  & $0.542\pm 0.38$ & $0.965\pm 0.02$ \\
\textbf{HIBCG}
            & $\mathbf{0.927\pm0.018}$ & $0.912\pm0.010$   & $\mathbf{0.814\pm 0.06}$  & $0.972\pm 0.008$ \\
\bottomrule
\end{tabular}
\end{table}

\begin{table}[h]
\centering\small
\caption{MAgent Battle: tail-10\% \emph{mean} test return (final), best
\emph{per-seed} peak test return during training, and convergence rate
(fraction of seeds whose tail mean is positive).  Training budgets are
$3\mathrm{M}$ environment steps at $n{=}36$ and $2\mathrm{M}$ steps at
$n{\in}\{64,100\}$ (same budgets for final-return and peak columns).
Best in \textbf{bold}, second \underline{underlined}.
All entries are over 5 seeds unless noted otherwise.}
\label{tab:magent}
\renewcommand{\arraystretch}{1.12}
\setlength{\tabcolsep}{4pt}
\begin{tabular}{@{}lcccccccccc@{}}
\toprule
& \multicolumn{3}{c}{\textbf{Final return}}
& \multicolumn{3}{c}{\textbf{Peak return}}
& \multicolumn{3}{c}{\textbf{Conv.\ rate}} \\
\cmidrule(lr){2-4}\cmidrule(lr){5-7}\cmidrule(lr){8-10}
\textbf{Method} & $n{=}36$ & $n{=}64$ & $n{=}100$
                & $n{=}36$ & $n{=}64$ & $n{=}100$
                & $n{=}36$ & $n{=}64$ & $n{=}100$ \\
 & {\footnotesize$3\mathrm{M}$} & {\footnotesize$2\mathrm{M}$} & {\footnotesize$2\mathrm{M}$}
 & {\footnotesize$3\mathrm{M}$} & {\footnotesize$2\mathrm{M}$} & {\footnotesize$2\mathrm{M}$}
 & & & \\
\midrule
QMIX       & $2.19\pm0.37$  & $0.43\pm1.22$ & $-0.22\pm0.95$
           & $2.61$ & $1.27$ & $1.05$ & $5/5$ & $3/7$ & $3/6$ \\
MAGI       & $1.11\pm0.96$  & $0.15\pm0.77$ & $-0.40\pm0.36$
           & $1.67$ & $1.48$ & $0.59$ & $4/5$ & $2/5$ & $0/5$ \\
CommFormer & $\underline{2.50\pm0.13}$ & $0.11\pm0.63$  & $-0.52\pm1.04$
           & $2.79$ & $1.43$ & $0.32$ & $5/5$ & $3/7$ & $1/6$ \\
GACG       & $1.63\pm1.79$  & $-0.46\pm0.82$
           & $-0.08\pm1.09$
           & $2.16$ & $0.74$ & $1.29$
           & $4/5$ & $4/5$ & $3/6$ \\
ExpoComm   & $1.66\pm0.53$ & $0.22\pm0.46$ & $-0.98\pm0.33$
           & $2.35$ & $0.71$ & $0.66$
           & $7/7$ & $4/5$ & $0/6$ \\
\midrule
AIB-only   & $1.87\pm0.84$  & $\underline{0.62\pm0.57}$ & $-0.63\pm0.84$
           & $2.41$ & $1.32$ & $1.38$ & $5/5$ & $5/5$ & $2/6$ \\
BVME       & $1.43\pm1.85$  & $0.08\pm1.08$ & $0.59\pm0.76$
           & $2.19$ & $1.21$ & $1.57$ & $4/5$ & $4/5$ & $5/6$ \\
HIB-flat   & $1.72\pm1.36$  & $\underline{0.62\pm1.27}$  & $0.42\pm0.23$  & $2.21$ & $1.65$& $1.19$ & $4/5$ & $4/5$ & $5/6$ \\
\textbf{HIBCG}
           & $\mathbf{2.74\pm0.43}$
           & $\mathbf{1.38\pm0.84}$
           & $\mathbf{0.83\pm0.54}$
           & $\mathbf{3.07}$ & $\mathbf{2.13}$ & $\mathbf{1.72}$
           & $6/6$  & $6/6$  & $5/6$ \\
\bottomrule
\end{tabular}
\end{table}

\section{Mechanism and Information-Allocation Analysis}
\label{app:info-allocation}

Figure~\ref{fig:mechanism} provides three diagnostic panels referenced from \S\S\ref{sec:exp-theory}--\ref{sec:exp-info} of the main body: (a) graph density evolution on 3s5z, validating the dual-path coupling between AIB and XIB; (b) intra-/cross-group AIB loss on MAgent-36, showing the architectural cross/intra ratio of $\sim\!906\times$; and (c) AIB mean-KL across MAgent scales, confirming that the group-conditional prior tracks the $O(n^2)$ edge growth while the flat prior saturates.

\begin{figure}[h]
\centering
\includegraphics[width=\textwidth]{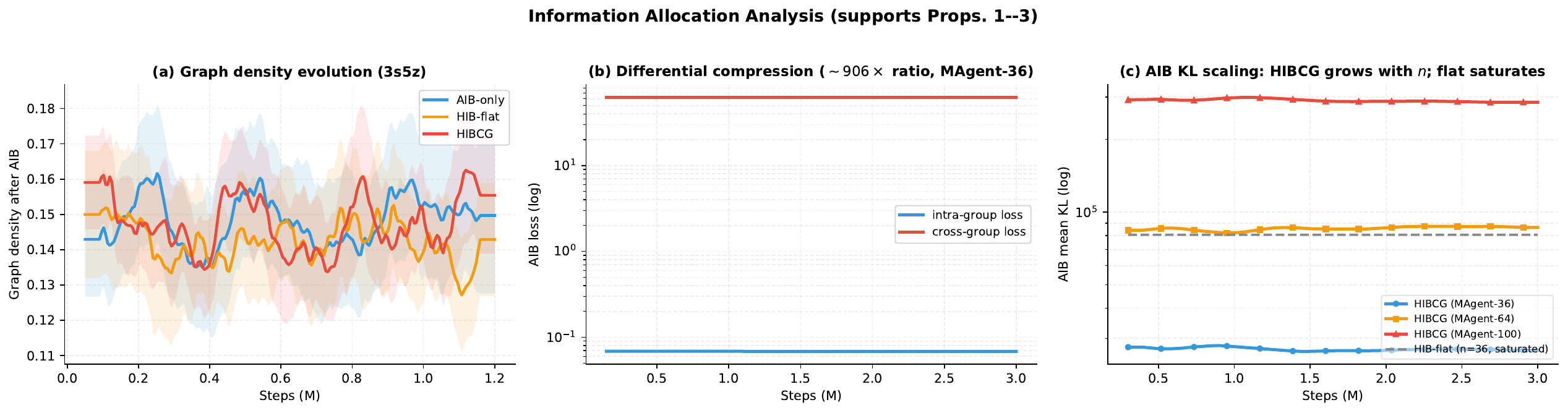}
\caption{\textbf{Mechanism and information-allocation analysis.}
(a) Graph density evolution on 3s5z: AIB-only follows a
``warm-then-prune'' trajectory; HIBCG reaches a comparable asymptote
with lower variance, confirming dual-path coupling.
(b) Differential compression inside HIBCG on MAgent-36: intra-group
AIB loss stays at $\sim\!0.07$ while cross-group loss stabilises at
$\sim\!62$---a $\sim\!906\times$ ratio that holds at $n\in\{36,64,100\}$
(structural architectural invariant).
(c) AIB mean-KL across scales: HIBCG's KL scales with agent count
(tracking the $O(n^2)$ edge count), whereas the flat prior saturates
at its structural ceiling---direct evidence that the group-conditional
prior exploits the block-diagonal edge structure.
}
\label{fig:mechanism}
\end{figure}

\subsection{XIB Collapse Diagnostics on MAgent and SMACv2}
\label{app:xib-collapse}

The dual-path decomposition (Prop.~\ref{prop:dual-path}) predicts
that AIB and XIB carry complementary information; whether either path
dominates is therefore benchmark-dependent. We summarise here the
quantitative diagnostics behind the short discussion in
\S\ref{sec:exp-info} of the main body.

\noindent\textbf{XIB-active regime (SMACv1).}\;
On role-rich, low-dimensional SMACv1 observations, the XIB path is the
primary lift: BVME (XIB-only on GACG's graph) reaches
$0.755$ WR on MMM2 versus GACG's $0.488$---a $+27$\,pp gain from
message compression alone (Table~\ref{tab:ablation-component}).

\noindent\textbf{XIB-dormant regime (MAgent, SMACv2).}\;
On MAgent and SMACv2, the XIB encoder collapses to its prior within
$200$--$300$k steps; concretely, $\widehat{\mathrm{XIB}}\!\to\!0$ and
$\mathrm{eval\_without\_messages}\!=\!\mathrm{eval\_original}$ to two
decimals. The collapse is robust to standard remediation knobs:
tighter priors ($\sigma_0\!\in\!\{0.005,0.01,0.02\}$), extended warmup
($500$k steps), and per-channel rate floors all fail to reactivate it.
The cause is structural: the $d_{\mathrm{obs}}\!=\!338$ MAgent spatial
maps already carry enough local information for decentralised action
selection, so $I(\mathcal{D};Z_X\!\mid\!Z_A)$ in
Eq.~\eqref{eq:chain-rule-decomp} is small. Capacity is then
reallocated to the AIB + group-prior path, consistent with the chain
rule.

\noindent\textbf{Implication.}\;
HIBCG degrades gracefully when XIB collapses: the topology-side
guarantees (Remark~\ref{rem:no-regret} and Prop.~\ref{prop:group-decomp})
are unaffected and HIBCG continues to deliver state-of-the-art
performance on both benchmark families.
Designing an XIB encoder that remains active under high-$d_{\mathrm{obs}}$
inputs---for instance, by expressing the prior in a coordinate system
that strips out already-observable redundancy---is a natural direction
for future work.

\clearpage

\section{Proof of Dual-Path Decomposition (Proposition~\ref{prop:dual-path})}
\label{app:chain-rule-proof}

For the reader's convenience, we restate the proposition before proving it.

\begin{proposition*}[\ref{prop:dual-path}, Dual-Path Decomposition via Chain Rule]
The total information flow decomposes into a structural component and a
conditional message component:
\[
I(\mathcal{D};\;Z_A,\,Z_X)
\;=\;
\underbrace{I(\mathcal{D};\;Z_A)}_{\text{structural bits (AIB)}}
\;+\;
\underbrace{I(\mathcal{D};\;Z_X\mid Z_A)}_{\text{message bits given topology (XIB)}}.
\]
Under the conditional independencies
\begin{enumerate}
\item[(CI-A)] $Z_A^{(l)}\!\perp\!(\mathcal{D},Z_A^{(1:l-1)},Z_X^{(1:l-2)})
\;\mid\;(A^{(0)},Z_X^{(l-1)})$,
\item[(CI-X)] $Z_X^{(l)}\!\perp\!(\mathcal{D},Z_X^{(1:l-2)},Z_A^{(1:l-1)})
\;\mid\;(Z_X^{(l-1)},Z_A^{(l)})$,
\end{enumerate}
each term admits a variational upper bound:
\begin{align*}
I(\mathcal{D};\;Z_A)
&\;\le\;
\sum_{l=1}^{L}
\mathbb{E}\!\Big[\mathrm{KL}\!\big(
\mathbb{P}(Z_A^{(l)}\!\mid\!A^{(0)},Z_X^{(l-1)})\,\Vert\,
\mathbb{Q}(Z_A^{(l)})\big)\Big]
\;\triangleq\;\sum_{l=1}^{L}\mathrm{AIB}^{(l)},\\[4pt]
I(\mathcal{D};\;Z_X\mid Z_A)
&\;\le\;
\sum_{l=1}^{L}
\mathbb{E}\!\Big[\mathrm{KL}\!\big(
\mathbb{P}(Z_X^{(l)}\!\mid\!Z_X^{(l-1)},Z_A^{(l)})\,\Vert\,
\mathbb{Q}(Z_X^{(l)})\big)\Big]
\;\triangleq\;\sum_{l=1}^{L}\mathrm{XIB}^{(l)}.
\end{align*}
\end{proposition*}

We now provide a detailed, step-by-step derivation.  All
information-theoretic identities used are standard
(see~\cite{cover_thomas}, Chapters~2 and~8).

\begin{proof}
We proceed in seven steps.

\paragraph{Step~1: The generative structure of message-passing GNNs.}
Before proving the bound, we make the generative structure of the GNN
explicit.  The base structure $A^{(0)}$ is a \emph{stochastic} function of
$\mathcal{D}$: it is sampled from the GACG Gaussian
$\mathcal{N}(\boldsymbol{\mu}^t,\Sigma_A^t)$
whose parameters are computed from $\mathcal{D}$ (Eq.~\eqref{eq:base-struct}).
The sampling noise is independent of all downstream layer-wise latents
$\{Z_A^{(l)},Z_X^{(l)}\}_{l\ge1}$, and the proof below always conditions on the
realized value of $A^{(0)}$, so this stochasticity does not affect any
conditional independence used in the argument.
At each layer~$l$, the computation proceeds as:
\begin{equation}
\label{eq:app-markov}
\mathcal{D}\;\to\;A^{(0)},\qquad
\text{and for each }l:\quad
(A^{(0)},Z_X^{(l-1)}) \;\to\; Z_A^{(l)} \;\to\; Z_X^{(l)},
\end{equation}
where each arrow denotes a stochastic dependence.
\emph{Note:} this is \textbf{not} a simple chain---$A^{(0)}$ (and hence
$\mathcal{D}$) feeds into every layer's structural encoder via a skip
connection.  The precise conditional dependencies are:
\begin{itemize}
\item $Z_A^{(l)}$ is drawn from
$\mathbb{P}(Z_A^{(l)}\!\mid\! A^{(0)},Z_X^{(l-1)})$, the structural
encoder (Eq.~\eqref{eq:a-encoder}).  It depends on $\mathcal{D}$ through
$A^{(0)}$ (derived from $\mathcal{D}$) and $Z_X^{(l-1)}$.
\item $Z_X^{(l)}$ is computed by deterministic graph message passing
$Z_X^{(l)}=\phi_l(\widetilde{A}^{(l)}\,Z_X^{(l-1)}\,W_l)$ given
$(Z_X^{(l-1)},Z_A^{(l)})$ (equivalently $\widetilde{A}^{(l)}=g_l(Z_A^{(l)})$).
It depends on $\mathcal{D}$ only through $Z_X^{(l-1)}$ and $Z_A^{(l)}$.
The stochastic message encoder of Eq.~\eqref{eq:x-encoder} is applied
\emph{after} the $L$-layer stack to produce $Z_X^{\mathrm{out}}$, not at
each intermediate layer.
\end{itemize}
The key conditional independence properties are:
\begin{enumerate}
\item[(CI-A)] $Z_A^{(l)}\!\perp\!(\mathcal{D},\,Z_A^{(1:l-1)},\,Z_X^{(1:l-2)})
\;\mid\;(A^{(0)},\,Z_X^{(l-1)})$.
\item[(CI-X)] $Z_X^{(l)}\!\perp\!(\mathcal{D},\,Z_A^{(1:l-1)},\,Z_X^{(1:l-2)})
\;\mid\;(Z_X^{(l-1)},\,Z_A^{(l)})$.
\end{enumerate}
That is, $(A^{(0)},Z_X^{(l-1)})$ is a sufficient statistic of
$(\mathcal{D},Z_A^{(1:l-1)})$ for $Z_A^{(l)}$, and
$(Z_X^{(l-1)},Z_A^{(l)})$ is a sufficient statistic for $Z_X^{(l)}$.

\paragraph{Step~2: Chain rule of mutual information (the key identity).}
For any random variables $X$, $Y$, $Z$, the chain rule states:
\begin{equation}
\label{eq:app-chain-rule}
I(X;\,Y,Z) = I(X;\,Y) + I(X;\,Z\mid Y).
\end{equation}
\emph{Proof:} From the definition of MI and conditional MI,
\begin{align}
I(X;Y,Z)
&= H(X)-H(X\mid Y,Z) \nonumber\\
&= \big[H(X)-H(X\mid Y)\big]+\big[H(X\mid Y)-H(X\mid Y,Z)\big] \nonumber\\
&= I(X;Y)+I(X;Z\mid Y). \tag*{$\diamond$}
\end{align}
Apply this with $X=\mathcal{D}$, $Y=Z_A=\{Z_A^{(l)}\}_{l=1}^L$,
$Z=Z_X=\{Z_X^{(l)}\}_{l=1}^L$:
\begin{equation}
\label{eq:app-chain-applied}
I(\mathcal{D};\;Z_A,\,Z_X)
\;=\;
\underbrace{I(\mathcal{D};\;Z_A)}_{\text{structural bits (AIB)}}
\;+\;
\underbrace{I(\mathcal{D};\;Z_X\mid Z_A)}_{\text{message bits given topology (XIB)}}.
\end{equation}
This is Eq.~\eqref{eq:chain-rule-decomp} in the main body.

\paragraph{Step~3: Layer-wise decomposition and upper bound of the structural term.}
We further decompose $I(\mathcal{D};Z_A)=I(\mathcal{D};Z_A^{(1:L)})$ using
the chain rule applied to layers sequentially:
\begin{align}
I(\mathcal{D};\,Z_A^{(1:L)})
&= I(\mathcal{D};\,Z_A^{(1)})
 + I(\mathcal{D};\,Z_A^{(2)}\mid Z_A^{(1)})
 + \cdots
 + I(\mathcal{D};\,Z_A^{(L)}\mid Z_A^{(1:L-1)}) \nonumber\\
&= \sum_{l=1}^{L} I(\mathcal{D};\,Z_A^{(l)}\mid Z_A^{(1:l-1)}).
\label{eq:app-layer-chain}
\end{align}
We now bound each conditional term.
Since adding variables to the ``source'' of a mutual information can only
increase it\footnote{This follows directly from the chain rule:
$I(X,W;\,Y)=I(W;\,Y)+I(X;\,Y\mid W)\ge I(X;\,Y\mid W)$ for any~$W$.},
we have
\begin{equation}
\label{eq:app-source-bound}
I(\mathcal{D};\,Z_A^{(l)}\mid Z_A^{(1:l-1)})
\;\le\;
I(\mathcal{D},\,Z_A^{(1:l-1)};\;Z_A^{(l)}).
\end{equation}
\emph{Remark.} The inequality
$I(\mathcal{D};Z_A^{(l)}\mid Z_A^{(1:l-1)})\le I(\mathcal{D};Z_A^{(l)})$
(dropping the conditioning) does \emph{not} hold in general, because the
interaction information can have either sign.
The correct route is~\eqref{eq:app-source-bound}, which is always valid.

\medskip\noindent
By the conditional independence (CI-A) from Step~1, the following
Markov chain holds---given the ``direct inputs''
$(A^{(0)},Z_X^{(l-1)})$, the current output $Z_A^{(l)}$ is independent
of all ``distant past'' variables $(\mathcal{D},Z_A^{(1:l-1)})$:
\[
\underbrace{(\mathcal{D},\, Z_A^{(1:l-1)})}_{\text{distant past}\;(X)}
\;\longrightarrow\;
\underbrace{(A^{(0)},\, Z_X^{(l-1)})}_{\text{direct inputs}\;(Y)}
\;\longrightarrow\;
\underbrace{Z_A^{(l)}}_{\text{current output}\;(Z)}.
\]
The data-processing inequality (DPI) states that for any Markov chain
$X\!\to\!Y\!\to\!Z$, the mutual information satisfies
$I(X;Z)\le I(Y;Z)$---information can only be lost, not created,
through intermediate processing.
Applying DPI to the chain above:
\begin{equation}
\label{eq:app-markov-ineq}
\underbrace{I(\mathcal{D},\,Z_A^{(1:l-1)};\;Z_A^{(l)})}_{\text{info shared with distant past}}
\;\le\;
\underbrace{I(A^{(0)},Z_X^{(l-1)};\;Z_A^{(l)})}_{\text{info shared with direct inputs}}.
\end{equation}
Note that strict equality does \emph{not} hold in general, because
$(A^{(0)},Z_X^{(l-1)})$ carries stochastic noise from the
structural reparameterization variables $\varepsilon_A^{(1:l-1)}$
injected by earlier AIB encoders (Eq.~\eqref{eq:a-encoder}),
which affect $Z_X^{(l-1)}$ through gated message passing.
This noise is independent of the distant past
$(\mathcal{D},Z_A^{(1:l-1)})$ yet informative about~$Z_A^{(l)}$.
The inequality suffices for the bound; any gap here only makes the overall
bound looser.
Combining \eqref{eq:app-layer-chain}, \eqref{eq:app-source-bound},
and~\eqref{eq:app-markov-ineq}, we chain the inequalities for each layer
and then sum:
\begin{align}
I(\mathcal{D};\,Z_A^{(1:L)})
&\;=\;\sum_{l=1}^{L}
  I\!\big(\mathcal{D};\,Z_A^{(l)}\,\big|\,Z_A^{(1:l-1)}\big)
  &&\text{(chain rule, Eq.~\eqref{eq:app-layer-chain})}
  \nonumber\\[4pt]
&\;\le\;\sum_{l=1}^{L}
  I\!\big(\mathcal{D},\,Z_A^{(1:l-1)};\;Z_A^{(l)}\big)
  &&\text{(source augmentation, Eq.~\eqref{eq:app-source-bound})}
  \nonumber\\[4pt]
&\;\le\;\sum_{l=1}^{L}
  I\!\big(A^{(0)},Z_X^{(l-1)};\;Z_A^{(l)}\big).
  &&\text{(data-processing inequality (DPI), Eq.~\eqref{eq:app-markov-ineq})}
  \label{eq:app-structural-sum}
\end{align}
The first line is an \emph{equality}; both subsequent lines are
\emph{inequalities} that hold for every~$l$ individually, so they are
preserved under summation.

\paragraph{Step~4: Variational upper bound on each structural layer.}
For any distribution $\mathbb{Q}(Z)$, the mutual information satisfies:
\begin{equation}
\label{eq:app-variational-identity}
I(X;Z)
= \mathbb{E}_{X}\!\big[\mathrm{KL}(\mathbb{P}(Z\mid X)\Vert\bar{\mathbb{P}}(Z))\big],
\end{equation}
where $\bar{\mathbb{P}}(Z)=\mathbb{E}_X[\mathbb{P}(Z\mid X)]$ is the
marginal (or ``aggregate posterior''), which is intractable because it requires integrating the posterior over the data distribution.
To bypass this, we introduce a tractable variational approximation (or ``prior'') $\mathbb{Q}(Z)$:
\begin{align}
I(X;Z)
&= \mathbb{E}_{X}\!\big[\mathrm{KL}(\mathbb{P}(Z\mid X)\Vert\mathbb{Q}(Z))\big]
  -\mathrm{KL}(\bar{\mathbb{P}}(Z)\Vert\mathbb{Q}(Z)) \nonumber\\
&\le \mathbb{E}_{X}\!\big[\mathrm{KL}(\mathbb{P}(Z\mid X)\Vert\mathbb{Q}(Z))\big],
\label{eq:app-vub-derivation}
\end{align}
since $\mathrm{KL}(\bar{\mathbb{P}}\Vert\mathbb{Q})\ge 0$.  The gap
$\mathrm{KL}(\bar{\mathbb{P}}\Vert\mathbb{Q})$ measures how well the
prior $\mathbb{Q}$ approximates the true marginal---this is precisely the
quantity that group-conditional priors reduce
(Remark~\ref{rem:no-regret}).

We now instantiate~\eqref{eq:app-vub-derivation} for each structural
layer.  The generic bound states $I(X;Z)\le
\mathbb{E}_{X}[\mathrm{KL}(\mathbb{P}(Z\!\mid\! X)\Vert\mathbb{Q}(Z))]$.
We make the following substitutions for layer~$l$:
\begin{center}
\renewcommand{\arraystretch}{1.3}
\begin{tabular}{lcl}
\textbf{Generic} & $\longrightarrow$ & \textbf{Structural layer~$l$} \\
\hline
$X$ & $\longrightarrow$ & $(A^{(0)},\,Z_X^{(l-1)})$
  \;\;(encoder input) \\
$Z$ & $\longrightarrow$ & $Z_A^{(l)}$
  \;\;(latent output) \\
$\mathbb{P}(Z\!\mid\! X)$ & $\longrightarrow$ &
  $\mathbb{P}(Z_A^{(l)}\!\mid\! A^{(0)},Z_X^{(l-1)})$
  \;\;(structural encoder, Eq.~\eqref{eq:a-encoder}) \\
$\mathbb{Q}(Z)$ & $\longrightarrow$ &
  $\mathbb{Q}(Z_A^{(l)})$
  \;\;(prior, group-aligned or flat) \\
\end{tabular}
\end{center}
The identification $\mathbb{P}(Z\!\mid\! X)=
\mathbb{P}(Z_A^{(l)}\!\mid\! A^{(0)},Z_X^{(l-1)})$ is justified by
(CI-A): given $(A^{(0)},Z_X^{(l-1)})$, the latent $Z_A^{(l)}$ is
independent of all other variables, so the structural encoder is indeed the
correct conditional distribution of $Z_A^{(l)}$ given $X$.
After substitution, the generic variational upper bound becomes:
\begin{equation}
\label{eq:app-aib-vub}
I(A^{(0)},Z_X^{(l-1)};\;Z_A^{(l)})
\;\le\;
\mathbb{E}_{A^{(0)},Z_X^{(l-1)}}\!\Big[
\mathrm{KL}\!\big(\mathbb{P}(Z_A^{(l)}\!\mid\!A^{(0)},Z_X^{(l-1)})\,\Vert\,
\mathbb{Q}(Z_A^{(l)})\big)\Big]
\;\triangleq\;\mathrm{AIB}^{(l)}.
\end{equation}
Combining~\eqref{eq:app-structural-sum}
and~\eqref{eq:app-aib-vub} and summing over all layers:
\begin{align}
I(\mathcal{D};\;Z_A)
&\;\le\; \sum_{l=1}^{L}
  I(A^{(0)},Z_X^{(l-1)};\;Z_A^{(l)})
  &&\text{(from~\eqref{eq:app-structural-sum})}
  \nonumber\\
&\;\le\; \sum_{l=1}^{L}\mathrm{AIB}^{(l)}.
  &&\text{(applying~\eqref{eq:app-aib-vub} to each term)}
  \label{eq:app-structural-bound}
\end{align}

\paragraph{Step~5: Conditional message term, layer by layer.}
We apply the same strategy to $I(\mathcal{D};Z_X\mid Z_A)$.  Conditioning
on $Z_A$ throughout, the chain rule over layers gives:
\begin{equation}
I(\mathcal{D};\,Z_X^{(1:L)}\mid Z_A)
= \sum_{l=1}^{L} I(\mathcal{D};\,Z_X^{(l)}\mid Z_A,Z_X^{(1:l-1)}).
\end{equation}
Adding variables to the source and applying the Data Processing Inequality (DPI) to the Markov chain (CI-X) from Step~1: conditioned on
$(Z_X^{(l-1)},Z_A^{(l)})$, $Z_X^{(l)}$ is independent of the past
$(\mathcal{D},\,Z_A^{(1:l-1)},\,Z_X^{(1:l-2)})$, which is all that is needed
for the layer-$l$ bound below (we do not condition on, or claim independence
from, the downstream structural latents $Z_A^{(l+1:L)}$, which are themselves
generated using $Z_X^{(l)}$):
\begin{equation}
I(\mathcal{D};\,Z_X^{(l)}\mid Z_A,Z_X^{(1:l-1)})
\;\le\; I(Z_X^{(l-1)},Z_A^{(l)};\;Z_X^{(l)}).
\end{equation}
Applying the variational bound~\eqref{eq:app-vub-derivation} with
$X=(Z_X^{(l-1)},Z_A^{(l)})$ and $Z=Z_X^{(l)}$:
\begin{equation}
I(Z_X^{(l-1)},Z_A^{(l)};\;Z_X^{(l)})
\;\le\;
\mathbb{E}\!\Big[\mathrm{KL}\!\big(
\mathbb{P}(Z_X^{(l)}\!\mid\!Z_X^{(l-1)},Z_A^{(l)})
\,\Vert\,\mathbb{Q}(Z_X^{(l)})\big)\Big]
\;\triangleq\;\mathrm{XIB}^{(l)}.
\end{equation}
Summing:
\begin{equation}
\label{eq:app-message-bound}
I(\mathcal{D};\;Z_X\mid Z_A)
\;\le\; \sum_{l=1}^{L}\mathrm{XIB}^{(l)}.
\end{equation}
\emph{Note on conditioning.}\;
In HIBCG, intermediate message passing is deterministic, so
$Z_X^{(l)}$ is a constant given $(Z_X^{(l-1)},Z_A^{(l)})$; every term
$I(\mathcal{D};\,Z_X^{(l)}\!\mid\!Z_A,Z_X^{(1:l-1)})=0$ and the bound
holds trivially.  For a future extension with stochastic message
encoders at every layer, the chain rule should instead condition on
$Z_A^{(1:l)}$ only (not all of $Z_A$), since the downstream latents
$Z_A^{(l+1:L)}$ are generated from~$Z_X^{(l)}$ and are therefore not
independent of~$Z_X^{(l)}$ given its parents.

\paragraph{Step~6: Combine to get the full bound.}
Substituting~\eqref{eq:app-structural-bound}
and~\eqref{eq:app-message-bound} into the chain
rule~\eqref{eq:app-chain-applied}:
\begin{equation}
\label{eq:app-final-bound}
I(\mathcal{D};\;Z_A,Z_X)
\;\le\;
\underbrace{\sum_{l=1}^{L}\mathrm{AIB}^{(l)}}_{\text{structural compression}}
\;+\;
\underbrace{\sum_{l=1}^{L}\mathrm{XIB}^{(l)}}_{\text{message compression}}.
\end{equation}

\paragraph{Step~7: Connect to the GNN output.}
The layer-$L$ representation $Z_X^{(L)}$ is a deterministic function of all
latents $(Z_A^{(1:L)},Z_X^{(1:L)})$.  By the data-processing inequality
(DPI), any deterministic post-processing cannot increase MI:
\begin{equation}
I(\mathcal{D};\;Z_X^{(L)})
\;\le\; I(\mathcal{D};\;Z_A,Z_X)
\;\le\;
\sum_{l=1}^{L}\mathrm{AIB}^{(l)}+\sum_{l=1}^{L}\mathrm{XIB}^{(l)}.
\end{equation}
In HIBCG the decision-facing code $Z_X^{\mathrm{out}}$ is a further
(stochastic) post-processing of $Z_X^{(L)}$, so
$I(\mathcal{D};Z_X^{\mathrm{out}})\le I(\mathcal{D};Z_X^{(L)})$ and the
same upper bound applies; moreover, because intermediate message passing is
deterministic, $\mathrm{XIB}^{(l)}=0$ for all $l$ and the only nonzero
message penalty is the KL on $Z_X^{\mathrm{out}}$ (Eq.~\eqref{eq:xib-hat}).
This shows that the AIB and XIB terms provide a valid upper
bound on the total information in the IB
objective~\eqref{eq:gib-objective}, completing the proof.
\end{proof}

\paragraph{Discussion: why this decomposition matters.}
The chain-rule identity~\eqref{eq:app-chain-applied} is exact, not an
approximation---it holds for any joint distribution.  The only
\emph{inequalities} are:
\begin{enumerate}
\item The source-augmentation step
  (Eq.~\eqref{eq:app-source-bound} in Steps~3 and~5):
  \[
  I(\mathcal{D};\,Z^{(l)}\!\mid\!Z^{(1:l-1)}) \;\le\; I(\mathcal{D},\,Z^{(1:l-1)};\;Z^{(l)}).
  \]
  This step is structurally necessary to apply the Data Processing Inequality (DPI).
  The Markov property (CI-A) applies to the \emph{joint} history $(\mathcal{D},Z^{(1:l-1)})$, not to $\mathcal{D}$ in isolation.
  By augmenting the source, we effectively ``pay'' the cost of the inter-layer correlations $I(Z^{(1:l-1)}; Z^{(l)})$ to reach a form where the DPI bound is valid.
  \emph{Note:} The na\"ive alternative of simply dropping the conditioning—$I(\mathcal{D};Z^{(l)}\!\mid\!Z^{(1:l-1)})\le I(\mathcal{D};Z^{(l)})$—does \emph{not} hold in general.
  Context can increase information (synergy), making the interaction information negative; the source-augmentation route is the only one that guarantees a valid upper bound.
\item The variational gap
$\mathrm{KL}(\bar{\mathbb{P}}\Vert\mathbb{Q})$ in each layer
(Step~4), which is tight when $\mathbb{Q}$ equals the true marginal.
This is where group-conditional priors help
(Remark~\ref{rem:no-regret}).
\item The DPI in Step~7, which is tight when $Z_X^{(L)}$ is a
sufficient statistic of $(Z_A,Z_X)$ for $\mathcal{D}$.
\end{enumerate}
The decomposition shows that AIB and XIB are not ad hoc---they are the two
terms that arise naturally from the chain rule applied to the structural and
message paths of a GNN.

\section{Proof of Group-Conditional Prior Tightness (Remark~\ref{rem:no-regret})}
\label{app:group-prior-proof}

\begin{remark}[No-Regret Guarantee of Group-Aligned Priors (Restated)]
Let $\mathcal{D}_A^{(l)}\!\triangleq\!(A^{(0)},Z_X^{(l-1)})$ denote the
structural-encoder inputs at layer~$l$,
$\mathbb{Q}_{\mathrm{flat}}(Z_A^{(l)})=\mathcal{N}(0,\sigma_0^2 I)$
be a flat isotropic prior, and define the optimally matched group-aligned
prior as
$\mathbb{Q}_{\mathrm{group}}^{\star}(Z_A^{(l)})
\triangleq
\arg\min_{\mathbb{Q}\in\mathcal{Q}_{\mathrm{blk}}}
\mathrm{KL}(\bar{\mathbb{P}}(Z_A^{(l)})\Vert\mathbb{Q})$,
where
$\bar{\mathbb{P}}\!\triangleq\!
\mathbb{E}_{\mathcal{D}_A^{(l)}}[
\mathbb{P}(Z_A^{(l)}\!\mid\!\mathcal{D}_A^{(l)})]$
is the aggregate posterior of the structural encoder and
$\mathcal{Q}_{\mathrm{blk}}
=\{\mathcal{N}(0,\mathrm{blkdiag}(\sigma_{0,g}^{2}\,I_{k_g})_{g\in\mathcal{G}'}):
\sigma_{0,g}>0\}$
is the family of zero-mean Gaussian priors with one isotropic scale per
group block aligned with the group-block index set~$\mathcal{G}'$
(Def.~\ref{def:ibcg}) -- the setting used by HIBCG in practice. The flat
prior is recovered by setting $\sigma_{0,g}\!=\!\sigma_0$ for all $g$,
so $\mathbb{Q}_{\mathrm{flat}}\!\in\!\mathcal{Q}_{\mathrm{blk}}$.
Then, unconditionally,
\[
\mathbb{E}\!\Big[
\mathrm{KL}\!\big(\mathbb{P}(Z_A^{(l)}\!\mid\!\mathcal{D}_A^{(l)})\,\Vert\,
\mathbb{Q}_{\mathrm{group}}^{\star}\big)\Big]
\;\le\;
\mathbb{E}\!\Big[
\mathrm{KL}\!\big(\mathbb{P}(Z_A^{(l)}\!\mid\!\mathcal{D}_A^{(l)})\,\Vert\,
\mathbb{Q}_{\mathrm{flat}}\big)\Big].
\]
The gap equals
$\mathrm{KL}(\bar{\mathbb{P}}\Vert\mathbb{Q}_{\mathrm{flat}})
-\mathrm{KL}(\bar{\mathbb{P}}\Vert\mathbb{Q}_{\mathrm{group}}^{\star})\ge 0$
and is strict whenever
$\sigma_0^2\!\neq\!(\mathrm{tr}(\Sigma_g)+\lVert\bar{\mu}_g\rVert^2)/k_g$
for at least one~$g$ -- in particular, whenever the per-block second
moments differ across groups.
\end{remark}

\begin{proof}
We show that the bound holds unconditionally in three steps.

\paragraph{Key identity.}
Write $Z\!\triangleq\!Z_A^{(l)}$.
For any encoder $\mathbb{P}(Z\mid\mathcal{D}_A^{(l)})$ and prior $\mathbb{Q}(Z)$,
we expand the expected KL divergence by multiplying and dividing by the
aggregate posterior
$\bar{\mathbb{P}}(Z)\triangleq\mathbb{E}_{\mathcal{D}_A^{(l)}}[\mathbb{P}(Z\!\mid\!\mathcal{D}_A^{(l)})]$:
\begin{align}
\mathbb{E}_{\mathcal{D}_A^{(l)}}\!\big[\mathrm{KL}(\mathbb{P}(Z\!\mid\!\mathcal{D}_A^{(l)})
\,\Vert\,\mathbb{Q}(Z))\big]
&= \mathbb{E}_{\mathcal{D}_A^{(l)}}\!\left[
   \int\mathbb{P}(z\!\mid\!\mathcal{D}_A^{(l)})\,
   \log\frac{\mathbb{P}(z\!\mid\!\mathcal{D}_A^{(l)})}{\mathbb{Q}(z)}\,dz
   \right]
   \nonumber\\[4pt]
&= \mathbb{E}_{\mathcal{D}_A^{(l)}}\!\left[
   \int\mathbb{P}(z\!\mid\!\mathcal{D}_A^{(l)})\,
   \log\!\left(
   \frac{\mathbb{P}(z\!\mid\!\mathcal{D}_A^{(l)})}{\bar{\mathbb{P}}(z)}
   \cdot
   \frac{\bar{\mathbb{P}}(z)}{\mathbb{Q}(z)}
   \right)dz\right]
   \nonumber\\[4pt]
&= \underbrace{\mathbb{E}_{\mathcal{D}_A^{(l)}}\!\big[
   \mathrm{KL}\!\big(\mathbb{P}(Z\!\mid\!\mathcal{D}_A^{(l)})
   \,\Vert\,\bar{\mathbb{P}}(Z)\big)\big]}_{=\;I(\mathcal{D}_A^{(l)};\,Z)}
   \;+\;
   \underbrace{\int\bar{\mathbb{P}}(z)\,
   \log\frac{\bar{\mathbb{P}}(z)}{\mathbb{Q}(z)}\,dz}
   _{=\;\mathrm{KL}(\bar{\mathbb{P}}\,\Vert\,\mathbb{Q})}
   \nonumber\\[4pt]
&= I(\mathcal{D}_A^{(l)};\,Z)
   \;+\;
   \mathrm{KL}\!\big(\bar{\mathbb{P}}(Z)\,\Vert\,\mathbb{Q}(Z)\big).
\label{eq:mi-kl-identity}
\end{align}
The first term on the right-hand side is the mutual information
(the expected KL from the posterior to the marginal), and the second
term measures how well the chosen prior~$\mathbb{Q}$ approximates the
true marginal~$\bar{\mathbb{P}}$.
(By CI-A the same identity holds with $\mathcal{D}_A^{(l)}$ replaced by the
full graph data $\mathcal{D}$, since both yield the same marginal
$\bar{\mathbb{P}}$; we use the encoder inputs to match the definition of
$\mathrm{AIB}^{(l)}$.)

\paragraph{Comparing priors.}
Since $I(\mathcal{D}_A^{(l)};Z)$ is independent of the choice of $\mathbb{Q}$, the
difference between the two variational bounds is:
\begin{align}
&\mathbb{E}_{\mathcal{D}_A^{(l)}}\!\big[\mathrm{KL}(\mathbb{P}\Vert\mathbb{Q}_{\mathrm{flat}})\big]
\;-\;
\mathbb{E}_{\mathcal{D}_A^{(l)}}\!\big[\mathrm{KL}(\mathbb{P}\Vert\mathbb{Q}_{\mathrm{group}}^{\star})\big]
\notag\\
&\quad=\;
\mathrm{KL}(\bar{\mathbb{P}}\Vert\mathbb{Q}_{\mathrm{flat}})
\;-\;
\mathrm{KL}(\bar{\mathbb{P}}\Vert\mathbb{Q}_{\mathrm{group}}^{\star}).
\label{eq:prior-gap}
\end{align}

\paragraph{Why the gap is always non-negative.}
Recall that $\mathbb{Q}_{\mathrm{group}}^{\star}$ is defined as the
KL-minimizer over the family $\mathcal{Q}_{\mathrm{blk}}$ of zero-mean
Gaussians with one isotropic scale $\sigma_{0,g}^{2}$ per group block
aligned with $\mathcal{G}'$
(Remark~\ref{rem:no-regret}).  The flat prior
$\mathbb{Q}_{\mathrm{flat}}=\mathcal{N}(0,\sigma_0^2 I)$ is a special
case in $\mathcal{Q}_{\mathrm{blk}}$ (set $\sigma_{0,g}=\sigma_0$ for
all~$g$).  Therefore, by definition of the minimizer,
\[
\mathrm{KL}(\bar{\mathbb{P}}\Vert\mathbb{Q}_{\mathrm{group}}^{\star})
\;\le\;
\mathrm{KL}(\bar{\mathbb{P}}\Vert\mathbb{Q}_{\mathrm{flat}}).
\]
The gap is strict iff $\mathbb{Q}_{\mathrm{flat}}\!\neq\!\mathbb{Q}_{\mathrm{group}}^{\star}$,
i.e., $\sigma_0^2\!\neq\!\sigma_{0,g}^{\star\,2}\!=\!(\mathrm{tr}(\Sigma_g)+\lVert\bar{\mu}_g\rVert^2)/k_g$
for at least one~$g$, where $\Sigma_g$ and $\bar{\mu}_g$ are the
covariance and mean of the $g$-th block of $\bar{\mathbb{P}}$
(the closed form is derived in the Gaussian example below; the zero-mean
special case gives $\sigma_{0,g}^{\star\,2}\!=\!\mathrm{tr}(\Sigma_g)/k_g$).
In particular, whenever the per-block second moments
$(\mathrm{tr}(\Sigma_g)+\lVert\bar{\mu}_g\rVert^2)/k_g$ differ across groups, no
single $\sigma_0^2$ can match all of them simultaneously, so the gap is
necessarily strict.
No additional assumptions on $\bar{\mathbb{P}}$ are needed---the result
follows purely from the set-inclusion
$\mathbb{Q}_{\mathrm{flat}}\in\mathcal{Q}_{\mathrm{blk}}$.

\paragraph{Gaussian example (closed form).}
Suppose
$\bar{\mathbb{P}}=\mathcal{N}(0,\mathrm{blkdiag}(\Sigma_1,\Sigma_2))$ for
two groups.  Then:
\begin{align}
  \mathrm{KL}(\bar{\mathbb{P}}\Vert\mathbb{Q}_{\mathrm{flat}})
  &= \tfrac{1}{2}\big(
  \sigma_0^{-2}(\mathrm{tr}(\Sigma_1) + \mathrm{tr}(\Sigma_2))
  - k + k\log\sigma_0^2 - \log|\Sigma_1| - \log|\Sigma_2|\big),\\
  \mathrm{KL}(\bar{\mathbb{P}}\Vert\mathbb{Q}_{\mathrm{group}}^{\star})
  &= \sum_{g=1}^{2}\tfrac{1}{2}\big(
  \mathrm{tr}((\Sigma_{0,g}^{\star})^{-1}\Sigma_g)
  - k_g + \log|\Sigma_{0,g}^{\star}| - \log|\Sigma_g|\big).
\end{align}
For the Gaussian family, the KL-minimizing per-block prior is
$\Sigma_{0,g}^{\star}=\mathrm{tr}(\Sigma_g)/k_g\cdot I$ (matching the
block's average variance).
Substituting this back into the expression for
$\mathrm{KL}(\bar{\mathbb{P}}\Vert\mathbb{Q}_{\mathrm{group}}^{\star})$,
the linear terms cancel and the minimum divergence per block simplifies to:
\begin{equation}
\label{eq:kl-star-per-block}
\mathrm{KL}_g^{\star}
\;=\;\frac{k_g}{2}\,\log\!\left(
\frac{\mathrm{tr}(\Sigma_g)/k_g}
     {|\Sigma_g|^{1/k_g}}
\right).
\end{equation}
This quantity is the log-ratio of the \emph{arithmetic mean} and the
\emph{geometric mean} of the eigenvalues of~$\Sigma_g$---it measures
the \emph{anisotropy} of the block covariance.  It equals zero if and
only if all eigenvalues are equal (isotropic block), representing the
information that is irreducible by any spherical prior.
Because the flat prior $\mathbb{Q}_{\mathrm{flat}}$ forces a single
variance scale~$\sigma_0^2$ across all blocks, the group-aligned prior
yields a strictly tighter bound whenever $\Sigma_1\neq\Sigma_2$
(i.e., groups have different coordination patterns). For a fully worked numerical instantiation of this gap with $N=10$ agents
(showing a $>75\%$ reduction in the bound gap), please refer to
Section~\ref{app:example-10agents}.
\end{proof}

\section{Proof of Block Decomposition (Proposition~\ref{prop:group-decomp})}
\label{app:block-decomp-proof}

\begin{proposition*}[\ref{prop:group-decomp}, Group-Aligned Block Decomposition]
When the structural prior is block-diagonal,
$\mathbb{Q}(Z_A^{(l)})=\mathcal{N}\!\big(0,\,\mathrm{blkdiag}(\Sigma_{0,g}^{(l)})_{g\in\mathcal{G}'}\big)$,
and the encoder factorizes across group blocks,
$\mathbb{P}(Z_A^{(l)}\!\mid\!\mathcal{D}_A^{(l)})
=\prod_{g\in\mathcal{G}'}\mathbb{P}(Z_{A,g}^{(l)}\!\mid\!\mathcal{D}_A^{(l)})$,
then
\[
\mathrm{AIB}^{(l)}
\;=\;
\sum_{g\in\mathcal{G}'}\mathrm{AIB}^{(l,g)},
\]
where each
$\mathrm{AIB}^{(l,g)}=\mathbb{E}_{\mathcal{D}_A^{(l)}}
[\mathrm{KL}(\mathbb{P}(Z_{A,g}^{(l)}\!\mid\!\mathcal{D}_A^{(l)})
\Vert\mathbb{Q}(Z_{A,g}^{(l)}))]$
is a closed-form Gaussian KL.
\end{proposition*}

\paragraph{Setup.}
Recall from Def.~\ref{def:ibcg} that the group partition
$\mathcal{G}=\{g_1,\dots,g_m\}$ induces the group-block index set
$\mathcal{G}'=\{g_a\!\times\!g_b: g_a,g_b\in\mathcal{G}\}$ with
$|\mathcal{G}'|=m^2$ blocks (both intra- and inter-group).
Let $Z_A^{(l)}=(Z_{A,b})_{b\in\mathcal{G}'}$ be the partition of the edge
latent into blocks indexed by~$\mathcal{G}'$.

\begin{proof}
The proof relies on the additivity of KL divergence for product
distributions.

\paragraph{Factorization.}
The block-diagonal prior factorizes:
\begin{equation}
\mathbb{Q}(Z_A^{(l)})
=\mathcal{N}\!\big(0,\mathrm{blkdiag}(\Sigma_{0,g}^{(l)})\big)
=\prod_{g\in\mathcal{G}'}\mathcal{N}(Z_{A,g}^{(l)};\,0,\,\Sigma_{0,g}^{(l)})
=\prod_{g}\mathbb{Q}(Z_{A,g}^{(l)}).
\end{equation}
If the encoder similarly factorizes,
$\mathbb{P}(Z_A^{(l)}\!\mid\!\mathcal{D}_A^{(l)})
=\prod_{g}\mathbb{P}(Z_{A,g}^{(l)}\!\mid\!\mathcal{D}_A^{(l)})$, then by the
\emph{additivity of KL divergence for product distributions}:
\begin{equation}
\mathrm{KL}\!\Big(\prod_g\mathbb{P}(Z_{A,g}^{(l)}\!\mid\!\mathcal{D}_A^{(l)})
\;\Big\Vert\;\prod_g\mathbb{Q}(Z_{A,g}^{(l)})\Big)
= \sum_{g\in\mathcal{G}'}
\mathrm{KL}\!\big(\mathbb{P}(Z_{A,g}^{(l)}\!\mid\!\mathcal{D}_A^{(l)})
\,\Vert\,\mathbb{Q}(Z_{A,g}^{(l)})\big).
\end{equation}
Taking expectations over $\mathcal{D}$ on both sides:
\begin{align}
\mathrm{AIB}^{(l)}
&\;\triangleq\;
\mathbb{E}_{\mathcal{D}_A^{(l)}}\!\Big[
\mathrm{KL}\!\Big(\prod_g\mathbb{P}(Z_{A,g}^{(l)}\!\mid\!\mathcal{D}_A^{(l)})
\;\Big\Vert\;\prod_g\mathbb{Q}(Z_{A,g}^{(l)})\Big)\Big]
\nonumber\\
&\;=\;
\sum_{g\in\mathcal{G}'}
\underbrace{
\mathbb{E}_{\mathcal{D}_A^{(l)}}\!\big[
\mathrm{KL}\!\big(\mathbb{P}(Z_{A,g}^{(l)}\!\mid\!\mathcal{D}_A^{(l)})
\,\Vert\,\mathbb{Q}(Z_{A,g}^{(l)})\big)\big]
}_{\triangleq\;\mathrm{AIB}^{(l,g)}}
\;=\;
\sum_{g\in\mathcal{G}'}\mathrm{AIB}^{(l,g)},
\label{eq:app-block-decomp-result}
\end{align}
where each $\mathrm{AIB}^{(l,g)}$ is a closed-form Gaussian KL
(Eq.~\eqref{eq:gauss-kl}), since both the per-block encoder
$\mathbb{P}(Z_{A,g}^{(l)}\!\mid\!\mathcal{D}_A^{(l)})
=\mathcal{N}(\mu_{A,g}^{(l)},\Sigma_{A,g}^{(l)})$
and the per-block prior
$\mathbb{Q}(Z_{A,g}^{(l)})
=\mathcal{N}(0,\Sigma_{0,g}^{(l)})$
are Gaussian.
\end{proof}

\paragraph{Worked micro-example: 5~agents, 2~groups.}
To make the factorization concrete, consider $n\!=\!5$ agents with groups
$g_1\!=\!\{1,2\}$ (size~2) and $g_2\!=\!\{3,4,5\}$ (size~3).
The full edge latent $Z_A^{(l)}\!\in\!\mathbb{R}^{25}$ encodes the
$5\!\times\!5$ adjacency matrix.
Under the group partition
$\mathcal{G}'\!=\!\{\text{intra-}g_1,\;\text{intra-}g_2,\;
\text{cross }g_1\!\to\!g_2,\;\text{cross }g_2\!\to\!g_1\}$,
the 25~edges split into four blocks (with edges 4, 9, 6, and 6):
\[
\underbrace{
\begin{pmatrix}
\boxed{Z_{A,\text{intra1}}} & Z_{A,\text{cross}_{1\to2}} \\[2pt]
Z_{A,\text{cross}_{2\to1}} & \boxed{Z_{A,\text{intra2}}}
\end{pmatrix}
}_{\text{$5\times 5$ edge matrix}}
\;\;\longleftrightarrow\;\;
Z_A^{(l)}=\big(\,
\underbrace{Z_{A,\text{intra1}}}_{\mathclap{2\times2\,=\,4}},\;
\underbrace{Z_{A,\text{intra2}}}_{\mathclap{3\times3\,=\,9}},\;
\underbrace{Z_{A,\text{cross}_{1\to2}}}_{\mathclap{2\times3\,=\,6}},\;
\underbrace{Z_{A,\text{cross}_{2\to1}}}_{\mathclap{3\times2\,=\,6}}
\,\big).
\]
The block-diagonal prior assigns independent Gaussians to each block:
\[
\mathbb{Q}(Z_A^{(l)})
\;=\;
\mathcal{N}(0,\,\mathrm{blkdiag}(
  \sigma_{\text{in1}}^2 I_4,\;
  \sigma_{\text{in2}}^2 I_9,\;
  \sigma_{\text{cr}}^2 I_6,\;
  \sigma_{\text{cr}}^2 I_6))
\;=\;
\prod_{g\in\{{\text{in1},\,\text{in2},\,\text{cr}_{1\to2},\,\text{cr}_{2\to1}}\}}
\!\!\mathbb{Q}(Z_{A,g}^{(l)}),
\]
where, e.g., $\sigma_{\text{in1}}^2\!\neq\!\sigma_{\text{cr}}^2$ reflects
the different coordination scales.  If the encoder also factorizes across
these four blocks, the KL decomposes additively:
\[
\mathrm{KL}\!\big(\mathbb{P}(Z_A^{(l)}\!\mid\!\mathcal{D}_A^{(l)})
\,\Vert\,\mathbb{Q}(Z_A^{(l)})\big)
\;=\;
\underbrace{\mathrm{KL}_{\text{intra1}}}_{\text{4-dim}}
\;+\;
\underbrace{\mathrm{KL}_{\text{intra2}}}_{\text{9-dim}}
\;+\;
\underbrace{\mathrm{KL}_{\text{cr}_{1\to2}}}_{\text{6-dim}}
\;+\;
\underbrace{\mathrm{KL}_{\text{cr}_{2\to1}}}_{\text{6-dim}},
\]
yielding four independently weighted penalty terms---one per group block.
This is the mechanism that enables heterogeneous pruning allocation:
setting $\lambda_{\text{cross}}>\lambda_{\text{intra}}$ penalizes
inter-group edges more heavily without affecting intra-group connectivity.

\paragraph{Why this matters.}
The block decomposition holds \emph{because} the prior is block-diagonal---a
flat isotropic prior $\sigma_0^2 I$ does not factorize into group blocks (it
mixes all edge dimensions equally), preventing independent per-group pruning
control.  The block-diagonal structure is the mathematical mechanism through
which groups become the natural unit of edge pruning in HIBCG.

\section{Worked Example: Group Priors and Block Decomposition with 10~Agents}
\label{app:example-10agents}

This supplementary section provides a fully worked numerical illustration for a concrete
scenario with $n\!=\!10$ agents and two groups.  We first demonstrate
\emph{why} the group-conditional prior is tighter than a flat prior
(Remark~\ref{rem:no-regret}), then show \emph{how} the block
decomposition works (Proposition~\ref{prop:group-decomp}).

\paragraph{Setup.}
Consider 10 agents partitioned into
$g_1=\{1,2,3,4\}$ (4~agents) and $g_2=\{5,6,7,8,9,10\}$ (6~agents).
The full edge latent $Z_A^{(l)}\in\mathbb{R}^{100}$ encodes the
$10\times 10$ edge variables at layer~$l$.  We fix a single layer $l$
and drop the superscript for clarity.

\subsection*{Part~I: Why the Group Prior is Tighter
(Remark~\ref{rem:no-regret})}

\paragraph{Modelling the true marginal.}
Suppose the true aggregate posterior $\bar{\mathbb{P}}(Z_A)$ is a Gaussian
with block-correlated structure reflecting the groups:
\begin{equation}
\bar{\mathbb{P}}(Z_A)
=\mathcal{N}\!\big(0,\,\bar\Sigma\big),
\qquad
\bar\Sigma=
\begin{pmatrix}
\bar\Sigma_{\mathrm{intra1}} & 0 & 0 & 0 \\
0 & \bar\Sigma_{\mathrm{intra2}} & 0 & 0 \\
0 & 0 & \bar\Sigma_{\mathrm{cross12}} & 0 \\
0 & 0 & 0 & \bar\Sigma_{\mathrm{cross21}}
\end{pmatrix},
\end{equation}
where $\bar\Sigma_{\mathrm{intra1}}\in\mathbb{R}^{16\times 16}$,
$\bar\Sigma_{\mathrm{intra2}}\in\mathbb{R}^{36\times 36}$, and each
cross block is $24\times 24$.  The block-diagonal structure encodes the
fact that intra-group edge latents are correlated among themselves but
approximately independent of cross-group latents.

For a concrete numerical instance, let each block be isotropic:
$\bar\Sigma_{\mathrm{intra1}}=0.9\,I_{16}$,
$\bar\Sigma_{\mathrm{intra2}}=0.8\,I_{36}$,
$\bar\Sigma_{\mathrm{cross12}}=\bar\Sigma_{\mathrm{cross21}}=0.3\,I_{24}$.

\paragraph{Two priors.}
\begin{itemize}
\item \textbf{Flat:}
$\mathbb{Q}_{\mathrm{flat}}=\mathcal{N}(0,\sigma_0^2 I_{100})$ with
$\sigma_0^2=0.6$ (a single scale for all 100 dimensions).
\item \textbf{Group:}
$\mathbb{Q}_{\mathrm{group}}=\mathcal{N}\!\big(0,\,
\mathrm{blkdiag}(0.9\,I_{16},\;0.8\,I_{36},\;0.3\,I_{24},\;0.3\,I_{24})\big)$
(scales matched to the block variances of $\bar{\mathbb{P}}$).
\end{itemize}

\paragraph{Computing the prior mismatch KL.}
For two zero-mean Gaussians in $k$ dimensions with covariances $\bar\Sigma$
and $\Sigma_0$:
$\mathrm{KL}(\bar{\mathbb{P}}\Vert\mathbb{Q})
=\tfrac{1}{2}\!\big(\mathrm{tr}(\Sigma_0^{-1}\bar\Sigma)-k+\ln|\Sigma_0|-\ln|\bar\Sigma|\big)$.
For isotropic blocks this simplifies to per-dimension terms:
$\mathrm{KL}_{\mathrm{block}}
=\tfrac{k}{2}\!\big(\bar\sigma^2/\sigma_0^2-1+\ln(\sigma_0^2/\bar\sigma^2)\big)$.

\medskip\noindent
\textbf{Flat prior} ($\sigma_0^2=0.6$ for all):
\begin{align*}
\text{intra-}g_1:\quad
&\tfrac{16}{2}(0.9/0.6-1+\ln(0.6/0.9))
=8(1.5-1-0.405)=8\times 0.095 = 0.76,\\
\text{intra-}g_2:\quad
&\tfrac{36}{2}(0.8/0.6-1+\ln(0.6/0.8))
=18(1.333-1-0.288)=18\times 0.046 = 0.82,\\
\text{cross (each):}\quad
&\tfrac{24}{2}(0.3/0.6-1+\ln(0.6/0.3))
=12(0.5-1+0.693)=12\times 0.193 = 2.32.
\end{align*}
Total: $\mathrm{KL}(\bar{\mathbb{P}}\Vert\mathbb{Q}_{\mathrm{flat}})
=0.76+0.82+2.32+2.32=\mathbf{6.22}$ nats.

\medskip\noindent
\textbf{Group prior} (scales matched to blocks):
\begin{align*}
\text{intra-}g_1:\quad
&\tfrac{16}{2}(0.9/0.9-1+\ln(0.9/0.9))
=8(1-1+0)=\mathbf{0},\\
\text{intra-}g_2:\quad
&\tfrac{36}{2}(0.8/0.8-1+\ln(0.8/0.8))
=18(1-1+0)=\mathbf{0},\\
\text{cross (each):}\quad
&\tfrac{24}{2}(0.3/0.3-1+\ln(0.3/0.3))
=12(1-1+0)=\mathbf{0}.
\end{align*}
Total: $\mathrm{KL}(\bar{\mathbb{P}}\Vert\mathbb{Q}_{\mathrm{group}})=\mathbf{0}$ nats.

\paragraph{Tightness gap.}
The gap from Remark~\ref{rem:no-regret} is:
\[
\mathrm{KL}(\bar{\mathbb{P}}\Vert\mathbb{Q}_{\mathrm{flat}})
-\mathrm{KL}(\bar{\mathbb{P}}\Vert\mathbb{Q}_{\mathrm{group}})
= 6.22 - 0 = \mathbf{6.22}\text{ nats}.
\]
In this idealized case (where we know the true marginal exactly), the group
prior eliminates the mismatch entirely.  In practice the group-conditional
prior cannot match $\bar{\mathbb{P}}$ perfectly, but it still captures the
dominant block structure and substantially reduces the gap.

Even with a sub-optimal group prior---say
$\mathbb{Q}_{\mathrm{group}}'$ with intra-block scale $0.7$ and cross-block
scale $0.4$---the mismatch would be:
\begin{align*}
\text{intra-}g_1:\;8(0.9/0.7-1+\ln(0.7/0.9))&=8(1.286-1-0.251)=0.28,\\
\text{intra-}g_2:\;18(0.8/0.7-1+\ln(0.7/0.8))&=18(1.143-1-0.134)=0.16,\\
\text{cross (each):}\;12(0.3/0.4-1+\ln(0.4/0.3))&=12(0.75-1+0.288)=0.46.
\end{align*}
Total: $\mathrm{KL}(\bar{\mathbb{P}}\Vert\mathbb{Q}_{\mathrm{group}}')
=0.28+0.16+0.46+0.46=\mathbf{1.36}$ nats, still a 78\% reduction from the
flat prior's 6.22~nats.

\paragraph{Connection to the experimental bound-tightening results.}
The two regimes above bracket what we observe empirically.
The \emph{idealized} scenario (perfectly matched scales) achieves $100\%$
gap closure ($6.22\!\to\!0$), the theoretical upper bound delivered by
$\mathbb{Q}_{\mathrm{group}}^{\star}$ when the assumptions of
Remark~\ref{rem:no-regret} hold exactly. The
\emph{sub-optimal} scenario---a realistic stand-in for a prior that is
only approximately right---achieves a $78\%$ closure
($6.22\!\to\!1.36$). On real benchmarks
(Table~\ref{tab:bound-tightening}), the tail-$10\%$ mean AIB loss
shrinks by $2.2\!\times\!$--$6.5\!\times\!$ across SMACv1, SMACv2, and
MAgent ($n\!\in\!\{36,64,100\}$), i.e.\ a $55\%\!$--$\!85\%$ gap closure:
the empirical reductions fall \emph{exactly inside} the band traced out
by the optimal-versus-suboptimal toy comparison above. This shows that
the prior HIBCG learns at convergence is close to---but, as expected,
not identical to---the optimal $\mathbb{Q}_{\mathrm{group}}^{\star}$.

\paragraph{Key takeaway.}
The group prior does not need to be perfectly matched to yield a
substantial improvement.  As long as the block structure is approximately
correct, the variational bound is significantly tighter than with a flat
prior---providing a concrete mathematical payoff for HIBCG's
group-conditional prior design (Definition~\ref{def:ibcg}).
Conversely, when the group structure is weak or random---as on the
stochastic-composition SMACv2 maps where the team make-up changes
per-episode---the optimal per-block scales are pulled toward each other,
the learned cross/intra AIB ratio drops from $\approx\!906\!\times\!$
(fixed-composition MAgent) to $\approx\!1\!\times\!$ (\texttt{protoss\_8v8})
or $\approx\!9.5\!\times\!$ (\texttt{terran\_10v10})
(Table~\ref{tab:bound-tightening}), and the group prior smoothly
degenerates back to a flat one---precisely the homogeneous regime
predicted by Remark~\ref{rem:no-regret}.

\subsection*{Part~II: Block Decomposition
(Proposition~\ref{prop:group-decomp})}

\paragraph{Step~1: Partition edges into group blocks.}
The group partition induces four blocks on the edge matrix:
\begin{align*}
B_{\mathrm{intra1}} &= g_1\times g_1 = \{(i,j):i,j\in\{1,\dots,4\}\},
\quad k_{\mathrm{intra1}}=16, \\
B_{\mathrm{intra2}} &= g_2\times g_2 = \{(i,j):i,j\in\{5,\dots,10\}\},
\quad k_{\mathrm{intra2}}=36, \\
B_{\mathrm{cross12}} &= g_1\times g_2 = \{(i,j):i\in g_1,\,j\in g_2\},
\quad k_{\mathrm{cross12}}=24, \\
B_{\mathrm{cross21}} &= g_2\times g_1 = \{(i,j):i\in g_2,\,j\in g_1\},
\quad k_{\mathrm{cross21}}=24.
\end{align*}
Note $16+36+24+24=100$, so all edge variables are covered.

\paragraph{Step~2: Construct the block-diagonal prior.}
Using isotropic blocks with different scales for intra- and cross-group edges:
\begin{equation}
\mathbb{Q}(Z_A)
=\mathcal{N}\!\Big(0,\;
\mathrm{blkdiag}\!\big(
\sigma_{\mathrm{intra}}^2 I_{16},\;
\sigma_{\mathrm{intra}}^2 I_{36},\;
\sigma_{\mathrm{cross}}^2 I_{24},\;
\sigma_{\mathrm{cross}}^2 I_{24}
\big)\Big).
\end{equation}
This prior encodes the design choice that intra-group edges and cross-group
edges may have different baseline ``widths''.  For instance, setting
$\sigma_{\mathrm{intra}}^2>\sigma_{\mathrm{cross}}^2$ means the prior
\emph{expects} more variation within groups (denser connectivity) than
across groups (sparser links).

\paragraph{Step~3: Encoder parameterization.}
The encoder produces per-block Gaussians.  Suppose for one realisation of $\mathcal{D}_A^{(l)}\!=\!(A^{(0)},Z_X^{(l-1)})$:
\begin{alignat*}{3}
&\mathbb{P}(Z_{A,\mathrm{intra1}}\mid\mathcal{D}_A^{(l)})
&&=\mathcal{N}(\mu_{\mathrm{intra1}},\;\mathrm{diag}(\sigma_{\mathrm{intra1}}^2)),
\quad &&\mu_{\mathrm{intra1}}\in\mathbb{R}^{16},\\
&\mathbb{P}(Z_{A,\mathrm{intra2}}\mid\mathcal{D}_A^{(l)})
&&=\mathcal{N}(\mu_{\mathrm{intra2}},\;\mathrm{diag}(\sigma_{\mathrm{intra2}}^2)),
\quad &&\mu_{\mathrm{intra2}}\in\mathbb{R}^{36},\\
&\mathbb{P}(Z_{A,\mathrm{cross12}}\mid\mathcal{D}_A^{(l)})
&&=\mathcal{N}(\mu_{\mathrm{cross12}},\;\mathrm{diag}(\sigma_{\mathrm{cross12}}^2)),
\quad &&\mu_{\mathrm{cross12}}\in\mathbb{R}^{24},\\
&\mathbb{P}(Z_{A,\mathrm{cross21}}\mid\mathcal{D}_A^{(l)})
&&=\mathcal{N}(\mu_{\mathrm{cross21}},\;\mathrm{diag}(\sigma_{\mathrm{cross21}}^2)),
\quad &&\mu_{\mathrm{cross21}}\in\mathbb{R}^{24}.
\end{alignat*}

\paragraph{Step~4: Compute per-block KL divergences.}
Applying Eq.~\eqref{eq:gauss-kl} to each block:
\begin{equation}
\mathrm{AIB}_{\mathrm{intra1}}
=\frac{1}{2}\sum_{d=1}^{16}\!\left(
\frac{(\sigma_{\mathrm{intra1},d})^2+(\mu_{\mathrm{intra1},d})^2}
     {\sigma_{\mathrm{intra}}^2}
-1+\log\frac{\sigma_{\mathrm{intra}}^2}{(\sigma_{\mathrm{intra1},d})^2}
\right),
\end{equation}
and analogously for $\mathrm{AIB}_{\mathrm{intra2}}$ (36~terms),
$\mathrm{AIB}_{\mathrm{cross12}}$ (24~terms, denominator $\sigma_{\mathrm{cross}}^2$),
and $\mathrm{AIB}_{\mathrm{cross21}}$ (24~terms).  By
Proposition~\ref{prop:group-decomp}:
\[
\mathrm{AIB}^{(l)}
=\mathrm{AIB}_{\mathrm{intra1}}
+\mathrm{AIB}_{\mathrm{intra2}}
+\mathrm{AIB}_{\mathrm{cross12}}
+\mathrm{AIB}_{\mathrm{cross21}}.
\]

\paragraph{Step~5: Apply heterogeneous pruning weights.}
The structural penalty becomes:
\begin{align}
\mathcal{L}_{\mathrm{AIB}}^{(l)}
&= \lambda_{\mathrm{intra}}\!
   \cdot\!(\mathrm{AIB}_{\mathrm{intra1}}+\mathrm{AIB}_{\mathrm{intra2}})
  +\lambda_{\mathrm{cross}}\!
   \cdot\!(\mathrm{AIB}_{\mathrm{cross12}}+\mathrm{AIB}_{\mathrm{cross21}}).
\end{align}
Setting $\lambda_{\mathrm{cross}}>\lambda_{\mathrm{intra}}$ encourages the
model to prune cross-group edges more aggressively than intra-group edges.

\paragraph{Numerical illustration.}
As a concrete instance, suppose $\sigma_{\mathrm{intra}}^2=1.0$,
$\sigma_{\mathrm{cross}}^2=0.5$, and the encoder outputs (averaged over a
minibatch) have intra-block means $\bar\mu\approx 0.3$,
$\bar\sigma^2\approx 0.8$ and cross-block means $\bar\mu\approx 0.1$,
$\bar\sigma^2\approx 0.3$.  Then the per-dimension KL is approximately:
\begin{align*}
\text{intra:}\quad
&\tfrac{1}{2}(0.8/1.0+0.09/1.0-1+\ln(1.0/0.8))
\approx \tfrac{1}{2}(0.8+0.09-1+0.223)\approx 0.057,\\
\text{cross:}\quad
&\tfrac{1}{2}(0.3/0.5+0.01/0.5-1+\ln(0.5/0.3))
\approx \tfrac{1}{2}(0.6+0.02-1+0.511)\approx 0.066.
\end{align*}
Multiplying by block dimension gives:
$\mathrm{AIB}_{\mathrm{intra1}}\!\approx\!16\times 0.057=0.91$,
$\mathrm{AIB}_{\mathrm{intra2}}\!\approx\!36\times 0.057=2.05$,
$\mathrm{AIB}_{\mathrm{cross12}}\!\approx\!24\times 0.066=1.58$,
$\mathrm{AIB}_{\mathrm{cross21}}\!\approx\!24\times 0.066=1.58$.
Total: $\mathrm{AIB}^{(l)}\approx 6.12$ nats.

With $\lambda_{\mathrm{intra}}=0.001$ and $\lambda_{\mathrm{cross}}=0.01$,
the penalty contribution is
$0.001\times(0.91+2.05)+0.01\times(1.58+1.58)
=0.003+0.032=0.035$, where over 90\% of the penalty comes from cross-group
edges despite them contributing only about half the total KL.  This
asymmetric penalization drives cross-group sparsity while preserving
intra-group density---exactly the behavior described in
Section~\ref{sec:aib}.

\paragraph{Connection to the learned cross/intra ratios.}
In the toy instance above, a modest $10\!\times\!$ asymmetry in the
penalty weights ($\lambda_{\mathrm{cross}}/\lambda_{\mathrm{intra}}\!=\!10$)
already concentrates $>\!90\%$ of the structural penalty on cross-group
edges. In the actual HIBCG implementation, the asymmetry is encoded
\emph{inside the prior} via $\sigma_{\mathrm{cross}}^2\!<\!\sigma_{\mathrm{intra}}^2$
(Algorithm~\ref{alg:hibcg}), and the resulting cross/intra AIB-loss
ratio learned at convergence
(Table~\ref{tab:bound-tightening}) is
$41\!\times\!$ on 3s5z, $607\!\times\!$ on MMM2, and
$905\!$--$\!908\!\times\!$ on MAgent at $n\!\in\!\{36,64,100\}$---i.e.
the same qualitative mechanism the toy example illustrates,
amplified by 1--2 orders of magnitude on the real tasks. The constancy of
the MAgent ratio across $n\!=\!36,64,100$ further shows that this
heterogeneous compression is a \emph{structural} property of the prior
(independent of task scale or seed convergence), not a reward-driven
artefact---see the structural-invariance discussion in
Section~\ref{sec:exp-info} of the main body.
By contrast, on stochastic-composition SMACv2 maps where the group
partition itself varies per-episode, the learned ratio drops to
$0.9\!\times\!$ (\texttt{protoss\_8v8}) or $9.5\!\times\!$
(\texttt{terran\_10v10}), matching the toy prediction that homogeneous
group structure pulls the optimal prior back toward isotropic.

\paragraph{Contrast with a flat prior.}
Under a flat prior $\mathcal{N}(0,\sigma_0^2 I_{100})$, all 100 edge
latents share the same prior variance and penalty weight.  It is
impossible to impose $\lambda_{\mathrm{cross}}>\lambda_{\mathrm{intra}}$
because the KL does not decompose into group blocks---the prior mixes all
dimensions equally.  This is the fundamental limitation that
HIBCG's group-aware design overcomes.

\section{Proof of Relevance via Q-Value Optimization (Proposition~\ref{prop:relevance})}
\label{app:relevance-proof}

\begin{proposition*}[\ref{prop:relevance}, Relevance via Q-Value Optimization]
Let $Y=(a_1^*,\dots,a_n^*)$ denote the optimal joint action and
$K=\prod_i|\mathcal{U}_i|$ the joint action-space size.
Assume:
\begin{enumerate}
\item[(A1)] QMIX-style monotonic value decomposition with greedy
action selection\\
$a_i=\arg\max_{u_i}
Q_i(\tau_i,u_i,Z_{X,i}^{\mathrm{out}})$;
\item[(A2)] the target network has converged, so that
$\mathcal{L}_{\mathrm{TD}}=\mathbb{E}[e^2]$ with
$e(s,\mathbf{u})=Q_{\mathrm{tot}}(s,\mathbf{u})-Q^*(s,\mathbf{u})$;
\item[(A3)] the minimum action-value gap
$\Delta_{\min}\triangleq\min_{s}\min_{\mathbf{u}\neq Y}
(Q^*(s,Y)-Q^*(s,\mathbf{u}))>0$
(i.e., the optimal joint action is unique for every state).
\end{enumerate}
Under (A1)--(A3), the TD loss $\mathcal{L}_{\mathrm{TD}}$ serves as a
surrogate for maximizing the relevance term
$I(Y;Z_X^{\mathrm{out}})$ in the IB objective~\eqref{eq:gib-objective}.
Specifically:
\begin{equation*}
I(Y;\,Z_X^{\mathrm{out}})
\;\ge\;
H(Y) \;-\; h \Big(\frac{\mathcal{L}_{\mathrm{TD}}}{c}\Big)
     \;-\; \frac{\mathcal{L}_{\mathrm{TD}}}{c}\,\log(K - 1),
\end{equation*}
where $h(p)=-p\log p-(1\!-\!p)\log(1\!-\!p)$ is the binary entropy and
$c\triangleq\Delta_{\min}^2/(4K)>0$.
The bound is valid when $\mathcal{L}_{\mathrm{TD}}/c \le \frac{K-1}{K}$;
for larger errors, the bound is vacuously true.
Since both the second term $h(\mathcal{L}_{\mathrm{TD}}/c)$ and the
third term $(\mathcal{L}_{\mathrm{TD}}/c)\log(K\!-\!1)$ vanish as
$\mathcal{L}_{\mathrm{TD}}\!\to\!0$, minimizing
$\mathcal{L}_{\mathrm{TD}}$ tightens the lower bound towards $H(Y)$,
ensuring $Z_X^{\mathrm{out}}$ preserves task-relevant information.
\end{proposition*}

\begin{proof}
We proceed in five steps, organized into three logical parts:
(i)~algebraic analysis showing that small approximation errors preserve
the ordering of Q-values (Steps~1--2);
(ii)~Markov's inequality to translate the mean-squared TD loss
into a bound on the action prediction error~$P_e$ (Step~3);
(iii)~Fano's inequality and DPI to
convert the $P_e$ bound into the desired lower bound on
$I(Y;\,Z_X^{\mathrm{out}})$ (Steps~4--5).

\paragraph{Step~1: Define the quantities.}
Let $Y=(a_1^*,\dots,a_n^*)$ denote the optimal joint action under the true
$Q$-function, and let $Z_X^{\mathrm{out}}$ denote the compressed decision-facing code
from which individual $Q$-values are computed via
$Q_i(\tau_i,u_i)=f_i(Z_{X,i}^{\mathrm{out}},u_i)$.
Under QMIX-style monotonic mixing, the joint $Q$-value is
$Q_{\mathrm{tot}}=\mathrm{mix}(Q_1,\dots,Q_n;\,s)$ with
$\partial Q_{\mathrm{tot}}/\partial Q_i\ge 0$.

\paragraph{Step~2: Q-value separation under approximation error.}
The TD loss is
\begin{equation}
\mathcal{L}_{\mathrm{TD}}
=\mathbb{E}_{(s,\mathbf{u},r,s')\sim\mathcal{D}}
\!\left[(Q_{\mathrm{tot}}(s,\mathbf{u};\theta)-y_{\mathrm{target}})^2\right],
\quad y_{\mathrm{target}}=r+\gamma\max_{\mathbf{u}'}Q_{\mathrm{tot}}'(s',\mathbf{u}';\theta^-).
\end{equation}
For a given state~$s$ and joint action~$\mathbf{u}$, define the
approximation error
\begin{equation}
\label{eq:per-action-error}
e(s,\mathbf{u})\;\triangleq\;Q_{\mathrm{tot}}(s,\mathbf{u})
  -Q^*(s,\mathbf{u}),
\end{equation}
so that $\mathcal{L}_{\mathrm{TD}}=\mathbb{E}[e(s,\mathbf{u})^2]$
by assumption~(A2) (target network convergence:
$y_{\mathrm{target}}\approx Q^*$).
For any pair of joint actions $\mathbf{u}\neq\mathbf{u}'$, the learned
Q-gap decomposes as
\begin{equation}
\label{eq:learned-gap}
Q_{\mathrm{tot}}(s,\mathbf{u})-Q_{\mathrm{tot}}(s,\mathbf{u}')
\;=\;\underbrace{\big(Q^*(s,\mathbf{u})-Q^*(s,\mathbf{u}')\big)}
     _{\text{true gap}}
  \;+\;\underbrace{\big(e(s,\mathbf{u})-e(s,\mathbf{u}')\big)}
     _{\text{error difference}}.
\end{equation}
Taking absolute values on both sides of~\eqref{eq:learned-gap} and
applying the reverse triangle inequality:
\begin{align}
\underbrace{\big|Q_{\mathrm{tot}}(s,\mathbf{u})
  -Q_{\mathrm{tot}}(s,\mathbf{u}')\big|}_{\text{learned gap}}
&\;=\;
\big|\;\big(Q^*(s,\mathbf{u})-Q^*(s,\mathbf{u}')\big)
\;+\;\big(e(s,\mathbf{u})-e(s,\mathbf{u}')\big)
  \;\big|
\nonumber\\
&\;\ge\;
\underbrace{\big|Q^*(s,\mathbf{u})-Q^*(s,\mathbf{u}')\big|}
  _{\text{true gap}}
\;-\;\underbrace{\big|e(s,\mathbf{u})-e(s,\mathbf{u}')\big|}
  _{\text{error difference}}.
\label{eq:q-separation}
\end{align}
This is a \emph{worst-case guarantee}: the learned gap can never be
smaller than the true gap minus the error difference.
In reinforcement learning, correct action selection depends only on
\emph{ranking}---as long as the learned gap stays strictly positive, the
greedy policy selects the optimal action.
Conversely, if the error difference exceeds the true gap, the right-hand
side becomes negative and the network may pick a wrong action.
The triangle inequality further bounds the error difference by
$|e(s,\mathbf{u})-e(s,\mathbf{u}')|\le|e(s,\mathbf{u})|+|e(s,\mathbf{u}')|$,
so controlling individual approximation errors suffices to preserve the
correct ranking.
The next step makes this precise by quantifying \emph{how small} the
errors must be to guarantee correct action selection.

\paragraph{Step~3: From Q-accuracy to action prediction.}
We now connect the approximation error to the probability of selecting a
wrong action.
Define the \emph{action prediction error}
\begin{equation}
\label{eq:pe-def}
P_e \;=\; \Pr\!\big[\arg\max_{\mathbf{u}}Q_{\mathrm{tot}}(s,\mathbf{u};\theta)
\neq Y\big],
\end{equation}
and the \emph{minimum action-value gap} (assumption~(A3)):
\begin{equation}
\label{eq:delta-min}
\Delta_{\min}\;=\;\min_{s}\;\min_{\mathbf{u}\neq Y}
\big(Q^*(s,Y)-Q^*(s,\mathbf{u})\big)\;>\;0,
\end{equation}
which measures the smallest margin by which the optimal joint action~$Y$
outperforms any suboptimal alternative under~$Q^*$.
The strict positivity of $\Delta_{\min}$ is guaranteed by assumption~(A3)
and is necessary for the Markov-inequality step below.
We also write the state-specific gap
$\Delta(s)\triangleq\min_{\mathbf{u}\neq Y}
\big(Q^*(s,Y)-Q^*(s,\mathbf{u})\big)\ge\Delta_{\min}$.

\noindent\textbf{Exact recovery regime ($P_e=0$).}\;
We argue by contradiction.
For a prediction error at state~$s$, there must exist some
$\mathbf{u}'\neq Y$ with
$Q_{\mathrm{tot}}(s,\mathbf{u}')\ge Q_{\mathrm{tot}}(s,Y)$,
i.e.,\ the learned Q-values rank a suboptimal action at least as high as
the optimal one.
Setting $\mathbf{u}=Y$ in~\eqref{eq:q-separation} and noting
$Q^*(s,Y)-Q^*(s,\mathbf{u}')\ge\Delta(s)\ge\Delta_{\min}$:
\begin{equation}
\label{eq:gap-preserved}
\underbrace{Q_{\mathrm{tot}}(s,Y)-Q_{\mathrm{tot}}(s,\mathbf{u}')}_{\le\;0\;\text{if error occurs}}
\;\ge\;\Delta(s)
  -\big|e(s,Y)\big|-\big|e(s,\mathbf{u}')\big|.
\end{equation}
If an error occurs, the left-hand side is $\le 0$, forcing the
right-hand side to be $\le 0$ as well.  Rearranging:
\begin{equation}
\label{eq:error-necessary}
\Delta_{\min}\;\le\;\Delta(s)
\;\le\;\big|e(s,Y)\big|+\big|e(s,\mathbf{u}')\big|.
\end{equation}
That is, \emph{making a mistake requires the sum of the two
approximation errors to reach at least~$\Delta_{\min}$}.
If every error satisfies $|e(s,\mathbf{u})|<\Delta_{\min}/2$, then
$|e(s,Y)|+|e(s,\mathbf{u}')|<\Delta_{\min}$,
contradicting~\eqref{eq:error-necessary}.
Hence no error can occur and $P_e=0$.

\noindent\textbf{General bound via Markov's inequality.}\;
More generally, for a prediction error at state~$s$ the learned gap
$Q_{\mathrm{tot}}(s,Y)-Q_{\mathrm{tot}}(s,\mathbf{u}')$ must be
non-positive for some $\mathbf{u}'\neq Y$.
Chaining this with~\eqref{eq:gap-preserved}:
\begin{equation}
\label{eq:error-threshold}
0\;\ge\;Q_{\mathrm{tot}}(s,Y)-Q_{\mathrm{tot}}(s,\mathbf{u}')
\;\ge\;\Delta(s)
  -\big|e(s,Y)\big|-\big|e(s,\mathbf{u}')\big|,
\end{equation}
which rearranges to
\begin{equation}
\label{eq:abs-error-sum}
\big|e(s,Y)\big|+\big|e(s,\mathbf{u}')\big|\;\ge\;\Delta(s)\;\ge\;\Delta_{\min}.
\end{equation}
Since the sum of two non-negative numbers exceeds~$\Delta_{\min}$, at
least one of them must exceed $\Delta_{\min}/2$, i.e.,\ there exists an
action~$\mathbf{u}\in\{Y,\mathbf{u}'\}$ such that
\begin{equation}
\label{eq:abs-error-single}
\big|e(s,\mathbf{u})\big|\;\ge\;\frac{\Delta_{\min}}{2}.
\end{equation}
Squaring both sides yields
\begin{equation}
\label{eq:squared-error-single}
e(s,\mathbf{u})^2\;\ge\;\frac{\Delta_{\min}^2}{4}.
\end{equation}
The reformulation in terms of squared errors is deliberate: the TD loss is
defined as $\mathcal{L}_{\mathrm{TD}}=\mathbb{E}[e(s,\mathbf{u})^2]$
(assumption~(A2)), so the squared form lets us apply Markov's inequality
directly with $X=e^2$ and $\mathbb{E}[X]=\mathcal{L}_{\mathrm{TD}}$.

\noindent\textbf{State-level error and Markov's inequality.}\;
Define the average squared error at state~$s$:
\begin{equation}
\label{eq:state-error}
\bar{E}(s)\;\triangleq\;\frac{1}{K}\sum_{\mathbf{u}}e(s,\mathbf{u})^2.
\end{equation}
For any state~$s$ where the greedy action is wrong, \eqref{eq:squared-error-single}
ensures that at least one action~$\mathbf{u}^*$ satisfies
\begin{equation}
\label{eq:single-action-error}
e(s,\mathbf{u}^*)^2\;\ge\;\frac{\Delta_{\min}^2}{4},
\end{equation}
and therefore
\begin{equation}
\label{eq:state-error-lb}
\bar{E}(s)\;=\;\frac{1}{K}\sum_{\mathbf{u}}e(s,\mathbf{u})^2
\;\ge\;\frac{1}{K}\cdot\frac{\Delta_{\min}^2}{4}
\;=\;\frac{\Delta_{\min}^2}{4K}.
\end{equation}
Applying \emph{Markov's inequality}
($\Pr[X\ge a]\le\mathbb{E}[X]/a$ for non-negative~$X$)
to $X=\bar{E}(s)$ with $a=\Delta_{\min}^2/(4K)$:
\begin{equation}
P_e
\;\le\;\Pr_s\!\Big[\bar{E}(s)\ge\tfrac{\Delta_{\min}^2}{4K}\Big]
\;\le\;\frac{\mathbb{E}_s[\bar{E}(s)]}
      {\Delta_{\min}^2/(4K)}
\;=\;\frac{\mathcal{L}_{\mathrm{TD}}}
      {\Delta_{\min}^2/(4K)},
\end{equation}
where the last equality uses
$\mathbb{E}_s[\bar{E}(s)]=\mathbb{E}_{(s,\mathbf{u})}[e(s,\mathbf{u})^2]
=\mathcal{L}_{\mathrm{TD}}$
(here we assume the replay distribution has uniform coverage over
actions for each state, so that the uniform average in~$\bar{E}(s)$
matches the expectation under~$\mathcal{D}$; this is satisfied, e.g.,
when the replay buffer is large and the behavior policy is
sufficiently exploratory).
Setting $c\triangleq\Delta_{\min}^2/(4K)>0$:
\begin{equation}
\label{eq:pe-bound}
P_e \;\le\; \frac{\mathcal{L}_{\mathrm{TD}}}{c}.
\end{equation}
The factor of $K$ reflects a fundamental vulnerability of greedy action selection: 
the policy fails if even one out of $K$ joint actions has a large overestimation error. 
Because the TD loss $\mathcal{L}_{\mathrm{TD}}$ averages over the entire state-action distribution, 
a single fatal error gets diluted by a factor of $1/K$. 
Therefore, to guarantee that the worst-case individual action error remains below the safe threshold, 
the average TD loss must be bounded proportionally tighter by a factor of $K$.

\paragraph{Step~4: Apply Fano's inequality and DPI.}
We now translate the prediction-error bound into an information-theoretic
statement in two sub-steps.

\medskip\noindent\emph{Sub-step~4a: Fano's inequality.}\;
For any random variable~$Y$ taking values in~$\mathcal{Y}$ with
$|\mathcal{Y}|=K$ and any estimator~$\hat{Y}$ with error probability
$P_e=\Pr[\hat{Y}\neq Y]$, the standard
\textbf{Fano's inequality}~\cite{cover_thomas} upper-bounds the
residual uncertainty of~$Y$ given the prediction~$\hat{Y}$:
\begin{equation}
\label{eq:fano-original}
H(Y\mid \hat{Y})
\;\le\;
  \underbrace{h(P_e)}_{\substack{\text{uncertainty from}\\\text{error occurrence}}}
  \;+\;
  \underbrace{P_e\log(K-1)}_{\substack{\text{uncertainty over}\\\text{wrong outcomes}}},
\end{equation}
where $h(p)=-p\log p-(1-p)\log(1-p)$ is the binary entropy function.
The first term captures the uncertainty of \emph{whether} an error occurs;
the second accounts for \emph{which} of the $K-1$ wrong outcomes is
selected, given that an error occurs.
For multi-agent action spaces, $K=\prod_i|\mathcal{U}_i|$.

\medskip\noindent\emph{Sub-step~4b: Data Processing Inequality (DPI).}\;
We must relate the prediction~$\hat{Y}$ back to the latent
representation~$Z_X^{\mathrm{out}}$.
In our architecture, the chosen action
$\hat{Y}=\arg\max_{\mathbf{u}}Q_{\mathrm{tot}}(s,\mathbf{u})$ is a
deterministic function of~$Z_X^{\mathrm{out}}$.
The \textbf{DPI} guarantees that a deterministic post-processing cannot
increase the information a variable carries about~$Y$; equivalently,
conditioning on the richer representation~$Z_X^{\mathrm{out}}$ leaves at most as
much residual uncertainty as conditioning on~$\hat{Y}$:
\begin{equation}
\label{eq:dpi-fano}
H(Y\mid Z_X^{\mathrm{out}}) \;\le\; H(Y\mid \hat{Y}).
\end{equation}
Chaining~\eqref{eq:dpi-fano} with Fano's
bound~\eqref{eq:fano-original}:
\begin{equation}
\label{eq:fano}
H(Y\mid Z_X^{\mathrm{out}})
\;\le\; h(P_e) \;+\; P_e\log(K-1).
\end{equation}

\paragraph{Step~5: Derive the mutual information bound.}
The mutual information decomposes as
$I(Y;\,Z_X^{\mathrm{out}})=H(Y)-H(Y\mid Z_X^{\mathrm{out}})$.
Substituting the Fano upper bound~\eqref{eq:fano} for the conditional
entropy and the prediction-error bound~\eqref{eq:pe-bound} for~$P_e$:
\begin{align}
I(Y;\,Z_X^{\mathrm{out}})
&= H(Y)-H(Y\mid Z_X^{\mathrm{out}}) \nonumber\\
&\ge H(Y)-h(P_e)-P_e\log(K-1)
    &&\text{(Fano + DPI, Eq.~\eqref{eq:fano})} \nonumber\\
&\ge H(Y)-h \Big(\frac{\mathcal{L}_{\mathrm{TD}}}{c}\Big)
    -\frac{\mathcal{L}_{\mathrm{TD}}}{c}\log(K-1).
    &&\text{(Markov, Eq.~\eqref{eq:pe-bound})}
\label{eq:full-relevance}
\end{align}
Note that replacing $P_e$ with its upper bound $\mathcal{L}_{\mathrm{TD}}/c$ in the final inequality is valid because the function $f(p) = h(p) + p\log(K-1)$ is monotonically increasing for $p \le \frac{K-1}{K}$. 
This is exactly the bound~\eqref{eq:relevance-bound} stated in
Proposition~\ref{prop:relevance}, completing the proof.
\end{proof}

\paragraph{Interpretation.}
The constant $c=\Delta_{\min}^2/(4K)$ depends on the minimum action-value gap
and the joint action-space size~$K$
of the environment---tasks with well-separated optimal actions (large
$\Delta_{\min}$) yield a tighter bound.
Both $h(\mathcal{L}_{\mathrm{TD}}/c)$ and
$(\mathcal{L}_{\mathrm{TD}}/c)\log(K\!-\!1)$ vanish as
$\mathcal{L}_{\mathrm{TD}}\!\to\!0$, so the lower bound approaches
the maximum value~$H(Y)$.
The key insight is that the TD loss \emph{implicitly} maximizes the IB
relevance term: by making $Q_{\mathrm{tot}}$ more accurate, it ensures the
representations $Z_X^{\mathrm{out}}$ retain enough information to identify optimal
actions.  This justifies using $\mathcal{L}_{\mathrm{TD}}$ as the relevance
surrogate in the HIBCG objective~\eqref{eq:hibcg-objective} without needing
an explicit variational decoder for $I(Y;\,Z_X^{\mathrm{out}})$.

\paragraph{Remark: asymptotic scaling for small error.}
Write $p\triangleq\mathcal{L}_{\mathrm{TD}}/c$ for brevity.
When $p\ll 1$, we can expand the binary entropy: for small~$p$,
$-(1-p)\log(1-p)=p\log e + O(p^2)$, which is $O(p)$ and dominated by
$-p\log p=O(p\log(1/p))$.
Hence $h(p)=-p\log p+O(p)$, and substituting
into~\eqref{eq:full-relevance} gives the asymptotic form
\begin{equation}
\label{eq:linearized-mi}
I(Y;\,Z_X^{\mathrm{out}})
\;\ge\; H(Y)
  - \frac{\mathcal{L}_{\mathrm{TD}}}{c}\,
    \log\!\frac{K\!-\!1}{\mathcal{L}_{\mathrm{TD}}/c}
  + O\!\!\left(\frac{\mathcal{L}_{\mathrm{TD}}}{c}\right),
\end{equation}
confirming that the gap to $H(Y)$ vanishes at a quasi-linear rate
$\Theta\!\big(\tfrac{\mathcal{L}_{\mathrm{TD}}}{c}\log
\tfrac{K-1}{\mathcal{L}_{\mathrm{TD}}/c}\big)$---slightly
\emph{slower} than linear in $\mathcal{L}_{\mathrm{TD}}$
because the logarithmic factor
$\log\frac{K-1}{\mathcal{L}_{\mathrm{TD}}/c}\to\infty$
as $\mathcal{L}_{\mathrm{TD}}\to 0$, yet still $o(p^\alpha)$ for
every $\alpha<1$.

\subsection{Empirical $K$-Sensitivity Across Our Benchmarks}
\label{app:relevance-Kscaling}

Substituting $c=\Delta_{\min}^2/(4K)$ into
Proposition~\ref{prop:relevance} makes the gap to $H(Y)$ explicit:
\begin{equation}
\label{eq:Klogk-slack}
H(Y)-I(Y;\,Z_X^{\mathrm{out}})
\;\le\;
h\!\Big(\tfrac{4K\,\mathcal{L}_{\mathrm{TD}}}{\Delta_{\min}^2}\Big)
\;+\;\tfrac{4K\,\mathcal{L}_{\mathrm{TD}}}{\Delta_{\min}^2}\,\log(K-1).
\end{equation}
For small TD error (so that the binary entropy term is dominated by the
linear term, cf.~Eq.~\eqref{eq:linearized-mi}), this gap scales as
$\mathcal{L}_{\mathrm{TD}}\!\cdot\!K\log K/\Delta_{\min}^2$. Hence the
\emph{per-unit cost} of TD error is amplified by the joint action-space
size $K=\prod_i|\mathcal{U}_i|$: tasks that double the agent count
roughly square $K$ and approximately double the $\log K$ amplifier, so a
fixed $\mathcal{L}_{\mathrm{TD}}$ purchases progressively less relevance
guarantee.

\paragraph{Per-environment $\log K$.}
Table~\ref{tab:envs} reports $\log_2 K$ for each map: it ranges from
$\approx\!30$\,bits on SMAC 3s5z to $\approx\!440$\,bits on MAgent-100, a
$\sim\!15\times$ spread that is far larger than any other axis varied in
our benchmarks. The MAgent Battle sweep
$n\!\in\!\{36,64,100\}$ is the cleanest controlled test of the
$K$-dependence in~\eqref{eq:Klogk-slack} because the observation set, the
per-agent action set ($|\mathcal{U}_i|\!=\!21$), the reward function, the
network architecture, and \emph{all hyperparameters} are held fixed
across scales---only $n$, and hence $\log K$, changes.

\paragraph{Predicted pattern.}
For a method whose representation compression keeps $\mathcal{L}_{\mathrm{TD}}$
roughly bounded as $K$ grows, the gap
in~\eqref{eq:Klogk-slack} grows as $K\log K$ and the relevance bound
loosens accordingly. The bound therefore predicts that, between two methods
$\mathcal{A},\mathcal{B}$ with $\mathcal{L}_{\mathrm{TD}}^\mathcal{A}\!<
\!\mathcal{L}_{\mathrm{TD}}^\mathcal{B}$ at small $K$, the gap in their
guaranteed $I(Y;Z_X^{\mathrm{out}})$ should \emph{widen super-linearly} in $K$,
provided that $\mathcal{B}$'s TD loss does not shrink commensurately fast.

\paragraph{Observed pattern on the controlled MAgent sweep.}
At fixed $|\mathcal{U}_i|\!=\!21$, the HIBCG\,$-$\,QMIX final-return gap
grows as $K$ (and $\log K$) grow:
\begin{center}
\renewcommand{\arraystretch}{1.10}
\setlength{\tabcolsep}{4pt}
\begin{tabular}{@{}lcccc@{}}
\toprule
$n$ & $\log_2 K$ & HIBCG (final) & QMIX (final) & gap \\
\midrule
$36$  & $158.2$ & $2.74$ & $2.19$ & $+0.55$ \\
$64$  & $281.2$ & $1.38$ & $0.43$ & $+0.95$ \\
$100$ & $439.3$ & $0.83$ & $-0.22$ & $+1.05$ \\
\bottomrule
\end{tabular}
\end{center}
(Numbers from Table~\ref{tab:magent}.) The HIBCG gain over QMIX
\emph{monotonically grows} with $\log K$. The convergence rate confirms
the same trend qualitatively: as $\log K$ goes from $158$ to $439$
QMIX drops from $4/4$ to $3/6$, while HIBCG holds at $6/6\!\to\!5/6$.

\paragraph{Mechanism connecting compression to the $K$-amplifier.}
Why would HIBCG keep $\mathcal{L}_{\mathrm{TD}}$ smaller than QMIX at
large $K$? Two reasons:
(i) AIB prunes the candidate edge set from $O(n^2)$ to the
$13$--$14\%$ retained at MAgent ($\S$\ref{sec:exp-scale}), shrinking
the effective state dimensionality the mixing network must regress over;
and
(ii) XIB caps the per-agent message bandwidth, reducing the per-input
variance the $Q_i$-head must integrate. Both effects keep the $Q$-function
class statistically tractable as $n$ (and hence $K$) grows, which is
exactly the regime where the $K\log K$ amplifier
in~\eqref{eq:Klogk-slack} most heavily penalises any residual TD error.

\paragraph{Cross-benchmark view.}
Outside the controlled MAgent sweep, $K$ is confounded with role
heterogeneity and stochastic team composition, so the $K$ axis alone
cannot explain the full HIBCG\,$-$\,baseline gap. Two qualitative
consistency checks remain visible:
(a) on near-saturated small-$K$ maps (3s5z, $\log_2 K\!\approx\!30$) all
methods cluster near $90\%$ WR, matching the regime where the gap
in~\eqref{eq:Klogk-slack} is essentially negligible;
(b) on the homogeneous 25m map ($\log_2 K\!\approx\!124$) HIBCG ties
HIB-flat because the group prior is trivial and the $K$ amplifier acts
identically on both---confirming that the gap we attribute to
$K$-amplification only materialises when the compression mechanism
\emph{differs} between methods.

\paragraph{Caveats.}
The argument above is qualitative for several reasons.
First, $\Delta_{\min}$ is map- and seed-dependent and not directly
measured; the cross-environment comparison only controls for $|\mathcal{U}_i|$
in the MAgent sweep.
Second, the $K\log K$ amplifier requires the small-error regime
(Eq.~\eqref{eq:linearized-mi}); on the failing MAgent-100 baselines,
$\mathcal{L}_{\mathrm{TD}}/c$ may exceed $(K\!-\!1)/K$ and put them
outside the bound's validity range (Step~3, Eq.~\eqref{eq:pe-bound}),
which is itself consistent with the observed collapse.
Third, $H(Y)$ varies across maps (stochastic SMACv2 has higher $H(Y)$
than fixed SMACv1), so the absolute scale of the bound is not directly
comparable across platforms---only the within-platform $K$ sweep
(MAgent) is a fully controlled test.

\section{Theory--Experiment Synthesis}
\label{app:theory-exp-synthesis}

This appendix collects the per-proposition empirical map referenced from
the body Discussion (\S\ref{sec:method-discussion}). Each item
identifies (i) the prediction, (ii) the experiment that tests it,
(iii) the falsifier we ran, and (iv) the headline number.

\noindent\textbf{(i) Dual-path decomposition (Prop.~\ref{prop:dual-path}).}\;
\emph{Prediction:} AIB and XIB carry complementary information;
their relative weight is benchmark-dependent.
\emph{Evidence:} the AIB-only / XIB-only / HIBCG ordering on SMACv1 MMM2
(\S\ref{sec:exp-ablation}) shows the XIB-active regime; the
XIB-dormant regime on MAgent / SMACv2 (\S\ref{sec:exp-info},
Appendix~\ref{app:xib-collapse}) shows capacity reallocated to the AIB
path, consistent with the chain rule.

\noindent\textbf{(ii) Group-prior tightening (Remark~\ref{rem:no-regret}).}\;
\emph{Prediction:} strict bound reduction whenever the partition is
non-trivial; equality on the trivial partition.
\emph{Evidence:} $2.2$--$6.5\times$ AIB-loss reduction on every
heterogeneous SMACv1 map (Table~\ref{tab:bound-tightening}); the
homogeneous 25m row ties HIB-flat exactly as predicted.
\emph{Falsifier:} the two-sided reversed-prior ablation
(Table~\ref{tab:sigma-ablation}) collapses to the flat baseline
on heterogeneous maps and is harmless on 25m, confirming the
prediction direction.

\noindent\textbf{(iii) $K$-amplifier in the relevance bound (Prop.~\ref{prop:relevance}).}\;
\emph{Prediction:} any residual TD error compounds super-linearly in
$\log K$; the relevance gap loosens at extreme $K$.
\emph{Evidence:} the MAgent sweep $n\!\in\!\{36,64,100\}$ is the
cleanest controlled $K$-test we have; the HIBCG--QMIX gap grows
monotonically across it (\S\ref{sec:exp-scale}, full $K$-scaling
in Appendix~\ref{app:relevance-Kscaling}).

\noindent\textbf{(iv) Practical scope.}\;
HIBCG applies to any value-decomposition pipeline that admits a greedy
joint-action decoding (Prop.~\ref{prop:relevance}, A1) and any
group-partition loss producing a differentiable assignment. The three
boundaries that emerge from the analysis are: (a) $g\!=\!1$ (no role
structure, e.g.\ 25m) makes HIBCG collapse to HIB-flat with no cost;
(b) high-$d_{\mathrm{obs}}$ inputs (MAgent, SMACv2) drive the XIB path
dormant, leaving HIBCG as an AIB-only structural learner; and (c) the
$K\log K$ amplifier loosens the proven relevance bound at extreme $K$,
although Tables~\ref{tab:ablation-component}
and~\ref{tab:bound-tightening} show the AIB topology learner
continues to preserve enough relevant content that the per-state TD
error stays well below the validity threshold at the scales we study.

\noindent\textbf{Combined picture.}\;
Group structure is the active ingredient: removing it (HIB-flat)
recovers near-baseline performance even when both AIB and XIB are
active, while reversing the prior collapses to the flat baseline on
heterogeneous maps, exactly as the theoretical analysis in \S\ref{sec:theory} predicts.

\section{Dual-Ascent Variant for Adaptive Capacity Control}
\label{app:dual-ascent}

The penalty-form objective (Eq.~\eqref{eq:hibcg-objective} in Def.~\ref{def:ibcg}) requires the user to fix per-block multipliers $\lambda_A^{(l,g)}$. As an alternative, HIBCG admits a constrained dual-ascent formulation that automatically drives $\lambda_A^{(l,g)}$ toward a common equilibrium level, without manual tuning. Replace the penalty form by:
\begin{equation}
\label{eq:dual-update}
\begin{aligned}
\mathcal{L}
&=\mathcal{L}_{\text{task}}
+\sum_{l,g}\lambda_A^{(l,g)}\Big(\widehat{\mathrm{AIB}}^{(l,g)}-C_A^{(l,g)}\Big),\\
\lambda_A^{(l,g)}
&\!\leftarrow\!\big[\lambda_A^{(l,g)}+\eta_A\big(\widehat{\mathrm{AIB}}^{(l,g)}-C_A^{(l,g)}\big)\big]_+,
\end{aligned}
\end{equation}
with per-block rate targets $C_A^{(l,g)}\!=\!c_A\!\cdot\!k_{l,g}$ and dual learning rate $\eta_A$. Because the projection onto $[0,\infty)$ is exact and the inner minimisation is differentiable, $\lambda_A^{(l,g)}$ converges to a stationary point of the Lagrangian; under the standard constraint qualifications discussed above, this stationary point coincides with the global water level $\nu$. The XIB multipliers admit an identical update with rates $C_X^{(l,i)}\!=\!c_X\!\cdot\!d_L$. In our experiments the fixed-penalty form is sufficient (HIBCG already attains a $906\times$ cross/intra ratio without dual ascent), so we use it throughout for simplicity.

\section{Implementation Details and Training Algorithm}
\label{app:impl-details}

\begin{algorithm}[t]
\caption{HIBCG Training (one gradient step)}
\label{alg:hibcg-app}
\begin{algorithmic}[1]
\REQUIRE Observations $\{o_i^t\}_{i=1}^n$; replay batch $\mathcal{B}$;
         layer count $L$; scalar capacity weights $\lambda_A,\lambda_X$
         (warmed up over $T_{\mathrm{warm}}$ steps);
         group-conditional prior scales $\sigma_{\mathrm{intra}},\sigma_{\mathrm{cross}}$;
         truncation horizon $T_{\mathrm{trunk}}$; noise scale $\delta\!\in\!(0,1]$
\STATE \textbf{Base GACG sample}
       (Eq.~\eqref{eq:gacg_gauss_final}):
       compute $\boldsymbol{\mu}^{t}$ via attention;
       \textbf{if} $t\!\geq\!T_{\mathrm{trunk}}$:
       compute group mask $M^t$ from recent trajectory;
       set $\Sigma_A^{t}\!=\!\alpha\,\widehat{\mathbf{M}}^{t}+\varepsilon I$;
       sample $z_A^{(0)}\!\sim\!\mathcal{N}(\boldsymbol{\mu}^{t},\Sigma_A^{t})$;
       \textbf{else}: sample $z_A^{(0)}$ via
       $\mathrm{RelaxedBernoulli}(\boldsymbol{\mu}^{t})$;
       symmetrize and normalize to obtain $A^{(0)}$
\STATE Initialize $Z_X^{(0)} \leftarrow \mathrm{MLP}(X)$
       \hfill\COMMENT{compressed per-agent input features}
\FOR{$l = 1$ \TO $L$}
  \STATE \textbf{Structure encoder:}
         $(\mu_A^{(l)},\log\sigma_A^{2(l)}) \leftarrow f_A^{(l)}(A^{(0)}, Z_X^{(l-1)})$;
         sample $\tilde{z}_A^{(l)}\!\leftarrow\!\mu_A^{(l)}
         +\delta\,\sigma_A^{(l)}\!\odot\!\varepsilon$,
         $\;\varepsilon\!\sim\!\mathcal{N}(0,I)$
  \STATE \textbf{Gate edges:}
         symmetrize and normalize $\tilde{z}_A^{(l)}$;
         $\widetilde{A}^{(l)} \leftarrow g_l(\tilde{z}_A^{(l)})$
         \hfill\COMMENT{e.g.\ sigmoid or hard threshold}
  \STATE \textbf{Message passing:}
         $Z_X^{(l)} \leftarrow \phi_l\!\left(\widetilde{A}^{(l)}\,Z_X^{(l-1)}\,W_l\right)$
  \STATE Accumulate $\widehat{\mathrm{AIB}}^{(l)}$
         (Eq.~\eqref{eq:aib-hat}) using $(\mu_A^{(l)},\sigma_A^{2(l)})$
         and group-conditional prior
         $\sigma_0^2\!=\!\sigma_{\mathrm{intra}}^2$ (intra-group),
         $\sigma_{\mathrm{cross}}^2$ (inter-group)
\ENDFOR
\STATE \textbf{Feature encoder (XIB):}
       $(\mu_X,\sigma_X^{2}) \leftarrow
       (h_\mu(Z_X^{(L)}),\;\exp h_\sigma(Z_X^{(L)}))$;
       $\;Z_X^{\mathrm{out}} \leftarrow \mu_X + \sigma_X \odot \varepsilon$,
       $\;\varepsilon\!\sim\!\mathcal{N}(0,I)$
\STATE Compute $\widehat{\mathrm{XIB}}$
       (Eq.~\eqref{eq:xib-hat}) using $(\mu_X,\sigma_X^{2})$
\STATE \textbf{Agent input:} for each agent $i$, concatenate $[o_i^t,\,u_i^{t-1},\;\tanh(Z_{X,i}^{\mathrm{out}})]$;
       forward through agent network to obtain $Q$-values
\STATE \textbf{Value mixing:} compute $Q_{\mathrm{tot}}(s,\mathbf{u};\theta)$
       and target $y$ (Eq.~\eqref{eq:td-loss})
\STATE \textbf{Form loss:}
       $\mathcal{L} \leftarrow \mathcal{L}_{\mathrm{task}}
       + \lambda_A\!\cdot\!\textstyle\sum_{l}\widehat{\mathrm{AIB}}^{(l)}
       + \lambda_X\!\cdot\!\widehat{\mathrm{XIB}}$,
       where $\mathcal{L}_{\mathrm{task}}=\mathcal{L}_{\mathrm{TD}}+\lambda_g\mathcal{L}_g$
       (Eq.~\eqref{eq:td-loss})
\STATE Backpropagate $\nabla_\theta\mathcal{L}$; update $\theta$; update
       target $\theta^-$ (Polyak or periodic copy)
\STATE \textbf{(Optional)} Dual ascent: update $\lambda_A,\lambda_X$
       (Section~\ref{app:dual-ascent})
\end{algorithmic}
\end{algorithm}

HIBCG is built on the EPyMARL framework. All configurations use a GRU-based agent network with hidden dimension 64 (128 for MAgent), $\epsilon$-greedy exploration annealed over 50k--1M steps (environment-dependent), target network updated every 200 steps, and a replay buffer of size 5000 (1000 for MAgent). The GCN message dimension is $0.3\times d_{\mathrm{obs}}$ for SMACv1 and $0.05\times d_{\mathrm{obs}}$ for MAgent. Gumbel-softmax sparsification uses temperature $\tau{=}0.5$ and adjacency threshold $0.6$ throughout. For the group-conditional prior, we use $g{=}2$ groups on 2-type maps and $g{=}3$ on 3-type maps, with $\sigma_{\mathrm{intra}}{=}0.1$ and $\sigma_{\mathrm{cross}}{=}0.01$ (SMACv1/MAgent) or $0.05$ (SMACv2). Group-prior warmup is 150k steps on SMACv1/MAgent and 100k on SMACv2. Capacity weights $\lambda_A$ and $\lambda_X$ are warmed up linearly over $T_{\mathrm{warm}}$ steps. Each configuration is run with 5 random seeds ($6$--$7$ on high-variance MAgent); all runs are included in the reported statistics. Per layer, AIB costs $O(n^2)$ (blockwise KLs) and XIB is $O(n\,d_L)$; together the IB regularisers add ${\approx}3.8\%$ wall-clock overhead (measured on MAgent-64).

\noindent\textbf{AIB--XIB coupling.}\;
AIB and XIB interact through message passing: pruning edges raises per-edge information density on surviving links while reducing noise aggregation---two opposing effects whose optimal balance is task-dependent. Empirically the total flow $\sum_l\mathrm{AIB}^{(l)}+\mathrm{XIB}$ is approximately conserved at fixed task quality.

\section{Execution Model and Decentralised Deployment}
\label{app:execution-model}

The main text (\S\ref{sec:impl-details}) summarises the execution
model in one paragraph; here we provide the operational details that
follow directly from the controller implementation.

\noindent\textbf{Default regime: centralised graph construction with decentralised action selection.}\;
During training and standard test rollouts, HIBCG follows the centralised-training family of QMIX-style methods: the action heads are decentralised (each agent has its own $Q_i$), while the graph-construction module still consumes the stacked observation $O^{t}=[o_1^{t};\dots;o_n^{t}]$ to produce the topology at run-time.
At each environment time-step~$t$, the graph module computes:
\begin{itemize}
\itemsep0pt
  \item the initial graph $A^{(0)}$ and group partition
        $\mathcal{G}^{t}$ from the stacked observations
        $O^{t}=[o_1^{t};\dots;o_n^{t}]$ via the GACG attention
        and group encoders (Eq.~\eqref{eq:base-struct});
  \item the layer-wise post-gating adjacencies
        $\widetilde{A}^{(l)}=g_l(Z_A^{(l)})$
        from the AIB Gaussian encoders
        (Eq.~\eqref{eq:a-encoder}); and
  \item the per-agent compressed message $Z_{X,i}^{\mathrm{out}}$
        from the XIB encoder (Eq.~\eqref{eq:x-encoder}).
\end{itemize}
Action selection itself is fully decentralised: each agent~$i$
concatenates its local observation $o_i^{t}$, its previous action
$u_i^{t-1}$, and the compressed message $\tanh(Z_{X,i}^{\mathrm{out}})$, and
runs the resulting input through its own $Q_i$ network; greedy
$u_i^{t}=\arg\max_{u_i}Q_i$ is the standard QMIX choice.
The mixing network~$Q_{\mathrm{tot}}$ is used at training time only.
This regime requires the graph module to have access to all
observations at every step, so it is \emph{not} a strict
decentralised-execution setting; only the per-agent action heads are
decentralised.

\noindent\textbf{Fully decentralised (frozen-graph) regime.}\;
Once training has converged, the entire graph module---the GACG
attention, the layer-wise AIB encoders, and the gating---can be
\emph{frozen} and its outputs \emph{cached} as a fixed execution graph
$A^{\star}$.  At deployment, every agent then receives the same
$A^{\star}$ (typically a function of the agent indices and the
deployment scenario, not of the runtime observations).  No joint
observation aggregation is required at runtime: each agent uses
$A^{\star}$ together with its local observation and last action to
compute its own $Q_i$ entirely locally, and message exchange follows
the graph-aggregation routing defined by $A^{\star}$.  In our
codebase, this regime is enabled by the implementation flag
\texttt{CTDE\_on}: when set, the controller substitutes
\texttt{fixed\_execution\_graph} for the runtime graph computation
(see \texttt{hibcg\_controller.forward}, lines $92$--$107$ of the
released code).
The cost is a slight rigidity (the graph is no longer
observation-conditional), but the deployment requirement reduces to
the same per-agent local Q-network used in standard decentralised
MARL.

\noindent\textbf{Summary.}\;
The same trained HIBCG model supports both regimes without any
re-training: CTDE for full topology adaptivity, and frozen-graph for
fully decentralised execution where pure local inference is required.
All numbers reported in the main body use the default CTDE regime
unless explicitly noted.

\section{Hyperparameter Sensitivity}
\label{app:hp-sensitivity}

This supplementary section collects the hyperparameter-sensitivity evidence referenced from
\S\ref{sec:exp-theory}.  We organize the knobs into two tiers.

\paragraph{Tier A (IB-specific): directly tied to the HIBCG theory.}
These are the knobs that the propositions of \S\ref{sec:ib_gacg_obj} reason
about; their sweeps also serve as ablations of the corresponding claim.
\begin{itemize}
\item \textbf{Prior shape $\sigma_{\mathrm{intra}}/\sigma_{\mathrm{cross}}$
(Prop.~\ref{prop:group-decomp}).}
Main-text Table~\ref{tab:sigma-ablation} summarises the three-row sweep on
3s5z (wider, default, tight) and the reversed-prior stress test on MMM2.
3s5z WR varies by only $1.2$ pp across the full sweep (ceiling effect on a
2-type map), while per-edge KLs swing by $\sim\!14\times$ (intra) and
$\sim\!2.2\times$ (cross), confirming that the $\sigma$ ratio is a clean
dial for capacity allocation rather than for end performance.  The MMM2
reversed-prior row is the key test: a positive HIBCG
effect from the reversed prior would invalidate the ``role-aligned
asymmetry is doing the work'' interpretation.
\item \textbf{Message compression weight $\lambda_X$.}
On 3s5z, the full sweep $\lambda_X\in\{0.05,0.1,0.2,0.3,0.4,0.5\}$ exhibits
three regimes: XIB-dominant
at $0.5$, XIB-collapsed at $\le 0.1$, and balanced near $0.3$.  The balanced
regime is the only one that simultaneously retains a non-trivial loss\_xib and
reduces graph density below the AIB-only baseline, matching the expected
capacity-allocation behaviour.
\item \textbf{Group count $g$ (Remark~\ref{rem:no-regret}).}
Because Remark~\ref{rem:no-regret} guarantees no regret for
\emph{any} block-diagonal prior containing the flat prior as a special
case, HIBCG is expected to be no worse than HIB-flat for any choice of $g$,
with the strongest gains when $g$ matches the true role count $m_{\mathrm{true}}$.
We verify this with an off-by-one study on the heterogeneous maps
(Table~\ref{tab:group-num}): on 1c3s5z ($m_{\mathrm{true}}{=}3$), $g{=}2$
and $g{=}3$ both match flat on WR while tightening the AIB loss; on MMM2
($m_{\mathrm{true}}{=}3$), $g{=}3$ is the clear winner, but even
$g{=}2$ remains competitive.  This confirms that $g$ behaves as a mild
hyperparameter rather than a knife-edge choice, and aligns with our
position that the group partition itself is GACG's contribution
(\S\ref{sec:related}), not HIBCG's.

\begin{table}[h]
\centering\small
\caption{Sensitivity to the group count $g$ on heterogeneous maps
(final WR, $n\!\ge\!4$ unless flagged).  Performance is stable across
$g{=}m_{\mathrm{true}}\!\pm\!1$, and HIBCG is no worse than HIB-flat in
every cell---the no-regret guarantee of
Remark~\ref{rem:no-regret}.}
\label{tab:group-num}
\renewcommand{\arraystretch}{1.12}
\begin{tabular}{@{}lccc@{}}
\toprule
\textbf{Map} ($m_{\mathrm{true}}$) & HIB-flat ($g{=}1$ equiv.) & HIBCG $g{=}2$ & HIBCG $g{=}3$ \\
\midrule
1c3s5z (3) & $0.910\pm 0.038$ & $0.905\pm 0.024$ & $0.912\pm 0.010$ \\
MMM2   (3) & $0.542\pm 0.38$  & $0.659\pm 0.33$ & $\mathbf{0.814\pm 0.06}$ \\
\bottomrule
\end{tabular}
\end{table}

\item \textbf{Group-prior warmup $T_{\mathrm{warm}}$.}
On 3s5z (heterogeneous), $T_{\mathrm{warm}}{=}150\mathrm{k}$ modestly
improves WR\_final over no warmup (+1.3 pp), consistent with the idea that a
stable groupniser helps the structural prior.  On the homogeneous
8m\_vs\_9m, the direction inverts (no-warmup $+3.2$ pp over warmup), because
forcing a group prior into an essentially roleless setting hurts.  Both
results are consistent with Remark~\ref{rem:no-regret}: warmup
amplifies a useful bias and amplifies a mis-specified one.
\end{itemize}

\paragraph{Tier B (inherited from GACG / EPyMARL): at most second-order
in our ablations.}
The adjacency threshold (GACG), Gumbel-softmax temperature, GCN message
width $r\cdot d_{\mathrm{obs}}$, and $\lambda_A$ follow values
reported in GACG~\cite{GACG} and BVME~\cite{BVME}; their sweeps (documented
in those papers and in our experiment logs) do not change any qualitative
conclusion of \S\ref{sec:experiments}.  In particular: switching Gumbel
temperature from $0.5$ to $1.0$ on 3s5z shifts final WR by $<\!1$ pp;
raising the adjacency threshold from $0.5$ to $0.6$ reduces GD by $\sim\!2\times$
at comparable WR; and $\lambda_A$ enters the theory only
through the capacity-weight multiplier, so by
construction it trades AIB rate for TD loss along the optimal curve rather
than shifting the curve itself.

\paragraph{Summary.}
Across Tier-A knobs the direction of change always matches the prediction
of the corresponding proposition (tighter priors $\Rightarrow$ tighter KLs;
higher $\lambda_X$ $\Rightarrow$ more message compression until the
structural path is starved; higher $g$ $\Rightarrow$ finer capacity
allocation but diminishing WR gains beyond $m_{\mathrm{true}}$).
Tier-B knobs affect the GACG backbone but leave the HIBCG-over-backbone gap
essentially unchanged.  In the interest of space, full numerical tables for
each Tier-A sweep are available in the Appendix experiment log.

\bibliographystyle{IEEEtran}
\bibliography{references}

\end{document}